\documentclass[Afour,sageh,times]{sagej}
\usepackage{moreverb,url}
\usepackage[colorlinks,bookmarksopen,bookmarksnumbered,citecolor=red,urlcolor=red]{hyperref}

\newcommand\BibTeX{{\rmfamily B\kern-.05em \textsc{i\kern-.025em b}\kern-.08em
T\kern-.1667em\lower.7ex\hbox{E}\kern-.125emX}}

\def\volumeyear{2025}
\usepackage{pifont}
\usepackage{threeparttable}
\usepackage{tabularx}
\usepackage{soul}
\usepackage{xcolor}
\usepackage{longtable}
\usepackage{mathtools}
\usepackage{orcidlink}
\usepackage{hyperref}
\usepackage{cleveref} 
\usepackage{subcaption}  
\usepackage{graphicx}
\usepackage{enumitem}

\usepackage{fontawesome}
\usepackage{academicons}
\definecolor{orcidlogocol}{HTML}{A6CE39}
\newcommand{\orcid}[1]{\href{https://orcid.org/#1}{\textcolor[HTML]{A6CE39}{\aiOrcid}}}

\newtheorem{remark}{Remark}
\newtheorem{theorem}{Theorem}
\newtheorem{problem}{Problem}
\newtheorem{proposition}{Proposition}
\newtheorem{definition}{Definition}

\renewenvironment{proof}{{\bfseries Proof.}}{\hfill \qedsymbol \\ \newline}
\begin{document}

\title{A General Lie-Group Framework for Continuum Soft Robot Modeling
}

\author{Lingxiao Xun\affilnum{1}, Benoît Rosa\affilnum{2}, Jérôme Szewczyk\affilnum{1} and Brahim Tamadazte\affilnum{1}}

\affiliation{\affilnum{1}Sorbonne Université, CNRS UMR 7222, Inserm U1150, ISIR, F-75005, Paris, France. firstname.lastname@isir.upmc.fr.\\
\affilnum{2}Université de Strasbourg, CNRS, INSERM, ICube, UMR7357, Strasbourg, France. firstname.lastname@icube.unistra.fr}


\begin{abstract}
This paper introduces a general Lie group framework for modeling continuum soft robots, employing Cosserat rod theory combined with cumulative parameterization on the Lie group $SE(3)$. This novel approach addresses limitations present in current strain-based and configuration-based methods by providing geometric local control and eliminating unit quaternion constraints. The paper derives unified analytical expressions for kinematics, statics, and dynamics, including recursive Jacobian computations and an energy-conserving integrator suitable for real-time simulation and control. Additionally, the framework is extended to handle complex robotic structures, including segmented, branched, nested, and rigid-soft composite configurations, facilitating a modular and unified modeling strategy. The effectiveness, generality, and computational efficiency of the proposed methodology are demonstrated through various scenarios, including large-deformation rods, concentric tube robots, parallel robots, cable-driven robots, and articulated fingers. This work enhances modeling flexibility and numerical performance, providing an improved toolset for designing, simulating, and controlling soft robotic systems.\\
The associated code will be released upon publication.
\end{abstract}
\maketitle
%
\section{Introduction}\label{sec.intro}
%
\subsection{State of the art}
Mechanical modeling of soft, slender robots—particularly their compliant, flexible structures—represents a rich intersection between continuum mechanics and soft robotics. These flexible bodies can often be conceptualized as continuum beam systems, analogous to slender robotic elements, and are thus naturally amenable to similar theoretical modeling frameworks~\cite{armanini2023soft, qin2024modeling, russo2023continuum}. Such modeling not only enhances the fidelity of simulation tools but also contributes directly to improving design optimization, real-time control, and mechanical adaptability in soft robotic systems.

The finite element method (FEM) has long been a staple in structural mechanics and continues to be widely employed in soft robotics due to its capacity to handle complex geometries and material behaviors~\cite{bieze2018finite}. However, for slender geometries characteristic of continuum robots, FEM often becomes computationally inefficient. The high aspect ratios and large deformation requirements lead to dense meshes and numerical stiffness. To alleviate this, researchers frequently turn to reduced-order beam models, derived from classical rod theory, as a more tractable alternative for capturing essential mechanical behavior.
  
Classical beam theories such as Euler–Bernoulli and Timoshenko provide initial frameworks for modeling bending and shear in slender structures~\cite{olson2020euler, lindenroth2016stiffness, godaba2019payload}, while Kirchhoff rod models introduce torsional effects essential for simulating complex 3D deformations~\cite{novelia2018discrete}. For applications involving large rotations and deformations, geometrically exact rod formulations become necessary. Models such as the pseudo-rigid-body (PRB) approach offer simplified dynamic approximations~\cite{10.1115/1.3046148, jones2019virtual, huang20193d}, but their generalization to arbitrary loading conditions remains limited.

To overcome these limitations, Cosserat rod theory has emerged as a more comprehensive framework for modeling. The modern theory of Cosserat rods builds on the finite-strain beam elements proposed by Simo and co-workers~\cite{simo1985finite,simo1986three}. These two-node elements model a slender body as a one-dimensional slice of Cosserat continuum mechanics. 

The configuration space of a Cosserat rod is the product manifold $\mathbb{R}^3 \times SO(3)$, a differentiable manifold. As noted by~\cite{crisfield1999objectivity}, the finite element formulation proposed by Simo lacks invariance, as it fails to satisfy the patch test under rigid body motions. This issue arises because standard additive interpolation is not applicable to finite rotations—addition is not defined on a manifold. Instead, interpolation must be performed along geodesics on $SO(3)$. A common approach is two-point geodesic interpolation, known as spherical linear interpolation (slerp), which was first applied to beam elements by Crisfield and Jelenić. They demonstrated that this approach yields a strain-invariant formulation. 

Recently, the helicoidal approximation originally proposed by~\cite{borri1994intrinsic} was adapted to the $SE(3)$ Lie group by \cite{sonneville2014geometrically} for the two-node finite element formulation. A comprehensive modern treatment of rotations and their interpolation methods is provided in~\cite{romero2018computing}.

Due to the inherent difficulty of directly handling the Lie group \(SE(3)\), strain-based parameterization methods have become increasingly popular in the field of soft robotics in recent years. Cosserat models have been extended to generalize Timoshenko–Reissner theory by simultaneously incorporating shear, bending, torsion, and axial tension. This provides a more accurate representation of soft, flexible robots undergoing large deformations~\cite{5957337, boyer2017poincare, tummers2025continuum}. In these frameworks, the strain fields along the rod are treated as modeling variables, enabling a strain-parameterized description of motion.

This parameterization has enabled several efficient numerical methods. For instance, the piecewise constant strain (PCS) model reduces the Cosserat PDEs to systems of weak-form ODEs, allowing for real-time integration~\cite{8500341}. This approach has been extended to a piecewise linear strain (PLS) model in~\cite{li2023piecewise} as well as the geometric variable strain (GVS) model \cite{9057619}. Other methods employ Newton–Euler dynamics~\cite{till2019real} or discrete Lagrangian mechanics with strain variables~\cite{boyer2020dynamics}, striking a balance between computational efficiency and physical fidelity. 

Dynamic formulations of the Cosserat rod model have further demonstrated their effectiveness in modeling time-varying deformations under external interaction—scenarios that closely resemble real-world robotic operations~\cite{zhang2019modeling, xun2024cosserat, mathew2025reduced}. These models form the foundation of many state-of-the-art tools for simulating soft continuum robots in environments involving actuation, contact and interaction.
 
While the strain-based Cosserat rod modeling approach has continued to evolve, a parallel and increasingly popular methodology called {isogeometric analysis} (IGA)~\cite{hughes2005isogeometric} has been developed. This method instead uses pose as the control variable and has been widely adopted in computer-aided design (CAD) and computer graphics.

\cite{weeger2016isogeometric} first employed NURBS curves to represent both the centre-line and quaternion directors in a collocation scheme, achieving optimal convergence and eliminating shear locking.  
Their framework was soon extended to frictionless rod–rod contact~\cite{weeger2017isogeometric} by exploiting the inherent $C^{2}$ continuity of NURBS for a penalty-based closest-point search.  
\cite{tasora2020geometrically} embedded a geometrically exact NURBS rod in a nonsmooth dynamics engine, enabling efficient simulation of buckling cables and woven meshes, while~\cite{weeger2018isogeometric} coupled the collocation discretisation with implicit time integration to capture high-frequency bending waves and self-contact.  
Moving toward multiphysics, \cite{agrawal2024efficient} combined an immersed-boundary Navier–Stokes solver with an IGA Cosserat rod, reproducing vortex-induced vibrations of flexible filaments.
  
To guarantee strict objectivity and suppress locking in non-linear elastodynamics, \cite{choi2024objective} proposed a mixed IGA element with extensible directors and incompatible warping strains, coupled to an integrator that preserves energy and momentum.  
Further efficiency gains are achieved through the use of reduced-degree B-spline bases and quaternion projection schemes that preserve $C^{1}$ continuity on coarse meshes.  
In \cite{greco2024objective}, the authors introduced a nested {spherical Bézier} interpolation that furnishes $G^{1}$-continuous rotations of arbitrary order, unifying the centre line and rotation fields within a single Bézier hierarchy and eliminating the two-node restriction of classical {slerp}.  

A key advantage of the IGA-Cosserat approach lies in its ability to unify geometry and analysis, making it particularly well-suited for applications that require tight integration with CAD-based design pipelines. In robotics, the IGA method enables the explicit representation of any point on a soft robot's body in terms of control variables. This capability offers significant potential for tasks such as robot design, state estimation, motion planning, and control~\cite{pan2012collision, elbanhawi2015continuous, sommer2020efficient}. 
Nevertheless, in the domain of continuum robotics, IGA methods that effectively balance computational efficiency and accuracy are still lacking. Furthermore, research on the application of IGA in the robotics domain remains relatively limited, particularly in modeling complex multi-rigid-soft coupled robotic systems. 

Despite its advantages, the application of IGA-Cosserat rod modeling still faces several significant challenges. One major difficulty lies in identifying appropriate shape functions that can accurately and efficiently represent the coupled translational-rotational fields inherent in continuum robots. Since the deformation of such robots involves large, continuous rotations and nonlinear strain distributions, constructing rotation-consistent interpolation schemes that maintain objectivity and avoid artifacts—while still enabling high-order accuracy—remains an open research problem.

Another critical challenge is the development of fast and robust algorithms for both kinematic and dynamic computations. The nonlinear nature of Cosserat rod formulations typically necessitates the use of iterative solvers and careful linearization strategies, which can be computationally intensive and thus unsuitable for real-time applications, such as control or interactive simulation.
%
\subsection{Contributions and outline}
This paper studies the modeling of Cosserat continuum rods for applications in soft robotics. It introduces a general Lie group modeling framework based on the cumulative parameterization of Lie group $SE(3)$, enabling a representation that is both geometrically consistent and computationally structured.  To the authors' knowledge, this study is the first to apply cumulative Lie group parameterization to continuum robot modeling within the field of soft robotics.

Building on this formulation, the paper derives analytical expressions for kinematics, statics, and dynamics. It extends the method to handle complex multi-rigid-soft coupled robotic systems. The framework is designed to offer an alternative to existing strain- or configuration-based models, aiming to enhance modeling flexibility and numerical behavior in soft robotic systems. 

The main academic contributions of this work are summarized as follows:
\begin{enumerate}
  \item We propose a novel deformation representation for soft continuum robots based on cumulative parameterization of \( SE(3) \). Unlike existing IGA-based methods, this approach eliminates the need for unit quaternion constraints. Compared to strain-based parameterizations, it provides geometric local control, thereby facilitating seamless integration of design and modeling within a unified isogeometric computer-assisted design (CAD) and simulation framework.

  \item We derive unified analytical expressions for kinematics, statics, and dynamics, including a recursive Jacobian computation and a symmetric, energy-conserving integrator suitable for real-time simulation and control.

  \item We extend the proposed modeling approach to encompass generalized structures, including branched tree-like configurations, concentric tube assemblies, and rigid–flexible coupled robotic systems with closed-loop constraints, thereby supporting a unified workflow and enabling a modular modeling strategy.

  \item We demonstrate the generality and computational efficiency of the method through representative scenarios, including a large deformation rod, a concentric tube robot, a parallel robot, cable-driven robots, and an articulated finger.
\end{enumerate}

The structure of this paper is organized as follows. Section~\ref{sec.cosserat} reviews the Cosserat rod theory and introduces the research problems. In Section \ref{sec.cp}, we present the cumulative parameterization and demonstrate its application on the ${SE}(3)$ group. Section~\ref{sec.kin} discusses the kinematics of the cumulatively parameterized Cosserat rod and complex multi-rod structures, followed by Section~\ref{sec.sta}, which addresses the statics. Section~\ref{sec.dyna} extends the discussion to the dynamics of the cumulatively parameterized Cosserat rod. Section~\ref{sec.ns} and Section~\ref{sec.applications} present a series of numerical simulations to evaluate the performance of the proposed model. Finally, Section~\ref{sec.con} concludes the paper with a summary of the findings and suggestions for future research.

Table \ref{tab:notation} presents all the mathematical symbols, along with their definitions, used in this article.
%
\begin{table}[htbp]
  \caption{Notation and definitions.}
  \centering
  \renewcommand{\arraystretch}{1.5}
  \begin{tabularx}{\linewidth}{llX}
    \toprule
    \textbf{Symbol} & \textbf{Unit} & \textbf{Definition} \\
    \midrule
    $\dot{}$ & –& Derivative with respect to time\\
    $'$ & –& Derivative with respect to space\\
    $\widehat{}$ & –&  Converts $\mathbb{R}^6$ in $\mathfrak{se}(3)$ \\
    $s$ & m & Reference arclength coordinate \\
    $t$ & s & Time \\
    $ p(s,t)$ & m & Position vector of cross-section \\
    $ R(s,t)$ & – & Rotation matrix of cross-section \\
    $ g(s,t)$ & – & Homogeneous transform in $SE(3)$ \\
    $T_{g}\mathcal M$ & – & The tangent space at $g$\\
    $\mathrm{d}L_{g}$& – & Differential of left multiplication\\
    $L$ & m & Total length of soft manipulator \\
    $\hat{\xi}(s,t)$ & – & Strain twist vector: $g^{-1}g'$ \\
    $\hat{\eta}(s,t)$ & – & Velocity twist vector: $g^{-1}\dot{g}$ \\
    $\operatorname{Ad}_g$ & – & Lie Group adjoint operators \\
    $\operatorname{ad}_\xi, \ \operatorname{ad}_\eta$ & – & Lie algebra adjoint operators \\
    $\xi_0(s)$ & – & Reference strain field \\
    $\Phi(s)$ & – & Shape function matrix \\
    $ J(q,s)$ & – & Kinematic Jacobian matrix \\
    $g_i$ & – & Configuration of ith control point  \\
    $\eta_i$ & – & Velocity of ith control point  \\
    $\mathcal{Q}$ & – &Set of control point configuration: $\{ g_i\}$ \\
    $\dot{q}$ & – &Set of velocities: $[\eta_1^\top,\ldots,\eta_N^\top]^\top$ \\
    $B_i(s)$ & – & Basis function (e.g., B-spline) \\
    $\bar B_i(s)$ & – & Cumulative basis: $\sum_{j=i}^n B_j$ \\
    $\Psi_i$ & – & Lie increment: $g_i \boxminus g_{i-1}$ \\
    $\tau(\cdot)$ & – & Retraction map (e.g., $\exp$, $\mathrm{cay}$) \\
    $\mathrm{d}\tau$ & – & Left-trivialized tangent of retraction map (e.g., $\mathrm{dexp}$, $\mathrm{dcay}$) \\
    $\boxplus,\ \boxminus$ & – & Group operators with retraction \\
    $\mathrm d\tau_x$ & – & Differential operators of retraction \\
    $\varepsilon$ & – & Strain of deformation: $\xi - \xi_0$ \\
    $\Lambda_i$ & – & Internal wrench \\
    $\Lambda_e$ & – & External wrench (e.g. gravity) \\
    $\mathcal K$ & – & Stiffness matrix \\
    $\mathcal M$ & – & Mass matrix per unit length \\
    $ b(Q)$ & – & Constraint function vector \\
    $ A(Q)$ & – & Constraint Jacobian \\
    $ r(Q)$ & – & Gradient of energy \\
    $\sigma$ & – & Lie algebra update increment \\
    \bottomrule
  \end{tabularx}
  \label{tab:notation}
\end{table}
%
\section{Cosserat Rod Theory}\label{sec.cosserat}
%
In this Section, we briefly review the fundamental theory of the Cosserat rod, along with two spatial discretization approaches: strain-based parameterization and the configuration-based parameterization method. This overview provides the foundation and motivation for the present work.
%
\subsection{Continuous formulation}
%
In the Cosserat framework, a soft manipulator is considered as a set of rigid cross-sections along its centroid line, as shown in Figure~\ref{introduction}. The configuration of a deformable rod with respect to the inertial frame at a specific time is characterized by a position vector ${p}(s,t)\in\mathbb{R}^3$ and a rotational matrix ${R}(s,t)\in SO(3)$, parameterized by the abscissa material $s\in [0,L]$ along the manipulator, where \( s \) denotes the coordinate along the rod axis and \([0,L] \subset \mathbb{R}\) represents the reference domain of the length of the rod. Thus, the configuration of any cross-section of the soft rod can be defined as a curve ${g}(s, t): s\mapsto g(s)\in SE(3)$ and time $t\in [0, \infty]\subset \mathbb{R}$ with the homogeneous transformation matrix:

\begin{definition}[Homogeneous transformation matrix]
The homogeneous transformation matrix for any cross section along the soft manipulator can be defined as $$\forall s\in [0, L],\ {g}(s,t)=\begin{bmatrix}
	{R}(s,t)&{p}(s,t)\\
	{0}&1
\end{bmatrix}\in SE(3)$$
where ${p}(s,t)\in\mathbb{R}^3$ is the position vector and ${R}(s,t)\in SO(3)$ represents an orthonormal rotation matrix. $L$ is the total arc length of the soft manipulator.
\end{definition}
\begin{figure}[t]
	\centering
	\includegraphics[width=0.45\textwidth]{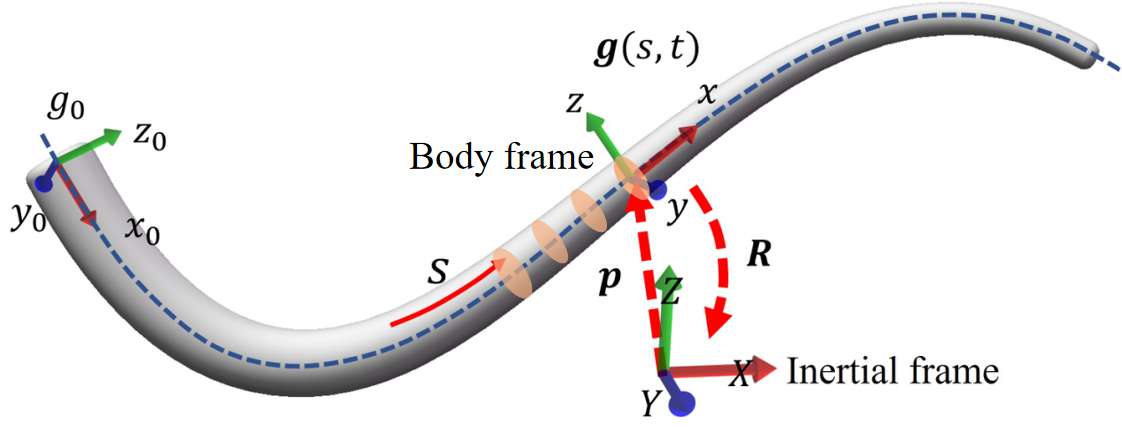}
	\caption{\textbf{Schematic diagram of the soft manipulator.}}
	\label{introduction}
\end{figure}

Subsequently, the strain and velocity are defined in terms of the tangent space of the homogeneous transformation matrix.
\begin{definition}[Strain and velocity of Cosserat rod]
	The strain and velocity in the body frame can be regarded as the left-trivialized tangent space of the homogeneous transformation matrix \textit{w.r.t.} space and time, i.e.,
	$$\begin{aligned}
		\mathrm{Strain}:\quad\hat{{\xi}}(s,t)&={g}^{-1}{g}^{\prime}\in \mathfrak{se}(3)\simeq \mathbb{R}^6,\\  \mathrm{Velocity}:\quad\hat{{\eta}}(s,t)&={g}^{-1}\dot{{g}}\in \mathfrak{se}(3)\simeq \mathbb{R}^6
	\end{aligned}$$
\end{definition}

For simplicity, we use ${(\cdot)^\prime}$ to denote the partial derivative \textit{w.r.t} space $\partial/\partial s$ and $\dot{(\cdot)}$ to denote the partial derivative \textit{w.r.t.} time $\partial/\partial t$. 

The continuous dynamic model of the Cosserat rod is governed by the following partial differential Poincaré equation~\cite{8500341}:
\begin{equation}\label{pde0}
	{\mathcal{M}}\dot{{\eta}}-ad_{{\eta}}^{\top}{\mathcal{M}}{\eta}={\Lambda}^{\prime}-\operatorname{ad}_{{\xi}}^{\top} {\Lambda}_{i}+\bar F
\end{equation}
where ${\mathcal{M}}\in \mathbb{R}^{6\times6}$ is the tensor of mass linear density along the central axis, ${\Lambda}\in \mathbb{R}^{6}$ is the elastic internal wrench. 
$\bar F\in \mathbb{R}^{6}$ is the external distributed load (e.g., contact load and gravity).  

Based on the above concepts, the continuous kinematic and dynamic systems of the Cosserat rod in PDE (partial differential equation) form are as follows~\cite{9985417}.
\begin{equation}\label{continouspde}
\left(
\begin{array}{c}
g' \\
\eta' \\
\dot{\eta}' \\
\Lambda'
\end{array}
\right)
=
\left(
\begin{array}{c}
g \hat{\xi} \\
-\text{ad}_{\xi} \eta + \dot{\xi} \\
-\text{ad}_{\xi} \dot{\eta} - \text{ad}_{\dot{\xi}} \eta + \ddot{\xi} \\
\text{ad}^T_{\xi} \Lambda + \mathcal{M} \dot{\eta} - \text{ad}^T_{\eta} \mathcal{M} \eta - \bar{F}
\end{array}
\right),
\end{equation}
with BCs:
\begin{equation}
(g, \eta, \dot{\eta})(0) = (1_{4 \times 4}, 0, 0), \quad \Lambda(L) = F_+.
\tag{62}
\end{equation}
for a manipulator, or:
\begin{equation}
\Lambda(0) = -F_-, \quad \Lambda(L) = F_+.
\tag{63}
\end{equation}
for a locomotor. Where $F_-$ and $F_+$ denote the concentrated external force applied on the front and rear ends of the rod.

To enable real-time solutions of the above PDE system, a common strategy is first to discretize the Cosserat rod spatially and introduce a parameterization. This is typically done using a set of parameters ${q}$, which serve as the generalized coordinates of the Cosserat rod. The original PDE system is then reformulated as an ODE system in time with respect to ${q}$, which is similar to the dynamics of rigid-body robotics. Alternatively, one can derive the time-dependent ODE system directly using Lagrangian mechanics based on the parameterized formulation. 

There are generally two approaches to spatially discretizing and parameterizing the Cosserat rod: strain-based parameterization and configuration-based parameterization. These methods will be introduced in the following sections.
%
\subsection{Strain-based parameterization}
%
In the strain-based method, the configuration of a Cosserat rod is described implicitly through its spatial strain field~\cite{boyer2020dynamics}.

The continuous kinematics model is described by the first two differential equations in~\eqref{continouspde}.
Because the Lie group is non-commutative, given the strain field, an exact analytical solution to \eqref{continouspde} cannot be obtained. Consequently, we must resort to numerical methods to compute the pose of any point on the soft rod from the strain field.
For numerical implementation, the continuous strain field is decomposed as
\[
\xi(s,t)=\xi_0(s)+\Phi(s)\,q(t),
\]
in which \(\xi_0(s)\) denotes the natural (initial) strain distribution, \(\Phi(s)\) is a matrix of shape functions, and \(q(t)\in\mathbb{R}^n\) is the vector of generalized coordinates. 

From the fist equation of~\eqref{continouspde}, the configuration can be implicitly derived using the Magnus expansion (see \cite{9057619,mathew2025reduced}).
From the second equation of~\eqref{continouspde}, the kinematic relationship is derived as:
\[
\eta(s) = J(q, s)\, \dot{q},
\]
where \( J \) denotes the kinematic Jacobian matrix.

Strain-based parameterization offers several attractive advantages for modeling soft robots. By representing the rod’s configuration through its spatial strain field, it allows computations and interpolations to be carried out in a linear Euclidean space, avoiding the singularities and coordinate complexities associated with direct rotation parameterizations on nonlinear manifolds such as $SO(3), \ SE(3)$. This formulation also integrates seamlessly with constitutive modeling and lends itself well to reduced-order approximations. 

However, strain-based modeling methods face certain challenges. Since the configuration is implicitly defined through the integration of the strain field, numerical errors can accumulate along the arc length—particularly in long and highly deformed structures.
Additionally, the approach introduces a form of geometric asymmetry, i.e., strain near the base has a more pronounced effect on the distal configuration than vice versa, leading to globally coupled system dynamics. As a result, the associated mass matrix tends to be dense, unlike the typically sparse or diagonal matrices seen in conventional finite element methods. 

Futhermore, in certain soft robotic systems, the configuration of the robot body may be continuous, while the strain field is not. This typically occurs at structural discontinuities, such as tendon anchor points in tendon-driven manipulators or at the nested end faces of concentric tubes in concentric tube robots. At these locations, strain discontinuities are inherent to the design. Modeling such structures using variable parameterization methods can become challenging and complex.

%
\subsection{Configuration-based parameterization}
%
In contrast to the strain-based approach, configuration-based methods directly parameterize the pose of the Cosserat rod \( g(s) \in SE(3) \), using translation and rotation variables as primary unknowns. This formulation eliminates the need to integrate the kinematic differential equation in~\eqref{continouspde} and enables more explicit control of the rod's geometry, which is particularly advantageous in motion planning and control applications.

Two parameterization strategies are most commonly used in practice: unit quaternions and exponential maps.
In the quaternion-based formulation, the configuration of the rod is described by the centerline position \( r(s) \in \mathbb{R}^3 \) and a unit quaternion \( Q(s) \in \mathbb{S}^3 \), which encodes the orientation of the local material frame. This representation provides a globally valid and singularity-free description of 3D rotations, allowing for smooth interpolation via spherical linear interpolation (slerp) or geodesic curves~\cite{sander2010geodesic} on the unit quaternion manifold. However, implementing and deriving such interpolations for arbitrary polynomial degrees is often complex and computationally intensive.

To address these issues, a simplified quaternion interpolation scheme has been proposed~\cite{weeger2017isogeometric}.
In this approach, unit quaternions are treated as standard vectors, and interpolation is carried out directly in this Euclidean space:
$$Q(s)=\Phi(s)q(t), \qquad 
Q(s)\in\mathbb{R}^4$$
Normalization is then enforced only at the interpolation nodes, reducing the overall computational complexity. However, because this method does not inherently respect the geometry of the unit quaternion manifold, a dense set of interpolation points is required to maintain accuracy; that is, the interpolation step size must be small. This results in a high-dimensional system, making the approach less suitable for real-time simulation and control of soft robots. Moreover, the quaternion formulation introduces a unit norm constraint, which must be preserved through renormalization, Lagrange multipliers, or projection methods~\cite{lang2011multi}.

An alternative strategy is to represent rotation using the exponential map~\cite{tasora2020geometrically, im2020geometrically} from the Lie algebra \(\mathfrak{so}(3)\). In this formulation, the orientation is parameterized by a rotation vector \({\phi}(s) \in \mathbb{R}^3\), typically expressed as \({\phi}(s) = \Phi(s)q(t)\), where \(\Phi(s)\) is a set of spatial shape functions and \(q(t)\) denotes the generalized coordinates. The corresponding rotation matrix is then recovered via
\[
R(s) = \exp\bigl(\tilde\phi(s)\bigr),
\]
where \(\tilde{\phi}(s)\) is the representation of the skew matrix of \({\phi}(s)\).

This approach avoids the unit-norm constraint required by quaternion-based methods and operates entirely within a linear space, which simplifies interpolation and numerical implementation. It is particularly well-suited for optimization-based algorithms and high-fidelity finite element formulations. However, special attention is needed when dealing with large rotations. The exponential map is inherently nonlinear and is typically defined over a principal domain (e.g., \( [-\pi, \pi] \)), which can introduce challenges for smooth interpolation. When the rotation angle approaches the boundaries of this domain, discontinuities such as sudden transitions from \( \pi \) to \( -\pi \) may occur, potentially leading to artifacts in motion interpolation and numerical instability.

%
\subsection{Challenges and motivation}
In summary, 
critical challenges still lies in developing an efficient parameterization strategy for soft slender robot that:
\begin{enumerate}
\item Explicitly defines configurations in a computationally efficient manner;
\item Ensures geometric consistency without introducing singularities or dense matrices;
\item Ensures geometric consistency without introducing singularities or dense matrices;
\item Supports the modeling of structures with discontinuous strain fields.
\item Provides robustness against numerical instabilities.
\end{enumerate}

To achieve the objectives outlined above, we propose a Cosserat rod modeling approach based on cumulative parameterization, specifically tailored for soft robotic applications. This method combines the robustness and incremental stability of strain-based formulations with the flexibility of explicit configuration representation and localized control on Lie groups. The details of this approach are presented in the following sections.

%
\section{Cumulative Parameterization}\label{sec.cp}
%
The method proposed in this paper performs explicit cumulative parameterization on the Lie group manifold $SE(3)$. To clearly illustrate this process, before extending the concept to the Lie group manifold $SE(3)$, we first introduce the concept of cumulative parameterization in the context of linear Euclidean space.
%
\subsection{Standard parameterization}
%
We begin by considering the general framework of polynomial interpolation in linear Euclidean space. To make the discussion concrete, we first specify the fundamental data that define the interpolant: a finite set of control points in \(\mathbb{R}^d\) and a corresponding set of scalar basis functions defined over a parameter domain.

Let the set of control points be
$$\mathcal{Q} = \{ q_i \in \mathbb{R}^{d} \mid i = 0,1,\dots,n \},
$$
and the associated family of scalar basis functions be
$$
\mathcal{B} = \{ B_i(s) : \Omega \rightarrow \mathbb{R} \mid i = 0,1,\dots,n \},
$$
where \(\Omega \subset \mathbb{R}\) denotes the parameter domain. The basis functions \(B_i(s)\) may take various forms (e.g., polynomials, splines, or radial functions), depending on the specific application.

The parametric curve \(C(s)\) is then typically represented in the standard form:
\begin{equation}
C(s) = \sum_{i=0}^{n} B_i(s)\, q_i.
\end{equation}

While the expression above is compact, it obscures how local modifications of the control points propagate along the curve. To make these sequential contributions explicit, we recast the interpolation in a cumulative form that replaces absolute control points with their successive increments and accumulates basis-function weights from the tail of the sequence forward. 
%
\subsection{Cumulative parameterization}
%
To enhance interpretability and support incremental analysis, we introduce an alternative yet equivalent formulation of the standard parameterization. We begin by defining the incremental differences between consecutive control points:
\begin{equation}
\Delta q_i = q_i - q_{i-1}, \quad i = 1, \dots, n,
\end{equation}
and construct the cumulative basis functions as backward partial sums of the original basis functions:
\begin{equation}
\bar{B}_i(s) = \sum_{j=i}^{n} B_j(s), \quad i = 0, \dots, n.
\end{equation}

This leads to a parameterization in the cumulative form~\cite{qin1998general}, which we formally define as follows:
\begin{definition}[Cumulative Parameterization]
The cumulative representation of the parametric curve \(C(s)\) is given by:
\begin{equation}\label{cmq}
C(s) = \bar{B}_{0}(s)\, q_{0} + \sum_{i=1}^{n} \bar{B}_{i}(s)\, \Delta q_{i}.
\end{equation}
\end{definition}

The incremental viewpoint developed above naturally extends to configuration spaces that are not linear, provided we replace vector addition with the appropriate group operation. In particular, the kinematics of a Cosserat rod is most conveniently expressed on the Lie group \(SE(3)\), which couples rotations and translations. By transporting the cumulative idea from \(\mathbb{R}^d\) to \(SE(3)\) via local coordinates and retraction maps, we obtain a manifold-compatible formulation in which each control-point increment lives in the Lie algebra and accumulates through group composition. The following subsection formalizes this generalization.
%
\subsection{Cumulative parameterization on Lie group}
%
Let us begin by noting that, in a Lie group, formula~\eqref{cmq} has no direct equivalent, as rotations cannot be subtracted directly without violating the Lie group structure. In other words, it is necessary to define a local operation on the Lie group manifold. 

When \( \mathcal{M} \) is a Lie group (e.g., \( SO(3) \), \( SE(3) \)), the tangent space is endowed with a Lie algebraic structure, denoted as \( \mathfrak{g} \). Typically, the retraction map is employed to map elements from the Lie algebra to the Lie group. 
\begin{definition}[Retraction map]
	The \textit{retraction map} \( \tau: \mathfrak{g} \to \mathcal{G} \) is a \( C^2 \)-diffeomorphism around the origin such that \( \tau(0) = e \). 
\end{definition}

There exist two retraction maps for Lie groups: the exponential map and the Cayley map. The exponential map is defined as below: 
\begin{equation}
	\mathrm{exp}(x) = \sum_{j=0}^{\infty} \frac{x^j}{j!}
\end{equation}

Specifically, for the $SO(3)$ and $SE(3)$ groups, these maps can be systematically simplified using the Rodrigues formula. The detailed expression is provided in the Appendix for reference.

The Cayley map is defined more simply:
\begin{equation}
	\mathrm{cay}(x) = \Bigl(\mathbb{I}-\frac{x}{2}\Bigr)^{-1}\Bigl(\mathbb{I}+\frac{x}{2}\Bigr)
\end{equation}
with $\mathbb{I}$ denoting the identity matrix.

Through the series expansion, we observe the equivalence between the Cayley map and the exponential map when $x$ is sufficiently small, i.e., $\mathrm{cay}(x)=\mathrm{exp}(x)+O(x^3)$.

Based on the retraction map, the local operations can be further defined on the Lie group manifold. We defined two local operators as follows:
\begin{definition}[Local operator]
The existence of homeomorphisms around any point on a Lie group $\mathcal{G}$ permits the definition of two local operators, $\boxplus$ and $\boxminus$, which encapsulate local operations on the manifold~\cite{hertzberg2013integrating}:
\begin{equation}
	\begin{aligned}
		&\boxplus: \mathcal{G} \times \mathbb{R}^n \rightarrow \mathcal{G}, \quad {g}_1 \boxplus {v} = {g}_1 \tau({v})\\
		&\boxminus: \mathcal{G} \times \mathcal{G} \rightarrow \mathbb{R}^n, \quad  {g}_2 \boxminus {g}_1 = \tau^{-1}({g}_1^{-1} {g}_2)
	\end{aligned}
\end{equation}
\end{definition}

The physical interpretation of ${g}_2 = {g}_1 \boxplus_{\mathcal{G}} {v}$ is adding a small perturbation ${v} \in \mathbb{R}^n$ to ${g}_1 \in \mathcal{G}$. And the inverse operation ${v} = {g}_2 \boxminus_{\mathcal{G}} {g}_1$ 
determines the perturbation $v$ which, when it is added to $g_1$, yields $g_2$ \cite{murray2017mathematical}.

The retraction map-based local operation allows for an explicit definition of cumulative parameterization on the Lie group manifold. A B-spline-based quaternion cumulative parameterization was first introduced in~\cite{kim1995general}. In this work, we generalize this concept to a general Lie group formulation.

\begin{definition}[Cumulative parameterization of Lie group] 
Consider a set of control points:
\begin{equation}
\mathcal{Q} = \{ g_i \in SE(3) \mid i = 0,1,\dots,n \},
\end{equation}
and an associated family of scalar basis functions:
\begin{equation}
\mathcal{B} = \{ B_i(s) : \Omega \rightarrow \mathbb{R} \mid i = 0,1,\dots,n \},
\end{equation}
defined on the parameter domain \( \Omega \subset \mathbb{R} \). Defining the incremental differences between consecutive control points on the Lie group manifold as:
$$\Psi_i={g}_i\boxminus{g}_{i-1}, $$
the cumulative representation corresponding to the parametric curve \(g(s)\) on Lie group manifold is given by:
\begin{equation}\label{gg}
    {g}(s) = \bar{B}_{0}(s){g}_0\boxplus\sum_{i=1}^{n}\bar{B}_{i}(s)\Psi_i
\end{equation}
\end{definition}
with the manifold-compatible cumulative formulation in place, we now illustrate how it reproduces two widely used spline families that are foundational in computer graphics and geometric modelling.

First, we revisit the ubiquitous B-spline, demonstrating that its piecewise polynomial nature aligns naturally with the cumulative framework. Subsequently, we derive the cubic Hermite curve as a special case.
%
\subsection{B-spline approximation}
%
A B-spline of order $k+1$ is a piecewise polynomial function of degree $k$ in a variable $s$. It is defined by the basis functions $B_{i,k}(s)$ and $n+1$ control points ${q}_i$, $i=0,\dots,n$. The values of $s$ where the pieces of the polynomial meet are known as knots, denoted $\{s_{0},s_{1},s_{2},\dots, s_{k+n}\}$ and sorted into nondecreasing order. Then, the general formulation of B-spline ${C}$ defined as above is given by:
\begin{equation}
{C}(s)=\sum_{i=0}^{n}B_{i,k}(s){q}_i
\end{equation}
where basis function $B_{i,k}(s)$ is calculated from $B_{i,0}(s)$ to $B_{i,k}(s)$ by the recursive formula below:\\
for $d=0$ and $ 0\leq i\leq k+n$ , $$B_{i,0}(s)=\left\{\begin{aligned}
	&1 \ \ \mathrm{if} \ s_i<s<s_{i+1}\\
	&0 \ \ \mathrm{else}
\end{aligned}
\right.$$
for $d>0$ and $0\leq i\leq k+n-d$ ,
$$B_{i,d}(s)=\frac{s-s_i}{s_{i+d-1}-s_i}B_{i,d-1}+\frac{s_{i+d}-s}{s_{i+d}-s_{i+1}}B_{i+1,d-1}$$

For convenience, we denote \( B_{i,k} \) as \( B_i \) and \( \bar{B}_{i,k} \) as \( \bar{B}_i \).
The B-spline curve may be reformulated in the following cumulative form:
\begin{equation}\label{ddds}
C(s) = \bar{B}_{0}(s)\,q_{0} + \sum_{i=1}^{n} \bar{B}_{i}(s)\,\Delta q_{i}.
\end{equation}

By leveraging the properties of B-spline basis functions, \eqref{ddds} can be further simplified.  
Suppose \( s \) lies within the spline segment \( [s_j, s_{j+1}] \). The basis functions \( \bar{B}_i(s) \) then exhibit the following properties, as illustrated in \Cref{bsplinec}:  
\begin{itemize}
    \item \( \bar{B}_i(s) = 1 \) for \( i \leq j \),
    \item \( \bar{B}_i(s) = 0 \) for \( i > j + k - 1 \).
\end{itemize}

These observations allow \eqref{ddds} to be simplified as follows. For \( s \in [s_j, s_{j+1}] \),
\begin{equation}
    C(s) = q_{j-k+1} + \sum_{i=j-k+2}^{j} \bar{B}_{i}(s)\,\Delta q_{i}.
\end{equation}

Applying this formulation to the Lie group \( {SE}(3) \), we obtain the following expression for the configuration \( {g}(s) \) over the interval \( s \in [s_j, s_{j+1}] \):
\begin{equation}
    {g}(s) = {g}_{j-k+1} \boxplus \sum_{i=j-k+2}^{j} \bar{B}_{i}(s)\, \Psi_i.
\end{equation}

\begin{remark}\label{rm1}
This result shows that for B-splines defined on a Lie group manifold, the value at any point can be efficiently evaluated using only \( k - 1 \) local operations. Specifically, for \( s \in [s_j, s_{j+1}] \), the Lie group B-spline curve \( {g}(s) \) depends solely on a limited set of parameters: \( {g}_{l-k+1}, \Psi_{l-k+1}, \ldots, \Psi_l \), corresponding to the \( k \) control points \( {g}_{l-k+1}, {g}_{l-k}, \ldots, {g}_l \). Consequently, any modification to the control point \( {g}_l \) affects the curve only within the segment \( [s_l, s_{l+k}] \). 
\end{remark}
\begin{figure}[t]  
    \centering
    \begin{subfigure}[t]{0.38\textwidth}
        \centering
        \includegraphics[width=\linewidth]{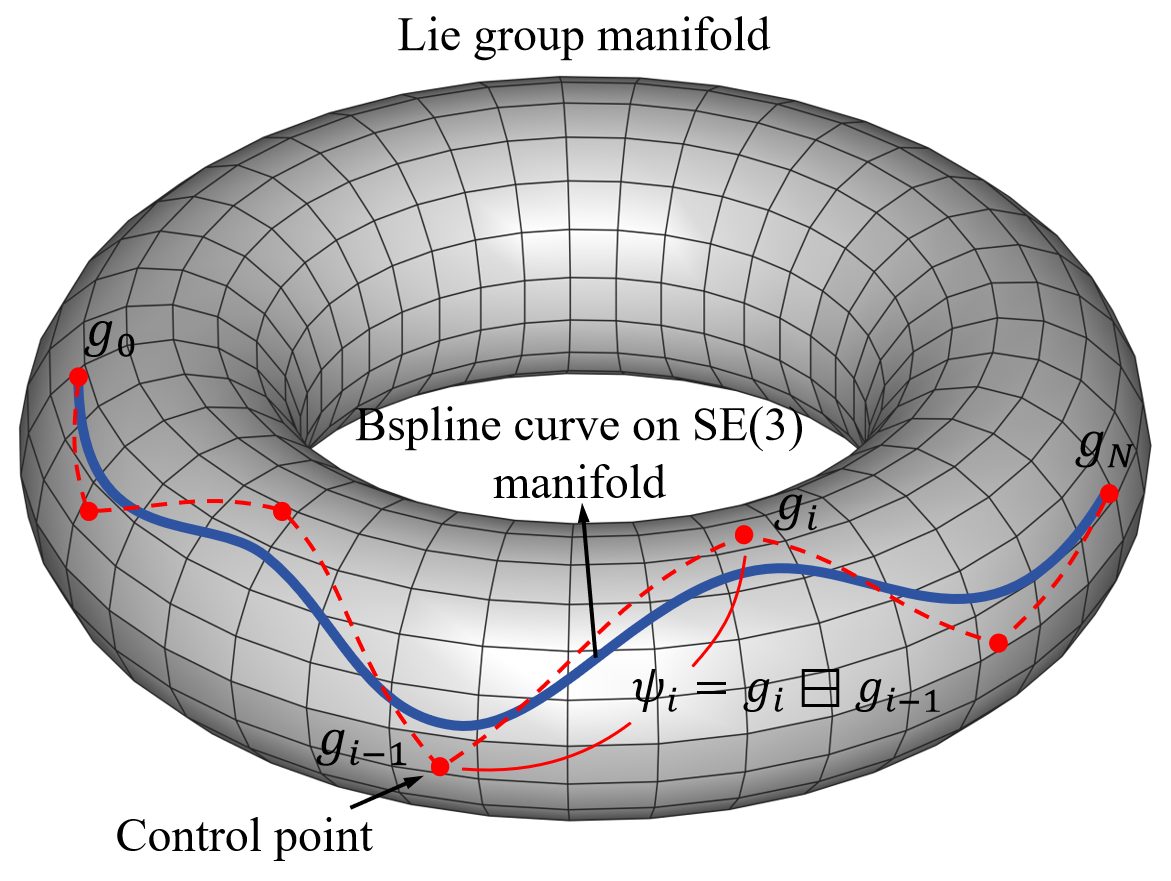}
        \caption{Bspline curve on SE(3) manifold.}
        \label{Liebspline}
    \end{subfigure}
    \hfill
    \begin{subfigure}[t]{0.38\textwidth}
        \centering
        \includegraphics[width=\linewidth]{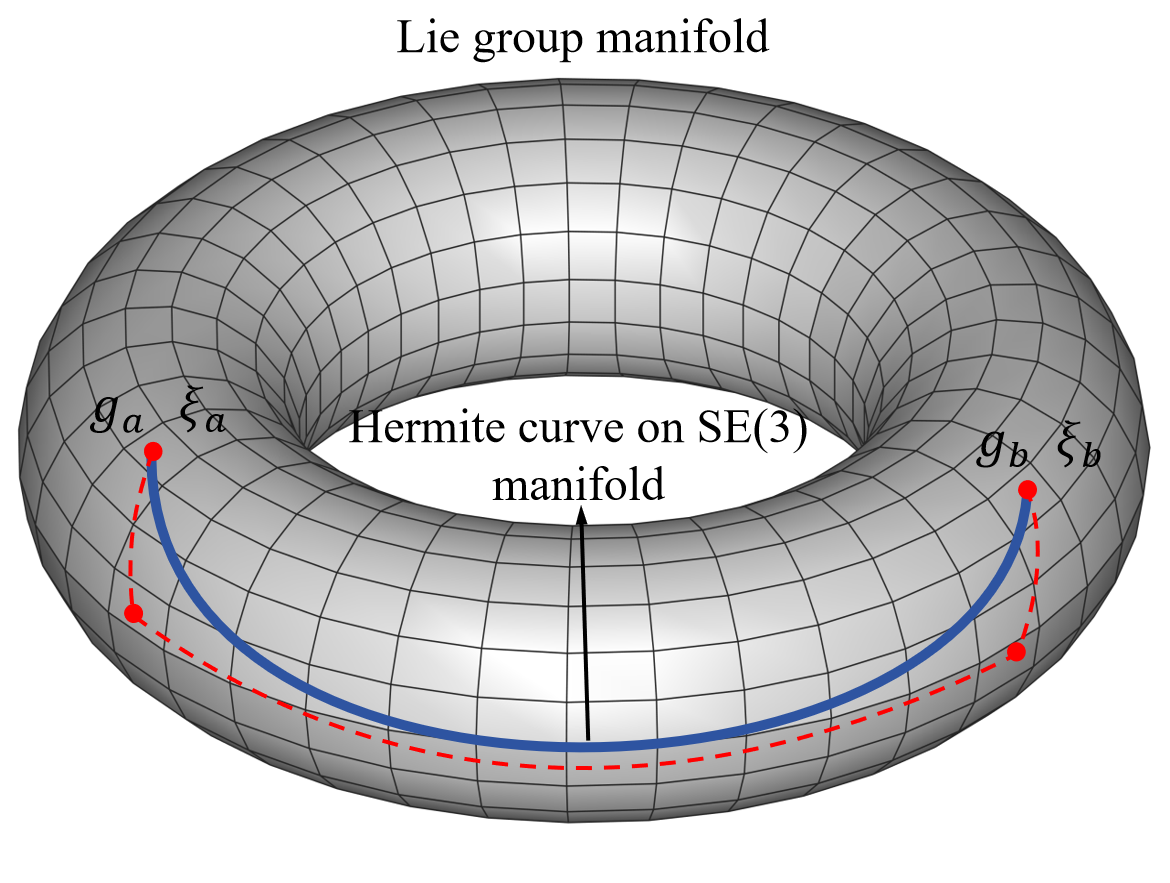}
        \caption{Hermite curve on SE(3) manifold.}
        \label{Liehermite}
    \end{subfigure}    
    \caption{\textbf{Bspline and Hermite curve on $SE(3)$ manifold.}}
    \label{fig:main}
\end{figure}

\Cref{Liebspline} illustrates a parameterized Lie group B-spline curve defined by the control elements $g_0, \ldots, g_N$ on the Lie group. Since Lie group manifolds cannot be intuitively visualized in three-dimensional space, a torus is used in this context as a metaphor for the ${SE}(3)$ manifold.
\begin{remark}
It is worth noting that when \( k = 1 \), the first-order incremental Lie group B-spline is equivalent to performing shortest geodesic interpolation (slerp) between each pair of Lie group control points. If one uses this to parameterize a Cosserat rod configuration, the interpolated segments maintain constant strain, which corresponds to the piecewise constant strain (PCS) approach \cite{8500341}.
\end{remark}

\begin{figure}[!h]
	\centering
	\includegraphics[width=0.48\textwidth]{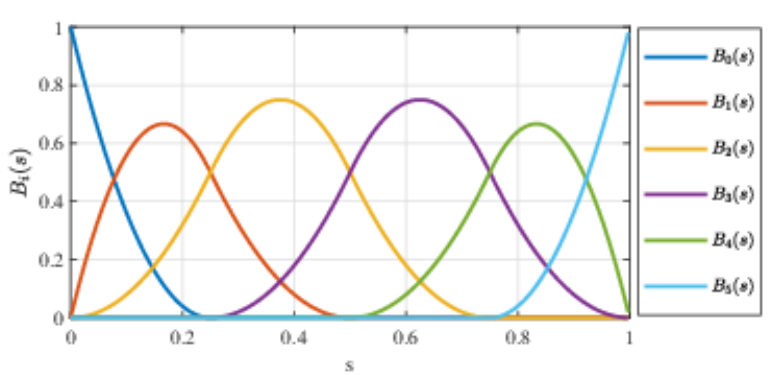}
	\caption{\textbf{B-spline basis functions.}}
	\label{bspline}
\end{figure}

\begin{figure}[htbp]  
    \centering
    \begin{subfigure}[t]{0.48\textwidth}
        \centering
   \includegraphics[width=\linewidth]{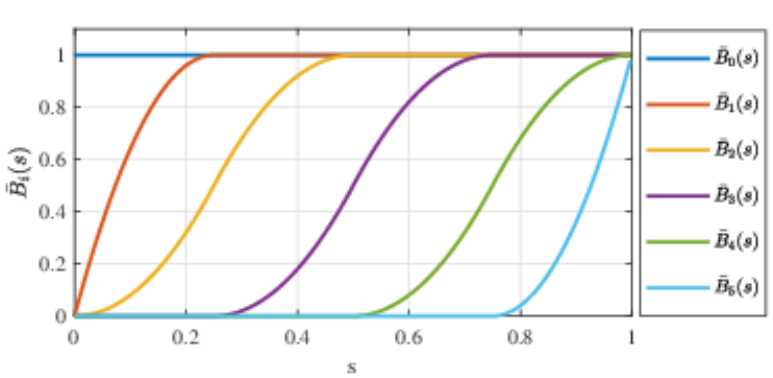}
        \caption{Cumulative B-spline basis functions.}
        \label{bsplinec}
    \end{subfigure}
    \hfill
    \begin{subfigure}[t]{0.48\textwidth}
        \centering
        \includegraphics[width=\linewidth]{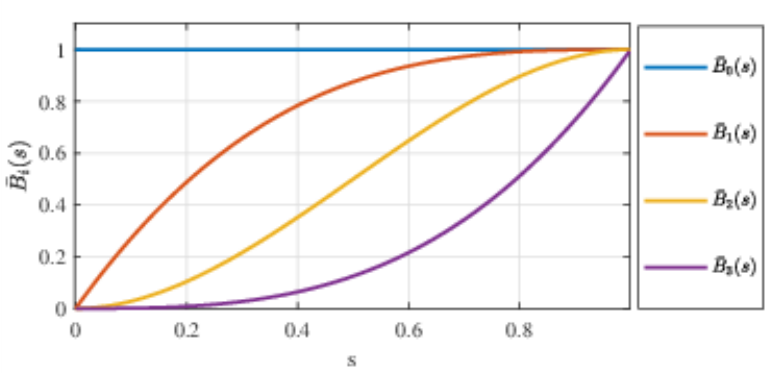}
        \caption{Cubic Hermite cumulative basis functions.}
        \label{hermit}
    \end{subfigure}    
    \caption{\textbf{Cumulative basis functions.}}
    \label{fig:main2}
\end{figure}
%
\subsection{Hermite interpolation}
%
A cubic Hermite curve is defined by two end points, $p_a$ and $p_b$, and two end velocities, $v_a$ and $v_b$. Alternatively, the Hermite curve can be represented as a cubic B-spline curve with the knots \( \{0, 0, 0, 0, 1, 1, 1, 1\} \), given by:
\begin{equation}\label{hermitb}
p(t) = \sum_{i=0}^{3} p_i \, \bar{B}_i(t),
\end{equation}
with the condition:
\begin{equation}
p_0 = p_a, \, p_1 = p_a + v_a /3, \, p_2 = p_b - v_b /3, \, p_3 = p_b
\end{equation}
where $\bar B_i(t) = \bar B_{i,3}(t)$ for $i = 0, 1, 2, 3$.

Similarly, a cubic B-spline $SE(3)$ curve can be used to define a Hermite $SE(3)$ curve which interpolates two end unit Lie group elements, $g_a$ and $g_b$, and two end strains, $\xi_a$ and $\xi_b$. From~\eqref{hermitb}, the cubic B-spline $SE(3)$ curve is given by:
\begin{equation}\label{cubsb}
    {g}(s) = \bar{B}_{0}(s){g}_0\boxplus\sum_{i=1}^{3}\bar{B}_{i}(s)\Psi_i
\end{equation}
The control points of~\eqref{cubsb} are then given by:
\begin{equation}
g_0 = g_a, \,
g_1 = g_a \boxplus\frac{\xi_a}{3}, \,
g_2 = g_b \boxplus\frac{-\xi_b}{3}, \,
g_3 = g_b.
\end{equation}

These four identities determine the three coefficients $\Psi_i$ of the cubic $SE(3)$ curve in~\eqref{cubsb} as follows:
\begin{align*}
\Psi_1 &= g_1\boxminus g_0 = \bigl(g_a \boxplus\frac{\xi_a}{3}\bigr)\boxminus g_a= \frac{\xi_a}{3}, \\
\Psi_2 &= g_2\boxminus g_1 = \bigl(g_b \boxplus\frac{-\xi_b}{3}\bigr)\boxminus \bigl(g_a \boxplus\frac{\xi_a}{3}\bigr), \\
\Psi_3 &= g_3\boxminus g_2 = g_b\boxminus\bigl(g_b\boxplus\frac{-\xi_b}{3}\bigr) = \frac{\xi_b}{3}.
\end{align*}
\Cref{Liehermite} illustrates a parameterized Lie group B-spline curve defined by two Lie group elements, $g_a$ and $g_b$, and two Lie algebra elements, $\xi_a$ and $\xi_b$.

To demonstrate the generality of the cumulative-form framework, Table~\ref{tab:cumulative_basis} summarizes several commonly used nonlinear parameterization schemes, along with their corresponding cumulative basis functions. These include polynomial (power) bases, B-spline functions and Hermite bases. In each case, the cumulative basis is constructed via backward summation over the original basis set. \Cref{bspline} shows the basis functions of a second-order B-spline curve with an open knot vector $\{0, 0, 0, 1/4, 1/2, 3/4, 1, 1, 1\}$.
\Cref{bsplinec} shows its corresponding cumulative basis functions. \Cref{hermit} shows the cumulative basis functions of the Hermite curve.

\begin{remark}
The cumulative basis can, in fact, be derived using a family of piecewise smooth transition functions \( S(s) \), which encode local incremental behavior. These functions typically take the form:
\begin{equation}
\bar B_i(s) =
\begin{cases}
0, & s < s_{i-1}, \\
1, & s \geq s_i, \\
S\left(\frac{s - s_{i-1}}{s_i - s_{i-1}}\right), & s_{i-1} \leq s < s_i,
\end{cases}
\end{equation}
where \(S(\cdot)\) denotes a smooth transition function on the interval \([0,1]\), such as a sigmoid or Hermite interpolant.
\end{remark}
%
\begin{table*}[t]
\centering
\begin{threeparttable}
\caption{Original and cumulative basis functions for several nonlinear parameterizations}
\label{tab:cumulative_basis}
\small
\begin{tabular}{p{3.5cm} p{6.8cm} p{5.8cm}}
\toprule
\textbf{Parameterization} &
\textbf{Original basis $B_i(\cdot)$} &
\textbf{Cumulative basis $\displaystyle\bar{B}_i(\cdot)=\sum_{j=i}^{n} B_j(\cdot)$} \\ 
\midrule

Power (monomial) &
$B_i(s)=s^{\,i}$ &
$\displaystyle \bar{B}_i(s)=\sum_{j=i}^{n} s^{\,j}= \dfrac{s^{\,i}-s^{\,n+1}}{1-s}\quad (s\neq 1)$ \\[6pt]

B\text{-}spline (degree $k$) &
$B_i(s)=N_{i,k}(s)$ &
$\displaystyle \bar{B}_{i}(s)=\sum_{j=i}^{n} N_{j,k}(s)$ \\[6pt]

Cubic Hermite$^{\dagger}$  
(Bézier) &
$B_i(s)=N_{i,3}(s)$  &
$\displaystyle \bar{B}_{i}(s)=\sum_{j=i}^{3}N_{j,3}(s)$\\

\bottomrule
\end{tabular}
\begin{tablenotes}
\small
\item[\textsuperscript{$\dagger$}] For a cubic Hermite segment with endpoints \( p_a \) and \( p_b \), and corresponding endpoint velocities \( v_a \) and \( v_b \), the control points are given by \( g_0 = g_a \), \( g_1 = g_a \boxplus \frac{1}{3}\xi_a \), \( g_2 = g_b \boxminus \frac{1}{3}\xi_b \), and \( g_3 = g_b \). The associated knots are \( \{0, 0, 0, 0, 1, 1, 1, 1\} \).
\end{tablenotes}
\end{threeparttable}
\end{table*}

Having established a unified cumulative parameterization on \( SE(3) \), we now shift focus from geometry to kinematics.

Differentiation of the configuration \( {g}(s,t) \) with respect to the arc-length coordinate \( s \) yields the body strain, defined as
$\hat{\xi}(s,t) = {g}^{-1}{g}'
$,
while differentiation with respect to time \( t \) gives the body velocity,
$
\hat{\eta}(s,t) = {g}^{-1} \dot{{g}}$.
The subsequent section derives closed-form expressions for both \( {\xi} \) and \( {\eta} \).
%
\section{Kinematics}\label{sec.kin}
%
\subsection{Strain and velocity}
To obtain closed-form expressions for the strain and velocity, we first need a mechanism that transfers variations on the manifold to elements of the Lie algebra. This role is played by the left-trivialized tangent of the chosen retraction map, introduced next, together with its inverse.
\begin{definition}[Left-trivialized tangent]
	Given a map \( \tau : \mathfrak{g} \to \mathcal{G} \), its left-trivialized tangent \( \mathrm{d}\tau_{x} : \mathfrak{g} \to T_{e}\mathcal G \) and its inverse \( \mathrm{d}\tau_{x}^{-1} : T_{e}\mathcal G \to \mathfrak{g}\) are defined such that, for a group element \( g = \tau(x) \in \mathcal{G} \) and a Lie algebra element \( y \in \mathfrak{g} \), the following relationships hold:
	$$
		\mathrm{d}\tau_{x}  (y) = \mathrm{d}L_{\tau(-x)}\circ\partial_{x} \tau(y)
	$$
	$$
		\mathrm{d}\tau_{x}^{-1}  (y) = \partial_{x}^{-1} \tau\circ\mathrm{d}L_{\tau(x)}(y)
	$$
	Here, $\partial_{x} \tau(y) \in T_{g}\mathcal{G}$ denotes the derivative of $\tau(x)$ in the direction of $y$.
\end{definition}

Based on this definition, the left-trivialized derivative of the exponential map \( \exp \) and its inverse are given as follows:
$$    \mathrm{dexp}_x(y) = \sum_{j=0}^{\infty} \frac{1}{(j+1)!} \, \mathrm{ad}_x^j(y),
$$
$$
    \mathrm{dexp}_x^{-1}(y) = \sum_{j=0}^{\infty} \frac{B_j}{j!} \, \mathrm{ad}_x^j(y),
$$
where \( B_j \) are the Bernoulli numbers. In practice, these series are typically truncated to achieve a desired level of accuracy. The first few Bernoulli numbers are \( B_0 = 1 \), \( B_1 = -\frac{1}{2} \), \( B_2 = \frac{1}{6} \), and \( B_3 = 0 \) \cite{hairer2006geometric}.

Similarly, the left-trivialized derivative of the Cayley map and its inverse are given by:
$$
    \operatorname{dcay}_x(y) = \left( \mathbb I - \frac{x}{2} \right)^{-1} y \left(\mathbb I + \frac{x}{2} \right)^{-1},
$$
$$
    \operatorname{dcay}_x^{-1}(y) = \left( \mathbb I - \frac{x}{2} \right) y \left( \mathbb I + \frac{x}{2} \right).
$$
with the left-trivialized differentials of the retraction map in hand, we are now prepared to characterize the propagation of an infinitesimal variation through a finite Lie algebra increment. The following theorem formalizes this relationship and constitutes a crucial step in differentiating the cumulative curve.
\begin{theorem}[Left–trivialized differential  of increment]\label{prop:vel_comp}

Consider two configurations
$g_a,g_b\in\mathcal G$ related by a Lie-algebra increment
\(
g_b
= g_a \boxplus x
= g_a \,\tau(x),
\;x\in\mathfrak g.
\)
Let
\[
\delta g_a\in T_{g_a}\mathcal G,
\quad
\delta g_b\in T_{g_b}\mathcal G,
\quad
\delta x\in\mathfrak g
\]
be the corresponding infinitesimal variations, their variation holds the following relationship:
\begin{equation}\label{dfgc}
	\delta \zeta_b =
	\operatorname{Ad}_{\tau(-x)}(\delta \zeta_a)
	+
	\mathrm d\tau_{x}(\delta x)
\end{equation} 
with
$\delta \zeta_a={g_a^{-1}}\delta g_a\in \mathfrak{g}$ and $\delta \zeta_b={g_b^{-1}}\delta g_b\in \mathfrak{g}$.
	The operator \( \mathrm{Ad}_g: \mathfrak{g} \to \mathfrak{g} \) can be regarded as a change of basis of a Lie algeria $\zeta\in\mathfrak{g}$ with respect to the argument \( g \in \mathcal{G} \) and is defined by
	$
	\mathrm{Ad}_g (\zeta) = g\zeta g^{-1}
	$.  
\end{theorem}
\begin{proof}

Taking the first‐order variation of \(g_b = g_a \tau(x)\), 
We can obtain:
\begin{equation}\label{oiu}
		\delta g_b = g_b\delta \zeta_b={g_a}\delta \zeta_a\tau(x)+{g_a}\mathrm D \tau_{x}(\delta x)
\end{equation}

Noting $\mathrm D \tau_{x}(\delta x)=\tau(x)\mathrm{d}\tau_{x}  (\delta x)$ and $	g_b
= g_a \tau(x)$, \eqref{oiu} can be further simplified as follows:
\begin{equation}\label{oiu2}
	g_b\delta \zeta_b={g_a}\delta \zeta_a\tau(x)+{g_b }\mathrm d \tau_{x}(\delta x)
\end{equation}

Multiplying~\eqref{oiu2} on the right by $g_b^{-1}$ and noting that $g_a^{-1}=\tau(x)g_b^{-1}$ we can finally get~\eqref{dfgc}.

\end{proof}

Let us now consider the configuration \( g(s,t) \) defined by~\eqref{gg} via a retraction map, and denote the variation \( \delta \zeta = g^{-1}(\delta g) \in \mathfrak{se}(3) \). By applying a step-by-step recurrence based on~\eqref{dfgc}, iterating from \( i = 1 \) to \( i = n \), we obtain the following recurrence formula:
\begin{align}
    \delta \zeta^{(0)} &= \delta \zeta_0, \label{zeta0}\\
    \delta \zeta^{(i)} &= \operatorname{Ad}_{\tau(-\bar B_i\Psi_i)}\delta \zeta^{(i-1)}\notag\\&\quad+\mathrm{d}\tau_{B_i\Psi_i}(\bar{B}'_{i}(s)\Psi_i\delta s+\bar{B}_{i}\dot{\Psi}_i\delta t).\label{zetai}
\end{align}
    
Given the decomposition \( \delta \zeta = \xi\,\delta s + \eta\,\delta t \), we can deduce the following two propositions from~\eqref{zeta0} and~\eqref{zetai}:

\begin{proposition}[Strain of the cumulative parametric rod]
	Considering a configuration $g$ defined by~\eqref{gg} via retraction map, denote its strain with respect to the body frame as $\hat\xi=g^{-1}g^\prime$. The strain at any point can be computed via the following recurrence formula:
    \begin{align}
    \xi^{(0)} &= 0, \label{xi0}\\
    \xi^{(i)} &= \operatorname{Ad}_{\tau(-\bar B_i\Psi_i)}\xi^{(i-1)}+\bar B^\prime_i\Psi_i.\label{xii}
    \end{align}
    \label{prop1}
    \end{proposition}
\begin{proposition}[Velocity of the cumulative parametric rod]Denote its velocity \textit{w.r.t}. body frame as $\hat\eta = g^{-1}\dot{g}$ and $\hat\eta_i=g_i^{-1}\dot{g}$. The velocity at any point can be computed via the following recurrence formula:
    \begin{align}
    \eta^{(0)} &= \eta_0,\label{eta0}\\
    \eta^{(i)} &= \operatorname{Ad}_{\tau(-\bar B_i\Psi_i)}\eta^{(i-1)}+Q_i\dot{\Psi}_{i}. \label{etai}
    \end{align}
    with  $Q_i=\bar B_i\mathrm{d}\tau_{\bar B_i\Psi_i}$ and $\dot{\Psi}_i=\mathrm{d}\tau_{\Psi_i}^{-1}(\eta_i-\mathrm{Ad}_{\tau(\Psi_i)}^{-1}\eta_{i-1})$.

As a result, the acceleration can be computed via the following recurrence formula by differentiating \eqref{eta0} and \eqref{etai}:
    \begin{align}
    \dot{\eta}^{(0)} &=  \dot{\eta}_0,\\
    \dot{\eta}^{(i)} &= \operatorname{ad}_{Q_i\dot{\Psi}_{i}}\operatorname{Ad}_{\tau(-\bar B_i\Psi_i)}\eta^{(i-1)}\notag\\&\quad+\operatorname{Ad}_{\tau(-\bar B_i\Psi_i)}\dot{\eta}^{(i-1)}\notag\\&\quad+\dot{Q}_i\dot{\Psi}_{i}+Q_i\ddot{\Psi}_{i}.
    \end{align}
    \label{prop2}
    \end{proposition}

The recurrence formulation in {Proposition}~\ref{prop2} demonstrates that the velocity at any point along the rod is linearly related to the velocities of the control points. This relationship can be expressed in matrix form. Specifically, we define this linear mapping as
\[
{\eta} = \sum_{i=0}^{n} {J}_i {\eta}_i, \quad \mathrm{or} \quad {\delta\zeta} = \sum_{i=0}^{n} {J}_i {\delta\zeta}_i
\]
where ${J}_i$ is the Jacobian of $\eta$ with respect to $\eta_i$. \( \delta\hat{\zeta} = g^{-1} \delta g \) and \( \delta\hat{\zeta}_i = g_i^{-1} \delta g_i \) denote the Lie algebra variations of \( g \) and \( g_i \) respectively.
According to {Proposition}~\ref{prop2}, the expression of recurrence for ${J}_i$ can be derived as the following proposition:
\begin{proposition}[Jacobian of kinematics] The Jacobian of $\eta$ \textit{w.r.t.} $\eta_i$ is computed via the following recurrence formula: 
    \begin{align}
    J_j^{(0)} &= P_j^0\mathbb I,\label{jaco0}\\
    J_j^{(i)} &= \operatorname{Ad}_{\tau(\bar B_i\Psi_i)}^{-1}J_j^{(i-1)}+P_{j}^iQ_i\mathrm{d}\tau_{\Psi_i}^{-1}\notag\\&\quad-P_j^{i-1}Q_i\mathrm{d}\tau_{\Psi_i}^{-1}\mathrm{Ad}_{\tau(\Psi_i)}^{-1}\label{jacoi}
    \end{align}
where $P_j^i=1$ when $i=j$ and $P_j^i=0$ when $i\neq j$.

The derivative of $J_i$ with respect to length $s$ of rod can be computed by differentiating \eqref{jaco0} and \eqref{jacoi}:
 \begin{align}
    J_j^{\prime(0)} &= 0,\label{jaco0'}\\
    J_j^{\prime(i)} &= \operatorname{ad}_{\bar B_i^\prime\Psi_i}\operatorname{Ad}_{\tau(-\bar B_i\Psi_i)}J_j^{(i-1)}\notag+\operatorname{Ad}_{\tau(-\bar B_i\Psi_i)}J_j^{\prime(i-1)}\notag\\&\quad+P_{j}^iQ^\prime_i\mathrm{d}\tau_{\Psi_i}^{-1}-P_j^{i-1}Q^\prime_i\mathrm{d}\tau_{\Psi_i}^{-1}\mathrm{Ad}_{\tau(\Psi_i)}^{-1}. \label{jacoi'}
    \end{align}
\end{proposition}

To enable a more compact matrix formulation, let us denote \( \dot{q} = [\eta_1^\top, \dots, \eta_N^\top]^\top \) and define the Jacobian matrix \( J = [J_1, \dots, J_N] \). Then, the vector \( \eta \) can be expressed in matrix form as
\begin{equation}
    \eta = J \dot{q}.
\end{equation}

The recurrence formulas above show that the strain, velocity, acceleration, and Jacobian all depend on the auxiliary quantities
\(Q_i,\;Q_i',\;\dot{Q}_i,\;\dot{\Psi}_i,\;\ddot{\Psi}_i\).
Efficient evaluation of these terms is therefore essential for a fast forward kinematics routine and for the Jacobian-based dynamics that follow.
We now give closed-form, recursively computable expressions for each of them.
%
\subsection{Computation of \texorpdfstring{$Q_i$, $Q_i^\prime$, $\dot Q_i$, $\dot\Psi_i$ and $\ddot\Psi_i$}{Qi, Q̇i, Ψ̇i and Ψ̈i}}
%
We begin by recalling the differential of the exponential map on a Lie algebra, which admits the following representation:
\begin{equation}
\mathrm{dexp}_x = \sum_{j=0}^{\infty} \frac{1}{(j+1)!} \, \mathrm{ad}_x^j = \int_{0}^{1} \operatorname{Ad}_{\exp(rx)}^{-1} \, \mathrm{d}r.
\end{equation}

Using map $\tau(\cdot)$ as $\exp(\cdot)$, we define the operator $Q_i(s)$ as
\begin{align}
Q_i(s)\coloneqq  \bar B_i(s)\,\mathrm{d}\tau_{\bar B_i(s)\Psi_i} =\bar B_i(s) \int_0^1 \operatorname{Ad}_{\tau(\bar B_i(s)\Psi_i r)}^{-1} \, \mathrm{d}r.\notag
\end{align}

Since the integral above is defined over the fixed interval \([0,1]\), evaluating it at each point along the rod requires recomputing the full integral independently, which significantly increases computational cost. To avoid such redundancy and improve efficiency, we reformulate the integral as an accumulation along the arc length of the rod.

Noting that \( \bar B_i(s) \) is smooth, strictly increasing, and maps \( [s_i,\, s_{i+k-1}] \) onto \([0,1]\), the scaled function \( r \mapsto \bar B_i(s)r \), with \( r \in [0,1] \), inherits these properties. We therefore introduce the substitution \( \bar B_i(x) = \bar B_i(s)r \). For \( s \in [s_i,\, s_{i+k-1}] \), applying the change of the integrated variable yields the following result:
\begin{equation}
Q_i(s) = \int_0^s \operatorname{Ad}_{\tau(\bar B_i(x)\Psi_i)}^{-1} \bar B_i'(x)\, \mathrm{d}x.
\end{equation}

This reformulation enables recursive evaluation: the value of \( Q_i(s) \) at each point can be efficiently updated from its value at the previous point. As a result, a single forward pass along the rod suffices to compute \( Q_i(s) \) at all points, significantly reducing computational complexity.

Given the compact support of the basis function $\bar B_i(s)$, specifically that $\bar B_i(s) = 0$ for $s < s_i$ and $\bar B_i(s) = 1$ for $s > s_{i+k-1}$, the function $Q_i(s)$ can be expressed piecewise as follows:
\begin{equation}\label{Qi}
Q_i(s) =
\begin{cases}
0, & s < s_i, \\[6pt]
\displaystyle\int_{s_i}^s \operatorname{Ad}_{\tau(\bar B_i(x)\Psi_i)}^{-1} \bar B_i'(x) \, \mathrm{d}x, & s_i \leq s \leq s_{i+k-1}, \\[12pt]
Q_i(s_{i+k-1}), & s > s_{i+k-1}.
\end{cases}
\end{equation}

We can immediately get its derivative \textit{w.r.t.} $s$:
\begin{equation}\label{Qi'}
Q_i'(s) =
\begin{cases}
0, & s < s_i, \\[6pt]
\operatorname{Ad}_{\tau(\bar B_i(x)\Psi_i)}^{-1} \bar B_i'(x) , & s_i \leq s \leq s_{i+k-1}, \\[6pt]
Q_i'(s_{i+k-1}), & s > s_{i+k-1}.
\end{cases}
\end{equation}

Note that when $\bar B_i(s_{i+k-1}) = 1$, it follows that
\begin{equation}
\mathrm{d}\tau_{\Psi_i} = Q_i(s_{i+k-1}) \eqqcolon A_i.
\end{equation}

For simplicity, we denote $ Q_i(s_{i+k-1})$ as $A_i$. With $A_i$ known, the time derivative of $\Psi_i$ can be computed as
\begin{equation}\label{eq:psidot}
\dot{\Psi}_i = A_i^{-1} \left( \eta_i - \operatorname{Ad}_{\tau(\Psi)}^{-1} \eta_{i-1} \right).
\end{equation}

Furthermore, given $Q_i(s)$ and $\dot{\Psi}_i$, we compute the time derivative $\dot{Q}_i(s)$ via:
\begin{equation}
\dot{Q}_i(s) = \int_0^s \operatorname{ad}_{Q_i(x)\dot{\Psi}_i} \, \operatorname{Ad}_{\tau(\bar B_i(x)\Psi_i)}^{-1} \bar B_i'(x) \, \mathrm{d}x.
\end{equation}

Finally, the second derivative $\ddot{\Psi}_i$ is obtained by differentiating \eqref{eq:psidot} and substituting accordingly:
\begin{equation}
\begin{aligned}
\ddot{\Psi}_i &= A_i^{-1} ( \dot{A}_i \dot{\Psi}_i + \dot{\eta}_i - \operatorname{ad}_{A_i \dot{\Psi}_i} \operatorname{Ad}_{\tau(\Psi)}^{-1} \eta_{i-1} \\&\quad- \operatorname{Ad}_{\tau(\Psi)}^{-1} \dot{\eta}_{i-1} ).
\end{aligned}
\end{equation}

The closed-form and recursive expressions for \(Q_i\), \(Q_i'\), \(\dot Q_i\), \(\dot\Psi_i\), and \(\ddot\Psi_i\) derived above provide all the ingredients required for a fast forward evaluation of rod kinematics.  Leveraging these pre-computed operators, we can now assemble an end-to-end procedure that traverses the rod once to accumulate geometric quantities and once more to recover velocities and Jacobians.
%
\subsection{Kinematic computation framework}
%
For a cumulatively parameterized Cosserat rod, given a discrete set of configurations \( \mathcal{Q} = \{ g_i \in SE(3) \mid i = 0,1,\dots,n \} \) and corresponding body velocities \( \{ \eta_i \in \mathbb{R}^6 \mid i = 0,1,\dots,n \} \), the continuous configuration \( g(s) \), strain \( \xi(s) \), velocity \( \eta(s) \), and Jacobian \( J(s) \) along the rod can be computed efficiently in two passes from \( s = 0 \) to \( s = L \).
\begin{enumerate}
    \item 
In the first pass, at each sample point, the configuration \( g(s) \) is computed using \eqref{gg}, and the strain \( \xi(s) \) is evaluated via the recurrence formulas \eqref{xi0} and \eqref{xii}. Subsequently, integrate formulation in length to compute \eqref{Qi} \( Q_i(s) \) for \( i=1,2,\dots,N \). During this process, all matrices \( A_i \) and the derivatives \( \dot{\Psi}_i \) can be obtained concurrently.
\item 
In the second pass, since \( Q_i \), \( \dot{\Psi}_i \), and \( A_i \) are already available from the first pass, the velocity \( \eta(s) \) is computed using the recurrence formulas \eqref{eta0} and \eqref{etai}, and the Jacobian \( J(s) \) is evaluated via \eqref{jaco0} and \eqref{jacoi}.
\end{enumerate}
%
\subsection{Structure of kinematic Jacobian}
As noted in Remark~\ref{rm1}, B-spline-based parameterization endows the kinematics with local adjustability, allowing for the decoupling of control points. This leads to a Jacobian matrix with a constant and narrow bandwidth, resulting in sparse stiffness and mass matrices. \Cref{jaco} illustrates the bandwidth patterns of the Jacobian matrices obtained from the piecewise linear strain parameterization and the Lie group B-spline parameterization, respectively. As shown, the Lie group B-spline approach maintains a constant bandwidth in the Jacobian matrix. In contrast, the bandwidth in the strain-based method increases progressively along the structure, ultimately resulting in a dense mass matrix of dynamics.
\begin{figure}[h]
	\centering
	\includegraphics[width=0.45\textwidth]{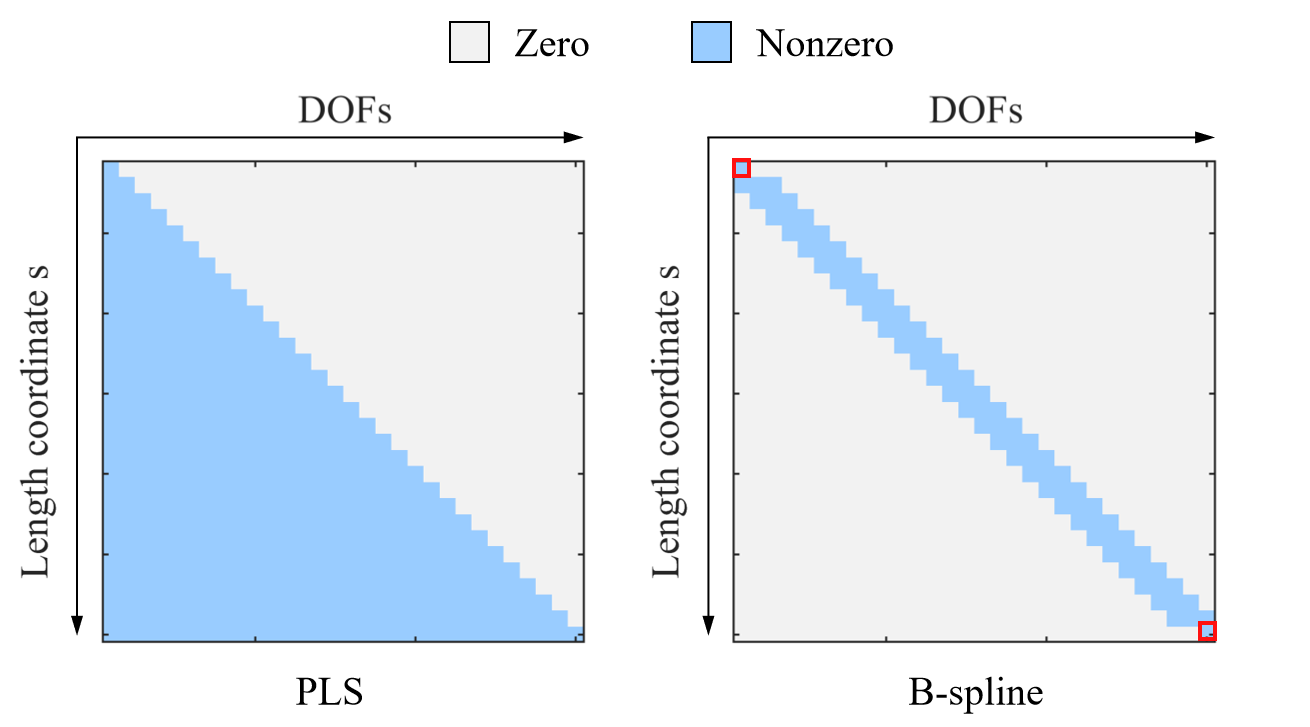}
	\caption{\textbf{The variation of the kinematic Jacobian matrix concerning the arc-length coordinate $s$.} The left figure shows the Jacobian matrix obtained using the Piecewise Linear Strain (PLS) \cite{li2023piecewise} parameterization method. In contrast, the right figure shows the Jacobian matrix obtained using the proposed Lie group B-spline parameterization method.
}
	\label{jaco}
\end{figure}

Compared to strain-based methods, the proposed approach, characterized by a constant bandwidth Jacobian, supports flexible, element-wise, and modular modeling of soft robots, akin to the finite element method. Owing to the properties of B-spline interpolation, the poses at the two ends of a Cosserat rod are precisely equal to the first and last control points. Consequently, the Jacobians at these ends depend only on those respective points, as highlighted by the red boxes at the diagonal in the figure. This structure ensures continuity across elements, since adjacent elements can share a control point, specifically, the tail control point of one element and the head control point of the next. In the following subsections, we describe a unified framework for modeling complex multi-body robotic structures, showing that element-wise modeling emerges as a special case within this formulation.
%
\subsection{Kinematics of tree-structured robot}
%
Cumulative parameterization provides a natural and efficient framework for geometric modeling of tree-structured robots that comprise rigid and soft components. This formulation eliminates the need to explicitly impose joint constraints, as the hierarchical structure of cumulative transformations inherently captures the connectivity between segments. It is particularly effective for modeling robots with branching, tree-like architectures, as shown in 
\Cref{treerobot}.
\begin{figure}[h]
	\centering
	\includegraphics[width=0.38\textwidth]{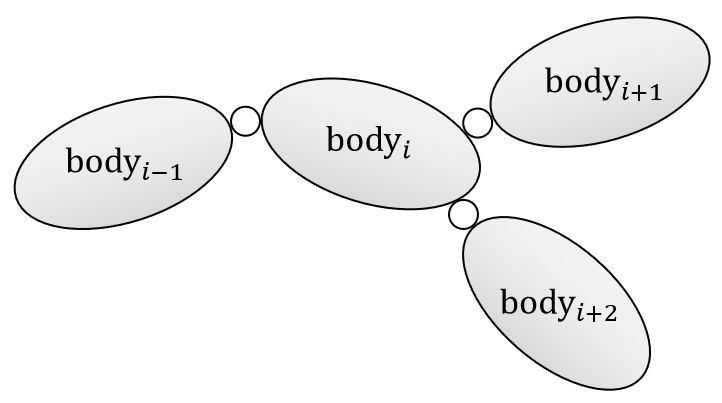}
	\caption{\textbf{Schematic of a tree-structured robot.} Each ellipse represents a soft or rigid body segment (e.g., $\text{body}_i$), and the small circles represent joints that allow rotational movement. The $\text{body}_i$ acts as a branching point, connecting to multiple child links $\text{body}_{i+1}$ and $\text{body}_{i+2}$, illustrating a typical tree structure.
}
	\label{treerobot}
\end{figure}

\Cref{connection} illustrates three types of connections in a tree-structured robot.  
For such systems, we define that all Cosserat rods must be connected end-to-end, which can always be achieved through segmentation if necessary.
\begin{figure*}[t]
	\centering
	\includegraphics[width=0.8\textwidth]{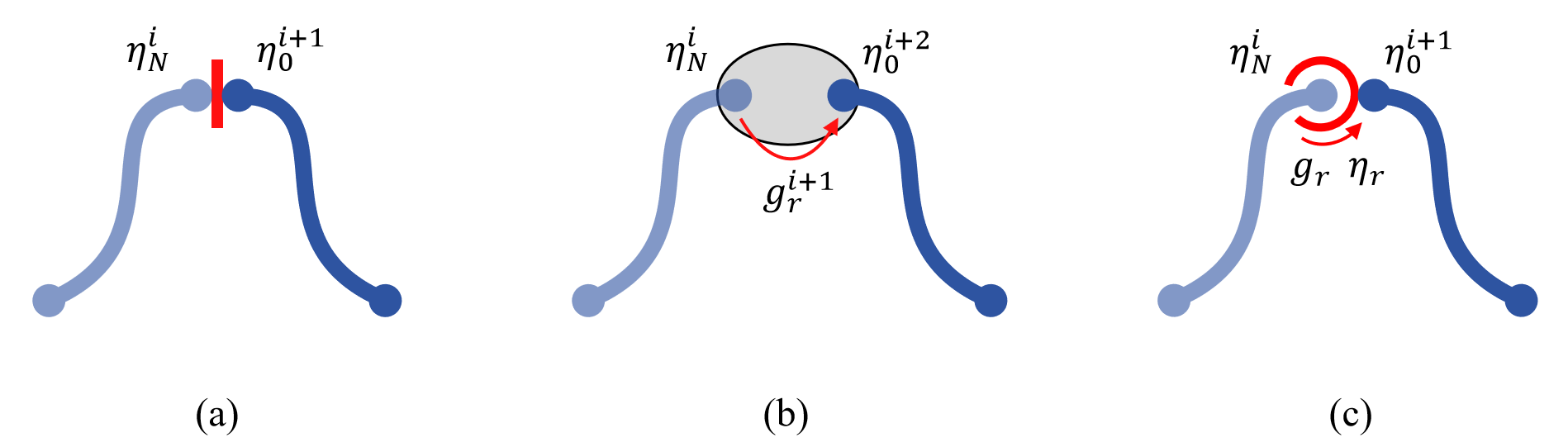}
	\caption{\textbf{Three types of connections.} (a) Direct rigid connection; (b) Rigid body insertion; (c) Movable joint connection.
}
	\label{connection}
\end{figure*}

For a rod segment $\text{body}_i$ parameterized by $\mathcal{Q}^i = \{g^i_j\}_{j=0}^{n}$, we denote its body velocity as $\eta^i$, and its kinematic Jacobian with respect to $g^i_j$ as $J^i_j$. Any physical quantity associated with $\text{body}_i$ is annotated with the superscript $(\cdot)^i$. The three types of connection are detailed below.
\\ \\
    1. \textbf{Direct rigid connection} (see \Cref{connection}.a):  
    If $\text{body}_{i+1}$ is rigidly connected to the end of $\text{body}_i$, and is parameterized by $\mathcal{Q}^{i+1} = \{g^{i+1}_j\}_{j=0}^{M}$, then the connection conditions are:
    \[
    g_0^{i+1} = g^i_n, \quad \eta_0^{i+1} = \eta^i_n
    \]
    Therefore, the velocity along $\text{body}_{i+1}$ is given by:
    \[
    \eta^{i+1}(s) = J^{i+1}_0(s)\eta^i_n + \sum_{j=1}^{M} J^{i+1}_j(s)\eta^{i+1}_j
    \]

    For a single-body continuum robot, the modeling can be modularized by representing each segment with an independent B-spline and connecting all segments sequentially through rigid connections.\\ \\
    2. \textbf{Rigid body insertion} (see \Cref{connection}.b):  
    If a rigid $\text{body}_{i+1}$ is inserted between soft $\text{body}_i$ and $\text{body}_{i+2}$, and $\text{body}_{i+2}$ is parameterized by $\mathcal{Q}^{i+2} = \{g^{i+2}_j\}_{j=0}^{M}$, then:
    \[
    g_0^{i+2} = g^i_n g^{i+1}_r, \quad \eta_0^{i+2} = \mathrm{Ad}_{g^{i+1}_r}^{-1} \eta^i_n
    \]
    where $g^{i+1}_r$ denotes the constant transformation matrix associated with the rigid body $\text{body}_{i+1}$. Then:
    \[
    \eta^{i+2}(s) = J^{i+2}_0(s)\mathrm{Ad}_{g^{i+1}_r}^{-1} \eta^i_n + \sum_{j=1}^{M} J^{i+2}_j(s)\eta^{i+2}_j
    \]
\\
    3. \textbf{Movable joint connection} (see \Cref{connection}.c):  
If $\text{body}_i$ and $\text{body}_{i+1}$ are connected via a movable joint, we introduce additional degrees of freedom represented by the joint transformation $g_r$. The initial velocity of $\text{body}_{i+1}$ is given by:
\[
\eta_0^{i+1} = \operatorname{Ad}_{g_r}^{-1} \eta^i + \eta_r,
\]
where $\eta_r = (g_r^{-1} \dot{g}_r)^\vee$ represents the relative velocity introduced by the joint. Then, the velocity along $\text{body}_{i+1}$ becomes:
\[
\eta^{i+1}(s) = J^{i+1}_0(s)\left(\mathrm{Ad}_{g_r}^{-1} \eta^i + \eta_r\right) + \sum_{j=1}^{M} J^{i+1}_j(s) \eta^{i+1}_j.
\]

For a spherical joint, the transformation $g_r$ is given by:
\[
g_r = 
\begin{bmatrix}
R_r & 0 \\
0 & 1
\end{bmatrix},
\]
where $R_r \in SO(3)$ is a rotation matrix.

For a prismatic joint, the transformation is:
\[
g_r = 
\begin{bmatrix}
\mathbb I & p \\
0 & 1
\end{bmatrix},
\]
where $p \in \mathbb{R}^3$ represents the translational displacement along a fixed axis.


With body velocity $\eta(s)$ and the Jacobian $J(s)$ now available at any point along the soft robot or the tree-structured robotic system, we have a complete kinematic description that links discrete control points to continuous fields.  

\subsection{Kinematics of concentric structured robot}

Connections between rods described previously allow us to model the kinematic structure of many soft and continuum robots, such as tendon-driven robots~\cite{lilge2022kinetostatic} or multi-backbone robots~\cite{janabi2021cosserat}. However, special care must be taken when modeling robots composed of a nested arrangement of tubes, such as concentric tube robots (CTR)~\cite{gilbert2016concentric} or concentric push-pull robots (CPPR)~\cite{tummers2025continuum}.

Concentric structured robots are composed of multiple pre-curved, elastic tubes that are telescopically and rotationally nested within one another. These tubes are coaxially aligned, and the robot’s end-effector position and orientation are controlled through the relative rotation and translation of the individual tubes. These robots offer high dexterity in compact spaces and are widely employed in applications such as minimally invasive surgery. However, accurately modeling their behavior remains challenging due to the strong coupling between tubes and the nonlinear nature of large deformations.

To address these challenges, we extend the cumulative parameterization framework to concentric tube robots by first considering a general multi-rod configuration. Specifically, we begin by considering the general case of two soft rods with equal length, as shown in~\Cref{concen} (a). Let \( g^1(s) \) and \( g^2(s) \) denote their respective configurations along the arc length \( s \). Let us assume a relative motion \( g^r(s) \in \mathrm{SE}(3) \) exists between them such that
\begin{equation}\label{ctr}
    g^2(s) = g^1(s) g^r(s).
\end{equation}

Suppose the configuration \( g^1(s) \) of the first (main) rod is parameterized by a set of control points \( \{g_i^1\}_{i=0}^n \) and cumulative basis functions \( \{\bar{B}_i^1(s)\}_{i=0}^n \). Similarly, let the relative motion \( g^r(s) \) be parameterized by control points \( \{g_i^r\}_{i=0}^M \) and basis functions \( \{\bar{B}_i^r(s)\}_{i=0}^M \). Then, according to equation~\eqref{ctr}, the configuration \( g^2(s) \) of the second (sub) rod is jointly determined by the sets \( \{g_i^1\}_{i=0}^n \) and \( \{g_i^r\}_{i=0}^M \).

The body velocity \( \eta^2 \) and strain \( \xi^2 \) of the second rod are related to those of the first rod and the relative motion as follows:
\[
\eta^2 = \operatorname{Ad}_{g^r}^{-1} \eta^1 + \eta^r, \qquad
\xi^2 = \operatorname{Ad}_{g^r}^{-1} \xi^1 + \xi^r,
\]
where \( \eta^1, \xi^1 \) denote the body velocity and strain of the first rod, and \( \eta^r, \xi^r \) represent those of the relative motion.

The kinematic Jacobian of \( \eta^2 \) with respect to the control points of the first rod \( g_i^1 \) is given by
\[
J_{i,1}^2 = \operatorname{Ad}_{g^r}^{-1} J_i^1,
\]
and concerning the control points of the relative motion \( g_i^r \) by
\[
J_{i,r}^2 = J^r_{i}.
\]

Note that $\partial_s\operatorname{Ad}_{g^r}^{-1}=-\operatorname{ad}_{\xi^r}\operatorname{Ad}_{g^r}^{-1}$. The spatial derivatives of these Jacobians are expressed as
\[
\partial_s J_{i,1}^2 = \operatorname{Ad}_{g^r}^{-1}\partial_s J_i^1-\operatorname{ad}_{\xi^r}\operatorname{Ad}_{g^r}^{-1}  J_i^1 , \quad
\partial_s J_{i,r}^2 = \partial_s J^r_{i-n}
\]

For concentric tube robots, the relative motion must be constrained to pure rotation, as shown in~\Cref{concen} (b). Since the two rods are coincident along their centerline, there is no relative translation between them, i.e., \( p^r = 0 \). Therefore, the relative motion reduces to:
\[
g^r = 
\begin{bmatrix}
R^r & 0 \\
0 & 1
\end{bmatrix}, \quad R^r \in \mathrm{SO}(3).
\]
\begin{figure}[t]
	\centering
	\includegraphics[width=0.45\textwidth]{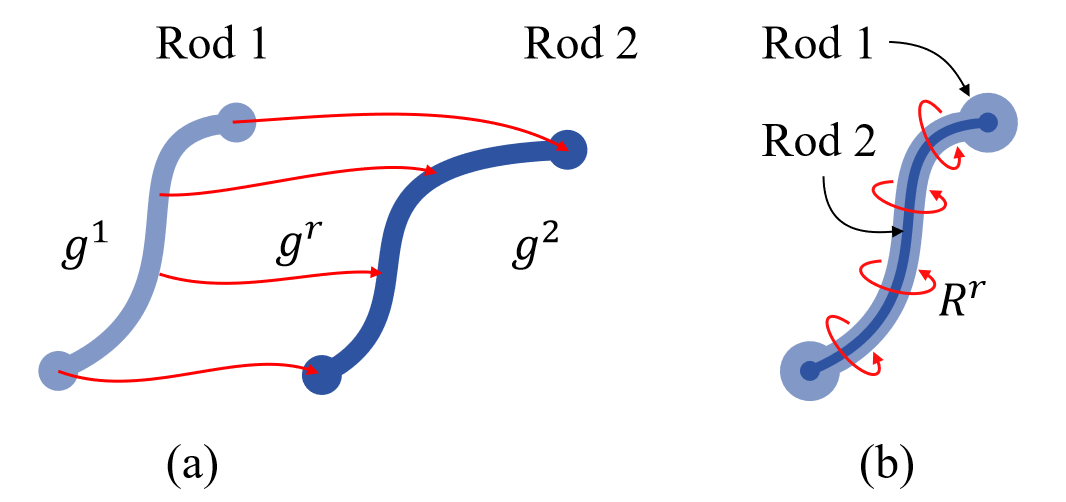}
	\caption{\textbf{Relative motion modeling in concentric tube robots.} 
(a) General case: Rod 1 and Rod 2 are connected via a relative motion $g^r(s) \in \mathrm{SE}(3)$, which includes both translation and rotation. (b) Concentric configuration: The rods are constrained to have no relative translation, resulting in a pure rotational relationship $R^r(s) \in \mathrm{SO}(3)$ along their shared centerline.}
	\label{concen}
\end{figure}

In this case, \( g^r \) is parameterized by a set of rotation matrices \( \{R^r_i\}_{i=0}^M\). Let \( \omega^r \in \mathbb{R}^3 \) and \( \kappa^r \in \mathbb{R}^3 \) denote the angular velocity and angular strain of the relative rotation, respectively. The kinematics of the second rod then simplifies to:
\[
\eta^2 = \operatorname{Ad}_{g^r}^{-1} \eta^1 + \mathrm{P} \omega^r, \qquad
\xi^2 = \operatorname{Ad}_{g^r}^{-1} \xi^1 + \mathrm{P} \kappa^r,
\]
where \( \mathrm{P} = [\mathbb{I}_{3 \times 3} ; 0_{3 \times 3}]  \in \mathbb{R}^{6 \times 3} \) projects angular components into the Lie algebra \( \mathfrak{se}(3) \).

The Jacobian of \( \eta^2 \) with respect to the relative rotation control points \( \{R^r_i\} \) is then given by:
\[
J_{i,1}^2 = \mathrm{P} J^r_{i}, \qquad
\partial_s J_{i,r}^2 = \mathrm{P}\partial_s J^r_{i}.
\]

In summary, this section presented the dynamic formulation of soft robots based on a cumulative parameterization framework, including the computation of strain, velocity, and the corresponding kinematic Jacobian. The formulation was then extended from a single Cosserat rod to more complex multi-body robotic structures, including tree-like configurations and concentric tube robots. In the next section, we turn our attention to the statics formulation based on the cumulative parameterization approach.

\section{Statics}\label{sec.sta}
%
\subsection{Variation principle}
In this section, we aim to derive the static equilibrium equations for a Cosserat rod based on cumulative parameterization. Specifically, we seek configurations that minimize the total potential energy, subject to geometric constraints. This is formulated as a variational problem on $SE(3)$, from which the corresponding first-order equilibrium conditions are derived.

To clearly establish the mathematical foundations for this derivation, we first recall some essential concepts from optimization on Riemannian manifolds. Let $\mathcal{G}$ be a smooth Riemannian manifold, and consider a smooth scalar field $f:\mathcal{G}\to\mathbb{R}$. The directional derivative of $f$ at a point $X \in \mathcal{G}$ along a tangent vector $v \in T_X \mathcal{G}$ is defined as:
$$\mathrm Df_X[v]\coloneqq\left.\frac{\mathrm{d}}{\mathrm{d}t}\right|_{t=0}f(X\boxplus tv),\qquad\forall\,(X,v)\in T\mathcal G.$$

This directional derivative allows us to express the variation of $f$ for a given variation $\delta \zeta \in \mathfrak{g}$ as $\delta f=\mathrm Df_X[\delta\zeta]$.

Turning to statics, we consider a soft rod subject to a distributed wrench density $\bar{F}$ along its length and a concentrated wrench $F_+$ at its distal tip. Throughout the following developments, $\bar{F}$ and $F_+$ are assumed to be given, either as explicit time functions or defined by a state-dependent model. The state is denoted by $g$ in statics, and by $(g, \eta)$ in dynamics.

A force $F(g)$ applied at a given cross-section, characterized by the configuration $g$, is said to be conservative if it derives from a potential. That is, there exists a scalar potential function $U_{\text{ext}}: SE(3) \to \mathbb{R}$ such that:
\begin{equation}
    \delta U_{ext}=\mathrm DU_{ext_{g}}[\delta\zeta]=-\delta \zeta^T F 
\end{equation}

When both $\bar{F}$ and $F_+$ are conservative, a global potential $U_{\text{ext}}$ exists for the entire rod such that:
\begin{equation}
    \delta U_{\text{ext}} = \int_0^L \delta \bar{U}_{ext} \, dX + \delta U_b, 
\end{equation}
where  
\( \delta \bar{U}_{\text{ext}} = -\delta \zeta^\top \bar{F} \) denotes the variation of the distributed external force potential between the boundaries, and  
\( \delta U_b = \delta \zeta(0)^\top F_- - \delta \zeta(L)^\top F_+ \) denotes the variation of the concentrated external force potential at the boundaries.

Assuming the rod is elastic and undergoes small deformations, the internal (strain) energy is defined as the integral of a quadratic form of the strain vector $\epsilon$ along the rod:
\begin{equation}
    U_{\text{int}} = \int_0^L \bar{{U}}_{\text{int}} \, dX = \frac{1}{2} \int_0^L \epsilon^T\Lambda         \, dX, \tag{13}
\end{equation}
with $\Lambda=\mathcal K\epsilon$ denoting the internal force, $\mathcal K = \mathrm{diag}(GJ_1, EJ_2, EJ_3, EA, GA, GA)$ representing the constant Hooke matrix of the rod. $\epsilon=\xi-\xi_0$. $\xi_0$ denotes the natural strain of the soft rod. Building on the preceding concepts and invoking D'Alembert's Principle, the static equilibrium can be formulated as the solution to the following optimization problem:

\begin{problem}\label{pr1}
Let \( \mathcal{Q} = \{ g_{i} \in SE(3) \mid i = 0,1,\dots,n \} \) denote the configuration set.
Find configuration set $\mathcal{Q}$ which minimizes:
\begin{equation}
    V(\mathcal{Q})=\int_0^L \bar{U}_{int}+\bar{U}_{ext}\, \mathrm{d}s+U_b
\end{equation}
where $V(\mathcal{Q})$ is the tolal static energy of the Cosserat rod.
\end{problem}

To convert the energy-minimisation statement into concrete equilibrium equations, we introduce an arbitrary admissible variation of the control poses,
\(g_i \mapsto g_i \boxplus t\,\delta\zeta_i\), 
and compute the first variation \(\delta V\).  
Requiring that \(\delta V\) vanishes for every choice of \(\delta\zeta_i\) leads to the stationarity (Euler–Lagrange) conditions summarised below.

\begin{proposition}[First-order optimality condition]
The first-order optimality condition of
Problem \ref{pr1} is $\delta V=0$,
detailed as:
\begin{equation}\label{statics}
    r(\mathcal{Q}):=\int_0^L (J'+\operatorname{ad}_\xi J)^\top \Lambda+J^\top F_e\, \mathrm{d}s=0
\end{equation}
\label{prop4}
\end{proposition}
\begin{proof}

Introduce an arbitrary {infinitesimal} variation of each control configuration,
\(g_i(t)=g_i\boxplus t\,\delta\zeta_i\),  
and evaluate derivatives at \(t=0\).  
Because the cumulative spline is linear in the control increments, the induced body-frame variation along the rod is
\[
\delta\zeta(s)=\sum_{i=1}^{n} J_i(s)\,\delta\zeta_i .
\]
Using the reduced-variation identity for Lie-group curves, we obtain
\[
\delta\xi
  = \delta\zeta' + \operatorname{ad}_{\xi}\!\delta\zeta
  = \sum_{i=1}^{n}\bigl(J_i' + \operatorname{ad}_{\xi}J_i\bigr)\,\delta\zeta_i .
\]
The virtual change of internal energy is therefore
\[
\delta \bar{U}_{\mathrm{int}}
   = \mathrm{D}\bar{U}_{\mathrm{int},g}[\delta\zeta]
   = \sum_{i=1}^{n} \delta\zeta_i^{\!\top}
     \bigl(J_i' + \operatorname{ad}_{\xi}J_i\bigr)^{\!\top}\!\Lambda ,
\]
while the variation of the external potential reads
\[
\delta \bar{U}_{\mathrm{ext}}
   = -\delta\zeta^{\!\top} F_e
   = -\sum_{i=1}^{n}\delta\zeta_i^{\!\top} J_i^{\!\top} F_e .
\]
Adding both contributions and integrating along the rod gives the first variation of the total potential energy:
\[
\delta V
   = \int_{0}^{L} \sum_{i=1}^{n} \delta\zeta_i^{\!\top}
     \Bigl( (J_i' + \operatorname{ad}_{\xi}J_i)^{\!\top}\!\Lambda
            + J_i^{\!\top} F_e \Bigr)\,\mathrm{d}s .
\]
Since each \(\delta\zeta_i\) is arbitrary, the integrand must vanish element-wise, yielding
\[
\int_{0}^{L}\!
      \bigl(J' + \operatorname{ad}_{\xi}J\bigr)^{\!\top}\!\Lambda
      + J^{\!\top} F_e \,\mathrm{d}s \;=\; 0 ,
\]
which is exactly the equilibrium condition~\eqref{statics}.  
\end{proof}

{Proposition}~\ref{prop4} presents the Euler–Lagrange formulation for statics. This condition can be used to determine whether the soft robot is in a state of static equilibrium.
To simplify the computational process, however, we do not solve this formulation directly. Instead, we determine the solution of statics by minimizing the energy, i.e., by solving Problem~\ref{pr1}.
\subsection{Riemannian optimization}
%
Problem~\ref{pr1} is an optimization problem defined on an $SE(3)$ manifold. Accordingly, we begin by recalling some fundamental concepts of optimization on Riemannian manifolds (which is the preferred approach in such a case), specifically the Riemannian gradient and the Riemannian Hessian~\cite{lang2012fundamentals}.

On a Riemannian manifold, the direction of steepest descent at a point is determined by the Riemannian gradient, which is obtained by projecting the Euclidean gradient onto the tangent space at that point. The Riemannian gradient is defined as follows:
\begin{definition}[Riemannian gradient]
	Let $\mathcal G$ be a smooth Riemannian manifold with a Riemannian metric $\langle \cdot,\cdot\rangle_x: T_x\mathcal G \times T_x\mathcal G \to \mathbb{R}$.  
	For a smooth scalar field $f:\mathcal G\to\mathbb R$, the Riemannian gradient of $f$ at a point $x\in\mathcal G$ is the unique tangent vector
	$
	\operatorname{grad}f(x)\in T_x\mathcal G
	$
	such that
	\[
	\mathrm Df_x[v]=\left\langle\operatorname{grad}f(x),\,v\right\rangle_x,
	\qquad\forall\,(x,v)\in T\mathcal G.
	\]
    Assume further that $\mathcal M$ is a Lie group and denote its Lie algebra by
	$\mathfrak g=T_e\mathcal G$.  
	The {left-trivialized gradient} of $f$ at $g\in  \mathcal G$
	is the element
	\[
		\operatorname{grad}^L  f(g):=
		g^{-1}\operatorname{grad}f(g)
		\;\in\;\mathfrak g,
	\]
	If the metric is {left-invariant}, then
	\[
		\mathrm Df_g[g\hat\xi]
		=\bigl\langle\operatorname{grad}^L f(g),\,\hat\xi\bigr\rangle_e,
		\qquad\forall\,\hat\xi\in\mathfrak g.
	\]	
\end{definition}

On a Lie group manifold, we say a Riemannian metric is left invariant if they are invariant under left group translation:
\[
\langle g \hat\xi_1, \,g \hat\xi_2 \rangle_{g} = \langle \hat\xi_1, \,\hat\xi_2 \rangle_e
\]
with $\hat\xi_1\in \mathfrak{g}$ and $\hat\xi_2\in \mathfrak{g}$. $e$ denotes the identity element.

The Riemannian Hessian on a manifold $\mathcal{G}$ is defined as the derivative of the gradient vector field along directions in the tangent space $T\mathcal{G}$. Given a tangent vector $u \in T_x\mathcal{G}$ and a vector field $V$, the derivative of $V$ in the direction of $u$ at point $x$ is denoted by $\nabla_u V$, where $\nabla$ is the affine connection (see~\cite{tu2017differential}) determined by the Riemannian metric $\langle \cdot, \cdot \rangle_x$. We now define the Riemannian Hessian as follows:
\begin{definition}[Riemannian Hessian]
Let $\mathcal{G}$ be a Riemannian manifold equipped with the Riemannian connection $\nabla$. The Riemannian Hessian of a function $f$ at a point $x \in \mathcal{G}$ is the linear map
\[
\mathrm{Hess}\,f(x): T_x\mathcal{G} \rightarrow T_x\mathcal{G}
\]
defined by:
\[
\mathrm{Hess}\,f(x)[u] = \nabla_u \, \mathrm{grad}\,f, \quad \forall u \in T_x\mathcal{G}.
\]
\end{definition}

It can be observed that, under a left invariant metric, the vector $r(\mathcal{Q})$ of~\eqref{statics} corresponds to the left trivialized gradient of the energy functional of the Cosserat rod, i.e., $[(\operatorname{grad}^L_{g_1} V)^\top,\dots,(\operatorname{grad}^L_{g_N} V)^\top]^\top=r(\mathcal{Q})$. Therefore, solving Problem~\ref{pr1} can be approached by computing the Riemannian gradient and Hessian of the energy functional $V(\mathcal{Q})$ at a given configuration set \( \mathcal{Q} = \{ g_i \in SE(3) \mid i = 0,1,\dots,n \} \) to determine a search direction  \( \{ \sigma_i \in \mathfrak{se}(3) \mid i = 0,1,\dots,n \} \) within the tangent space at \( \mathcal{Q} \) to minimize $V(\mathcal{Q})$. 
This enables the application of Riemannian optimization methods, analogous to Newton-type procedures in Euclidean spaces.

To compute the update direction at each iteration, we solve the following linear system:
\[
\mathrm{Hess}\,V(\mathcal{Q})\,\sigma = -r(\mathcal{Q}),
\]
where \( \sigma = [\sigma_1^\top, \dots, \sigma_N^\top]^\top \in \mathbb{R}^{6N} \) represents the stacked increment vector in the Lie algebra, and \( r(\mathcal{Q}) \) denotes the left trivialized gradient of the potential energy evaluated at the current configuration.

The left trivialized Hessian \( \mathrm{Hess}\,V(\mathcal{Q}) \) can be approximated using a second-order Taylor expansion of the energy functional. Specifically, we write:
\[
\mathrm{Hess}\,V(\mathcal{Q}) \approx K_t + \lambda \mathbb I,
\]
where \( K_t \) is the so-called tangent stiffness matrix (\cite{bathe2006finite}) given by:
\[
K_t = \int_0^L \left(J' + \operatorname{ad}_\xi J\right)^\top \mathcal K \left(J' + \operatorname{ad}_\xi J\right) \, \mathrm{d}s,
\]
and \( \lambda \mathbb I \) is a damping regularization term added to ensure positive definiteness and enhance numerical stability. The matrix \( K \) is the material stiffness matrix defined earlier.

Once the increment \( \sigma \) is obtained, the configuration is updated using the retraction map:
\[
g_i \gets g_i \boxplus \sigma_i.
\]

This process is repeated iteratively until convergence is achieved, leading to a configuration that satisfies the static equilibrium condition derived from the variational principle.

\subsection{Constrained system}
Soft robots frequently exhibit hybrid rigid–flexible architectures assembled in serial or parallel chains. When these chains form a closed loop, geometric constraints must be applied at the point of closure, as shown in \Cref{treerobot2}. 
\begin{figure}[h]
	\centering
	\includegraphics[width=0.4\textwidth]{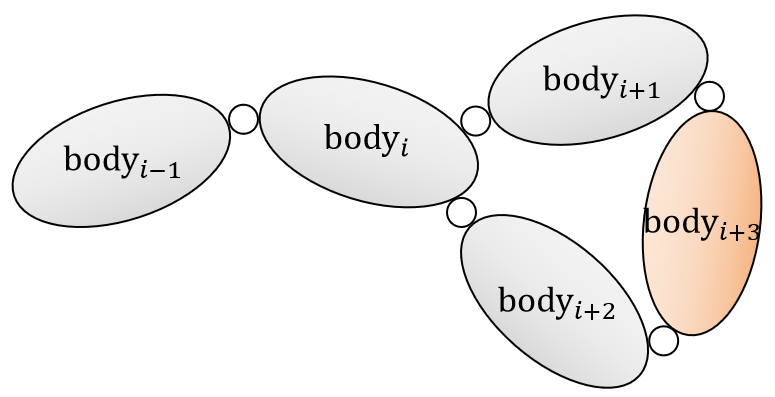}
	\caption{\textbf{Schematic of a closed-loop robotic structure.} Each ellipse represents a rigid body (e.g., $\text{body}_i$), and the small circles denote rotational joints. Unlike tree-structured robots, this configuration forms a closed kinematic loop: starting from $\text{body}_i$, connections propagate through $\text{body}_{i+1}$, $\text{body}_{i+3}$, and $\text{body}_{i+2}$, and return to $\text{body}_i$. Such closed-chain mechanisms are common in parallel robots and constrained articulated systems.}
	\label{treerobot2}
\end{figure}
These geometric constraints can be clamped joints, articulated hinges, or prescribed end-effector poses, and must be satisfied together with the rod’s elastic equilibrium.  

This subsection demonstrates how these constraints can be incorporated into the static analysis of the cumulatively parameterized Cosserat-rod model and presents a Newton–KKT solution strategy.

Let the constraint map be
\(
b(\mathcal{Q}) = [b_1(\mathcal{Q}),\dots,b_m(\mathcal{Q})]^\top\!:\; SE(3)^{\,n+1}\!\to\mathbb{R}^m,
\)
where each component \(b_j\) encodes, for instance, a fixed position, a relative orientation, or an inter-segment closure condition. Therefore, the statics under these constraint is given by:

\begin{problem}[Constrained energy minimisation]\label{pr3}
Given the configuration set
\(
\mathcal{Q} = \{ g_{i} \in SE(3) \mid i = 0,1,\dots,n \},
\)
find \(\mathcal{Q}\) that minimises the total potential energy
\[
V(\mathcal{Q}) \;=\; \int_{0}^{L} \!\bigl(\bar{U}_{\mathrm{int}} + \bar{U}_{\mathrm{ext}}\bigr)\,\mathrm{d}s \;+\; U_b
\]
subject to the constraint
\(
b(\mathcal{Q}) = {0}.
\)
\end{problem}

To solve Problem~\ref{pr3} iteratively, we linearise the constraints and employ a second-order expansion of the objective around the current iterate \(\mathcal{Q}\).  
The resulting local step is:

\begin{problem}[Quadratic Newton step]\label{pr4}
Find \(\sigma\in T_{\!\mathcal{Q}}SE(3)^{\,n+1}\) that minimises
\[
G(\sigma) \;=\; \frac12\,\sigma^\top \mathrm{Hess}\,V(\mathcal{Q})\,\sigma \;+\; r(\mathcal{Q})^\top\sigma
\]
subject to the linearised constraint
\[
b(\mathcal{Q}) + A(\mathcal{Q})\,\sigma \;=\; {0},
\]
where the Jacobian of constraints
\(A(\mathcal{Q}) \in \mathbb{R}^{m\times 6n}\) has blocks
\(A_{ij} = \nabla_{g_j} b_i\).
\end{problem}

\begin{table*}[t]
\centering
\begin{threeparttable}
\caption{Constraint functions and their Jacobians for different constraint types}
\label{tab:constraint}
\small
\begin{tabular}{p{3cm} p{4.5cm} p{5cm} p{3cm}}
\toprule
\textbf{Constraint Type} & \textbf{Constraint Function}$^{[1]}$ &
\textbf{Jacobiant w.r.t. $g_1$}$^{[2]}$ &
\textbf{Jacobiant w.r.t. $g_2$}$^{[3]}$ \\ 
\midrule

Fixed constraint &
$b = g_2 \boxminus g_1$ &
$\nabla_{g_1} b = \mathrm{d}\tau_b^{-1} \, \mathrm{Ad}_{\tau(b)}^{-1}$ &
$\nabla_{g_2} b = \mathrm{d}\tau_b^{-1}$ \\[6pt]

Articulated constraint &
$b = P_1 (g_2 \boxminus g_1)$ &
$\nabla_{g_1} b = P_1 \, \mathrm{d}\tau_b^{-1} \, \mathrm{Ad}_{\tau(b)}^{-1}$ &
$\nabla_{g_2} b = P_1 \, \mathrm{d}\tau_b^{-1}$ \\[6pt]

Sliding constraint &
$b = P_2 (g_2 \boxminus g_1)$ &
$\nabla_{g_1} b = P_2 \, \mathrm{d}\tau_b^{-1} \, \mathrm{Ad}_{\tau(b)}^{-1}$ &
$\nabla_{g_2} b = P_2 \, \mathrm{d}\tau_b^{-1}$ \\
\bottomrule
\end{tabular}
\begin{tablenotes}
\small
\item[1] Here, \( P_1 \) and \( P_2 \) are extraction matrices defined as: $P_1 = \begin{bmatrix} 0_{3 \times 3} & \mathbb{I}_3 \end{bmatrix}$, \quad
$P_2 = \begin{bmatrix} \mathbb{I}_3 & 0_{3 \times 3} \end{bmatrix}$.
\item[2] When the constraint error $b$ is small, $\mathrm{Ad}_{\tau(b)}^{-1}$ can be approximated by its first-order expansion:
$\mathrm{Ad}_{\tau(b)}^{-1} \approx \mathbb{I} - \mathrm{ad}_b$.
\item[3] When the constraint error $b$ is small, $\mathrm{d}\tau_b^{-1}$ can be approximated by its first-order expansion:
$\mathrm{d}\tau_b^{-1} \approx \mathbb{I} - \frac{1}{2} \, \mathrm{ad}_b$.

\end{tablenotes}
\end{threeparttable}
\end{table*}

Table~\ref{tab:constraint} summarizes the constraint functions and their Jacobians for different types of constraints defined between two configurations in $SE(3)$, $g_1$ and $g_2$.

Introducing the Lagrange multiplier vector \(\lambda\in\mathbb{R}^{m}\), the optimality conditions for Problem~\ref{pr4} are the saddle-point (KKT) equations
\begin{equation}\label{kkt}
    \begin{bmatrix}
        \mathrm{Hess}\,V(\mathcal{Q}) & A(\mathcal{Q})^\top\\[4pt]
        A(\mathcal{Q})                & {0}
    \end{bmatrix}
    \begin{bmatrix}
        \sigma\\[4pt] \lambda
    \end{bmatrix}
    =
    -
    \begin{bmatrix}
        r(\mathcal{Q})\\[4pt] b(\mathcal{Q})
    \end{bmatrix}.
\end{equation}

Solving \eqref{kkt} yields the constrained Newton increment \(\sigma\) alongside the multiplier update \(\lambda\). 

Finally, each control pose is advanced via the retraction map:
\[
g_i \;\gets\; g_i \,\boxplus\, \sigma_i,
\]

After which the process is repeated until the norm of the residual \(\|b(\mathcal{Q})\|\) and the gradient \(\|r(\mathcal{Q})\|\) fall below prescribed tolerances.  
This Newton–KKT scheme thus enforces both mechanical equilibrium and geometric compatibility in a unified, Lie-group–consistent framework.

The static formulation provides the equilibrium shape of the rod under given loads and constraints.  To capture how this configuration evolves when inertial and time–varying effects are present, we now extend the cumulative Cosserat model to dynamics.  Keeping the exact Lie group representation and constraints, we adopt a time-discrete variational approach, ensuring that the resulting integrator inherits the energy and momentum-preserving properties of the continuous system.
%
\section{Dynamics}\label{sec.dyna}
%
This section extends the cumulative Cosserat rod model to dynamic scenarios. We formulate the rod’s motion using a variational principle on \( SE(3) \), incorporating both kinetic and potential energy into a discrete Lagrangian framework. By discretizing the action integral, we derive structure-preserving time integration schemes. Specifically, we develop a symplectic integrator based on Lie group variational principles, which enables the accurate and stable simulation of soft rod dynamics.
%
\subsection{Variation-based discretization}
%
Consider a mechanical system with configuration space $\mathcal{G}$. Let the configuration at time $t$ be denoted by $x \in \mathcal{G}$ and the generalized velocity by $\dot{x} \in T_x \mathcal{G}$. The system’s Lagrangian is then defined using the kinetic and potential energies, $T(\dot{x})$ and $V(x)$, respectively, as:
\begin{equation}\label{lagr}
    \mathcal L(x, \dot{x}) := T(\dot{x}) - V(x).
\end{equation}

The core idea behind variational integrators is to discretize the Lagrangian~\eqref{lagr} to derive the discrete-time equations of motion~\cite{marsden2001discrete}. This approach ensures that the discrete Lagrangian retains the structure of the continuous one, thereby achieving excellent long-term energy conservation. The discrete Lagrangian $\mathcal L_d : \mathcal{G} \times  \mathcal{G}~ \rightarrow \mathbb{R}$ approximates the action integral over a time step as:
$$
    \mathcal L_d(x_k, x_{k+1}) \approx \int_{t_k}^{t_{k+1}} \mathcal L(x, \dot{x}) \, dt. 
$$

Consequently, the discrete form of the action integral over the entire time interval is expressed as:
$$
    \mathcal S_d = \sum_{k=0}^{N-1} \mathcal L_d(x_k, x_{k+1}). 
$$

To derive the discrete equations of motion, we consider variations in the tangent space of $\mathcal{G}$ and isolate the terms involving $\delta x_k \in T_{x_k} \mathcal{M}$, interpreting them as the discrete counterpart of integration by parts~\cite{marsden2001discrete}:
$$
    \delta \mathcal S_d = \sum_{k=1}^{N-1} \left\langle D_2 \mathcal L_d(x_{k-1}, x_k) + D_1 \mathcal L_d(x_k, x_{k+1}), \delta x_k \right\rangle
$$

Here, $D_i$ denotes the partial derivative concerning the $i$-th argument of $\mathcal L_d$, and $\langle \cdot, \cdot \rangle$ represents the canonical inner product. According to the least action principle, the discrete path is a stationary point of $\mathcal S_d$, which implies:
\begin{equation}\label{del}
    \nabla_{x_k}\mathcal S_d=D_1 \mathcal L_d(x_k, x_{k+1}) + D_2 \mathcal L_d(x_{k-1}, x_k) = 0. 
\end{equation}

We call~\eqref{del} the discrete Euler-Lagrange equation (DEL). Given the states $x_{k-1}$ and $x_k$, one can solve~\eqref{del} to obtain $x_{k+1}$. The total number of integration steps along discrete time is shown in Figure~\ref{integrator}. Further details about the total number of integration steps along discrete time can be found in~\cite{johnson2009scalable}.
\begin{figure}[h]
	\centering
	\includegraphics[width=0.49\textwidth]{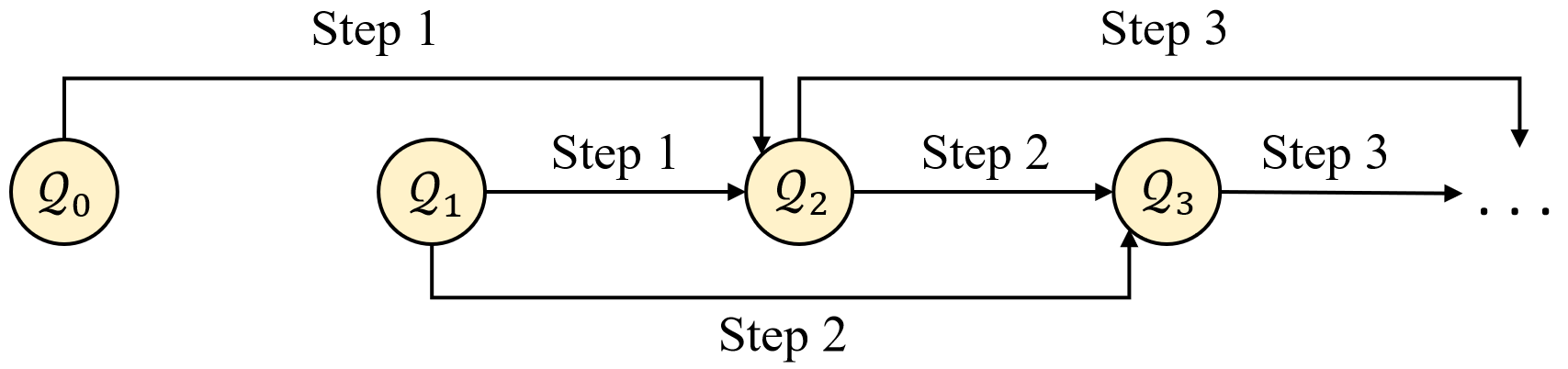}
	\caption{\textbf{Integration flowchat.} The discrete Euler–Lagrange equation is solved to determine the next configuration based on the previous and current ones. This process is iteratively applied to compute the entire trajectory. }
	\label{integrator}
\end{figure}

Next, we will see how to construct the symplectic integrator for the Cosserat rod model based on cumulative parameterization, whose configuration states are defined in $SE(3)$.
%
\begin{figure*}[t]
	\centering
	\includegraphics[width=0.98\textwidth]{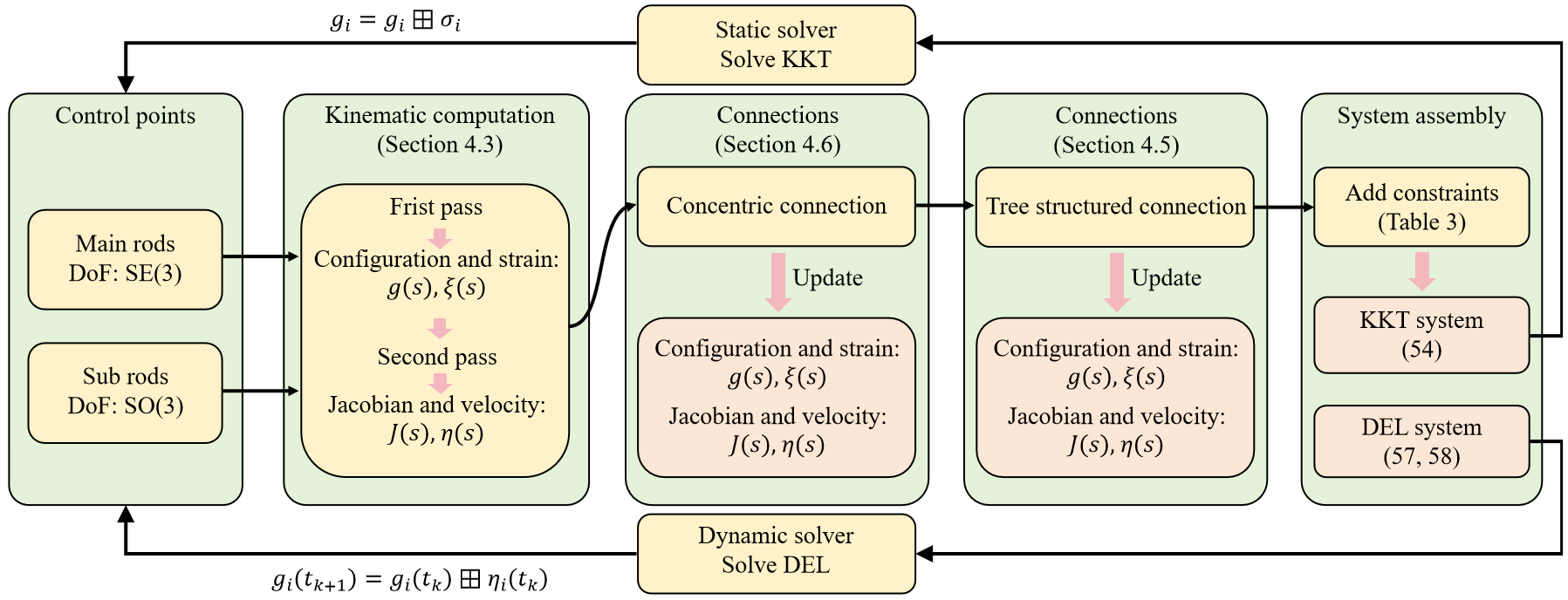}
	\caption{\textbf{Workflow of the proposed modeling method.} The process includes control point parameterization, kinematics computation, connection updates, system assembly, and static/dynamic solving.}
	\label{Workflow}
\end{figure*}

\subsection{Symplectic integrator on \texorpdfstring{$SE(3)$}{SE(3)}}
We now derive a time-stepping scheme for the Cosserat rod that advances the configuration states \(\mathcal{Q}_k = \{ g_{i,k} \}_{i=0}^{n} \subset SE(3)\) from time \(t_k\) to \(t_{k+1} = t_k + h\), while preserving the underlying variational structure of the system and honouring any kinematic constraints \(b_k(\mathcal Q_k)= 0\) that may arise from joints, fixtures, or
closure conditions.

For a given configuration set \(Q_k\) we approximate the action over one time step by the discrete Lagrangian
\[
\mathcal L_d(Q_k,Q_{k+1})
  \;=\;
  h\int_{0}^{L}\tfrac12\,\eta_k^\top\mathcal M\,\eta_k\,\mathrm ds
  \;+\;
  h\,V(Q_k),
\]
where \(h\) is the time step, \(\mathcal M\) is the mass-density matrix, and \(V(Q_k)\) is the total potential energy defined in Problem~\ref{pr1}. The body velocity field at time \(t_k\) is denoted by \(\eta_k(s)\). Here we define the velocity explicity, i.e., $\eta_k(s)=(g_{k+1}\boxminus g_k)/h$. Subsequently, the Lagrangian-based dynamics of a cumulative parameterization-based Cosserat rod under constraints is expressed as the following optimization problem:

\begin{problem}[Discrete action principle]\label{pr2}
Find the sequence \(\{Q_k\}_{k=0}^{N}\) that minimises the discrete action
\[
\mathcal S_d \;=\; \sum_{k=0}^{N-1}\mathcal L_d(Q_k,Q_{k+1}).
\]
subject to the constraints
\[
b(Q_k)= 0, 
\qquad k=0,\dots,N,
\]
where \(b:SE(3)^{\,n+1}\!\to\mathbb R^{m}\) collects all algebraic
constraints at time \(t_k\).
\end{problem}

To incorporate the constraints, we introduce Lagrange multipliers \(\lambda_k\in\mathbb R^{m}\) for each
time level, the augmented discrete action reads
\[
\widehat{\mathcal S}_d
  = \sum_{k=0}^{N-1}\!\Bigl(
      \mathcal L_d(Q_k,Q_{k+1})
      + \lambda_k^{\!\top} b(Q_k)
    \Bigr).
\]

Taking variations
\(\delta\zeta_{i,k}=g_{i,k}^{-1}\delta g_{i,k}\in\mathfrak{se}(3)\) and
requiring \(\delta\widehat{\mathcal S}_d=0\) yields:

\begin{proposition}[First-order optimality condition]\label{prop:dynamics}
For \(k=1,\dots,N-1\), the discrete Euler–Lagrange equation of the cumulative parameterized Cosserat rod is
\begin{equation}\label{eq:dynamics}
\begin{aligned}
\int_{0}^{L}\!J_k^\top
\bigl(
  \operatorname{Ad}_{\tau(h\eta_k)}^{-\top}\mathcal M J_k\dot q_k
  -\mathrm d\tau_{h\eta_k}^{-\top}\mathcal M J_{k-1}\dot q_{k-1}
\bigr)\,\mathrm ds
\\+
h\,\nabla_{Q_k}V_{k}+ h\,A_{k}^\top \lambda_{k}
=0,
\end{aligned}
\end{equation}
\begin{equation}\label{bbb}
    b(\mathcal Q_{k})+h\,A_k \dot{q}_k=0
\end{equation}
where \(J_k(\mathcal{Q}_k,s)\in\mathbb R^{6\times6(n+1)}\) is the Jacobian of the cumulative parameterisation at
\(t_k\) and \(\dot q_k=[\eta_{0,k}^{\!\top},\dots,\eta_{n,k}^{\!\top}]^\top\in\mathbb R^{6(n+1)}\)
collects the control-point body velocities. $A_{k} = \bigl[\nabla_{g_{0,k}}b \ \dots\ \nabla_{g_{n,k}}b\bigr]
      \in\mathbb R^{m\times 6(n+1)}$.
      
Given \(\dot q_{k-1}\) and \(Q_k\), equation~\eqref{eq:dynamics} and \eqref{bbb} determines
\(\dot q_k\) and $\lambda_k$; the configurations set $\mathcal{Q}_{k+1}$ at next time instant $t_k$ are then updated by
\[
g_{i,k+1} \;=\; g_{i,k}\boxplus h\,\eta_{i,k},
\qquad i=0,\dots,n.
\]
\end{proposition}
\begin{proof}

Suppose \(\eta_k\) is defined explicitly as \(\eta_k = (g_{k+1} \boxminus g_k)/h\), i.e.,
\[
g_{k+1} = g_k \boxplus (h \eta_k).
\]

According to Theorem~\ref{prop:vel_comp}, we have:
\[
\delta \eta_k = \mathrm{d}\tau^{-1}_{\eta_k h} \left( \delta \zeta_{k+1} - \operatorname{Ad}^{-1}_{\tau(\eta_k h)} \delta \zeta_k \right)/h.
\]

From this, we deduce:
$$
\begin{aligned}
	\mathrm{D}_1\mathcal{L}_d({Q}_k,{{Q}}_{k+1})=&\int_{0}^{L}-J_{k}^\top\operatorname{Ad}^{-\top}_{\tau(\eta_kh)}\mathrm{d}\tau_{\eta_kh}^{-\top}\mathcal M\eta_k\\&-h\nabla V(Q_k)\, \mathrm{d}s
\end{aligned}
$$
and similarly,
\[
\mathrm{D}_2 \mathcal{L}_d(Q_{k-1}, Q_k) = \int_0^L J_k^\top \, \mathrm{d} \tau_{\eta_{k-1} h}^{-\top}\mathcal M \eta_{k-1} \, \mathrm{d}s.
\]

The stationarity condition of the discrete action with constraints leads to:
\[
\begin{aligned}
\nabla_{Q_k} \widehat{\mathcal{S}}_d = \mathrm{D}_1 \mathcal{L}_d(Q_k, Q_{k+1}) + \mathrm{D}_2 \mathcal{L}_d(Q_{k-1}, Q_k) \\+ \nabla_{\mathcal{Q}_k} b(\mathcal{Q}_k) = 0,
\end{aligned}
\]
which yields the discrete dynamics equation~\eqref{eq:dynamics}. The condition \(\nabla_{\lambda_k} \widehat{\mathcal{S}}_d = b(Q_k)\) enforces the constraint.

To express the constraints in terms of the velocity \(\dot{q}_k\), we linearize them as:
\[
b(Q_{k+1}) = b(\mathcal{Q}_k) + h A_k \dot{q}_k + \mathcal{O}(h^2),
\]
and by neglecting the higher-order term \(\mathcal{O}(h^2)\), we obtain the approximation given in equation~\eqref{bbb}.
\end{proof}
\begin{remark}
Since the discrete time step is very small, both $\mathrm{d}\tau^{-1}_{\eta_kh}$ and $\operatorname{Ad}^{-1}_{\tau(\eta_kh)}$ can be approximated by their first-order series expansions. Specifically, we have $\mathrm{d}\tau^{-1}_{\eta_kh} = \mathrm{I} - \frac{1}{2} \mathrm{ad}_{h\eta_k} + \mathcal{O}(h^2)$ and $\operatorname{Ad}^{-1}_{\tau(\eta_kh)} = \mathrm{I} - \mathrm{ad}_{h\eta_k} + \mathcal{O}(h^2)$.
\end{remark}

In summary, the proposed unified Lie group framework provides a geometrically consistent foundation for modeling continuum soft robotic systems with rigid–soft coupling. The total workflows are summarized in~\Cref{Workflow}. It enables seamless integration of kinematic and dynamic constraints while operating directly on Lie manifolds, thus overcoming limitations commonly encountered in conventional finite element methods. Notably, the framework retains the modularity and extensibility characteristic of FEM-based formulations, making it well-suited for simulating complex, large-deformation behaviors in soft robots composed of both rigid and flexible components.
%
\section{Model Validation}\label{sec.ns}
In this section, we present a series of simulation scenarios to validate the performance of the proposed modeling method.
All simulations were conducted within the MATLAB environment with CPU {13th Gen Intel\textsuperscript{\textregistered} Core\textsuperscript{TM} i7-13850HX @ 2.10\,GHz}.
%
\subsection{Accuracy validation}
%
In this test, we validate the accuracy of the proposed modeling method by comparing the simulation results with the ground truth.
The numerical experimental setup is shown in~\Cref{fig:comparepls}: a soft rod is placed horizontally with one end fixed at the origin of the coordinate system and a moment applied at the other end. The physical parameters of the simulated soft rod are summarized in~\Cref{tab:cosserat_params}.
\begin{table}[ht]
  \centering
  \small
  \caption{Physical Parameters}
  \label{tab:cosserat_params}
  \begin{tabular}{l@{\hskip 115pt}c}
    \toprule
    \textbf{Parameter}        & \textbf{Value / Unit}            \\
    \midrule
    Length                   & 1.0 m                      \\
    Front end radius         & 3.0 cm                     \\
     Rear end radius         & 1.5 cm                    \\
    Young’s modulus          & $2.0\times10^{5}$ Pa     \\
        Possion's ratio        &  0.45\\
    \bottomrule
  \end{tabular}
\end{table}

In the simulation, the rod was dicretized by a cubic B-spline function consisting of 15 control points. The ground-truth solution is obtained by solving the static Poincaré equations of the Cosserat rod using the shooting method. 

Figure~\ref{fig:comparepls} illustrates the simulation results under different force and moment conditions. \Cref{fig:error1N1} presents the position error comparison between the configuration space parameterization method and the ground truth. The relative error between them is defined as:
$$
e(s) = \frac{1}{L}\| g(s) \boxminus g_t(s) \|_2.
$$
where \( g(s) \) denotes the simulated configuration of the soft rod, and \( g_t(s) \) denotes the ground-truth configuration. $L$ is the total length of soft rod.

We can observe that with the results of the proposed methed align closely with the ground truth, with the relative error remaining below 1\%.
\begin{figure}[h]
	\centering
\includegraphics[width=0.45\textwidth]{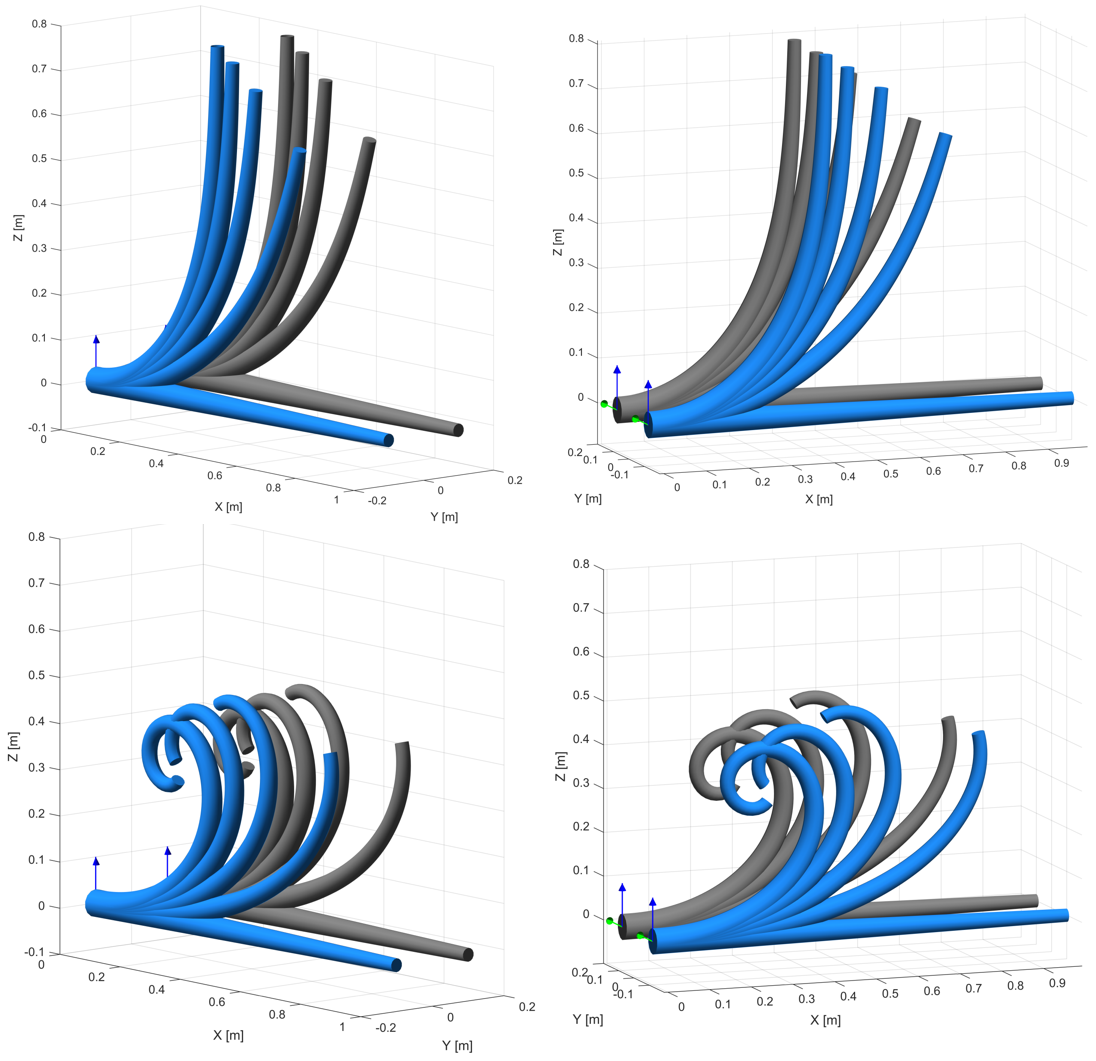}
	\caption{\textbf{Comparison of deformation results between the proposed method and the ground truth.} Blue rods represent the results from the proposed method, while gray rods correspond to the ground truth. In the upper panel, an upward vertical force is applied at the right end of the rods, increasing from 0~N to 1~N in increments of 0.25~N. In the lower panel, a moment about the negative \( y \)-axis is applied at the same location, increasing from 0~Nm to 0.2~Nm in increments of 0.05~Nm.}
	\label{fig:comparepls}
\end{figure}
\begin{figure}[h]
	\centering
\includegraphics[width=0.5\textwidth]{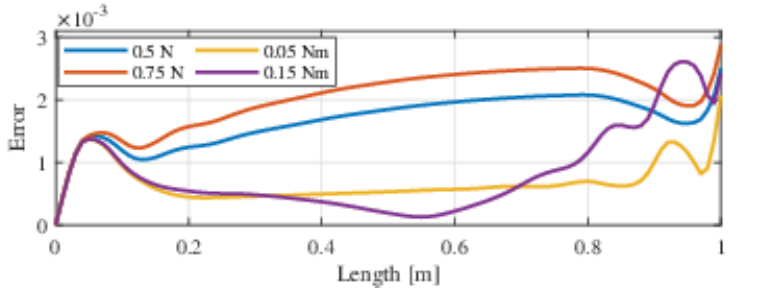}
	\caption{Error between the simulation results and the ground truth along the length of the soft rod under different end-tip loads.}
	\label{fig:error1N1}
\end{figure}
%
\subsection{Convergence behavior}\label{cbs}
%
In this subsection, we evaluate the convergence behavior of the proposed static solver. The test case involves the same soft rod used in the previous section, with one end fixed and a vertical force of $0.25$~N applied at the free end. 

The initial guess for the shape of the soft rod is its naturally straight configuration.
\Cref{fig:iterations} illustrates the evolution of the residual norm $\|{r+b}\|$, as defined in~\eqref{kkt}, over successive iterations of the static solver. The solver exhibits rapid convergence, with the residual dropping below $10^{-3}$ within approximately five iterations. Subsequently, the residual continues to decrease and stabilizes around the order of $10^{-10}$, indicating both the robustness and numerical stability of the proposed method.
\begin{figure}[h]
	\centering
\includegraphics[width=0.5\textwidth]{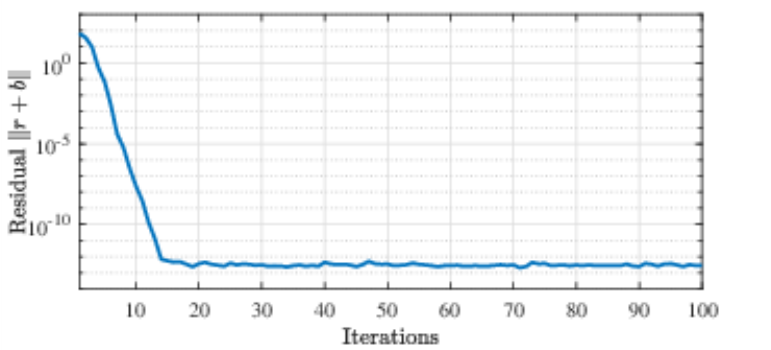}
	\caption{\textbf{Convergence behavior of the static solver.} The plot shows the residual norm $\|{r}\|$ over iterations during the static equilibrium computation.}
	\label{fig:iterations}
\end{figure}
\subsection{Convergence rate}
%

In this subsection, we investigate the effect of various spatial discretization strategies on the accuracy of the proposed model. The convergence rate quantifies how rapidly the discretization error decreases as the number of degrees of freedom increases. In other words, it reflects how efficiently the numerical solution approaches the exact solution with improved computational resolution.

We investigate the convergence rate under B-spline basis functions of varying orders. Specifically, for each order, we incrementally increase the number of control points and compare the simulation results against the ground-truth solution. The error is defined as the average norm of the configuration deviation across all points on the soft rod and is computed as:
\[
e = \int_0^L \| g(s) \boxminus g_t(s) \|_2 \, \mathrm{d}s,
\]
where \( g(s) \) denotes the simulated configuration of the soft rod, and \( g_t(s) \) denotes the ground-truth configuration.

The simulation scenario follows the same setup as described in Subsection~\ref{cbs}, with the B-spline order ranging from 2 to 5. 

\begin{figure}[h]
	\centering
\includegraphics[width=0.5\textwidth]{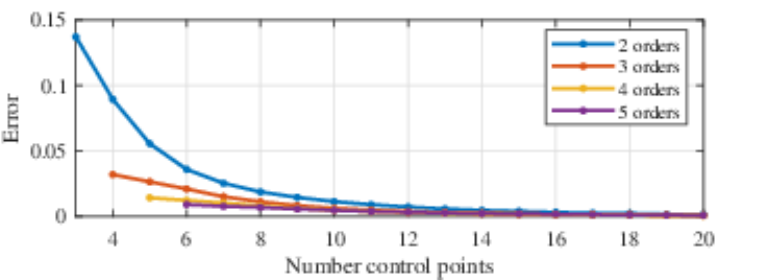}
	\caption{Convergence of the configuration error with increasing number of control points under B-spline basis functions of order 2 to 5.}
	\label{fig:num_point_order}
\end{figure}

The convergence behavior is shown in ~\Cref{fig:num_point_order}. Across all B-spline orders, the error decreases as the number of control points increases, demonstrating effective convergence.  Notably, when the number of control points is limited, higher-order basis functions (e.g., order 4 and 5) achieve significantly lower error than lower-order ones. This highlights the advantage of using smooth, high-order basis functions for efficient and accurate spatial discretization. Beyond a certain point, further increasing the number of control points yields diminishing improvements, suggesting a practical balance between accuracy and computational cost.
%
\subsection{Long-term energy behavior}
%
In this subsection, we validate the energy-conserving property of the proposed variational integrator. As illustrated in \Cref{fig:onerod}, a soft rod is fixed at its left end to the origin of the coordinate system, while a vertical upward force of 10 N is applied at the right end, placing the rod in a state of static equilibrium. The physical parameters of the rod are listed in~\Cref{tab:cosserat_paramsen}. This static configuration is taken as the initial state. Once the external force is removed, the rod begins to oscillate vertically. In the simulation, the rod was discretized using a third-order B-spline function with eight control points. The time step for the discrete integration was set to \( 0.01\,\mathrm{s} \). We have provided the simulation video in the supplemental material.

\begin{figure}[h]
\centering
\includegraphics[width=0.35\textwidth]{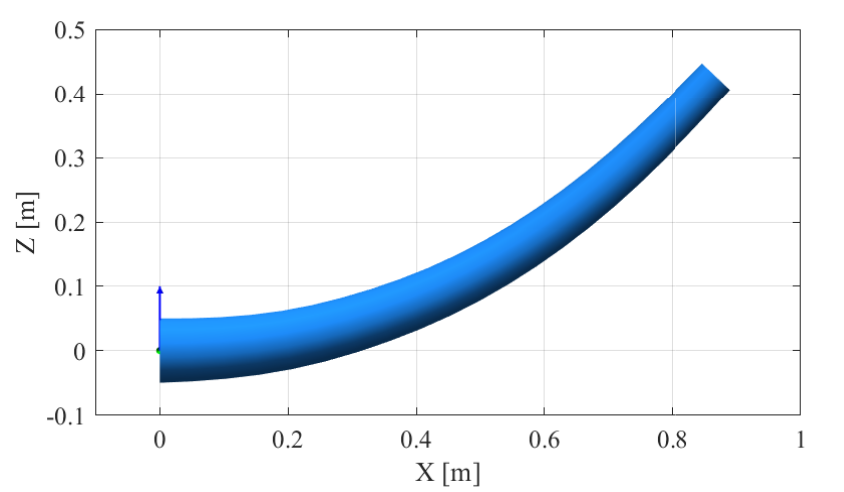}
\caption{\textbf{Initial equilibrium configuration of the soft rod.} The rod is fixed at its left end and subject to a 10 N upward force at the right end, resulting in a curved static equilibrium shape. This configuration is used as the initial condition for subsequent free oscillation analysis.}
\label{fig:onerod}
\end{figure}
\begin{table}[ht]
  \centering
  \small
  \caption{Physical Parameters}
  \label{tab:cosserat_paramsen}
  \begin{tabular}{l@{\hskip 115pt}c}
    \toprule
    \textbf{Parameter}        & \textbf{Value / Unit}            \\
    \midrule
    Length                   &  1.0 m                      \\
    Front end radius              & 5.0 cm                     \\
    Rear end radius             &  3.0 cm                    \\
    Young’s modulus          &  $2.0\times10^{6}$ Pa     \\
    Possion's ratio        &  0.45\\
    Viscosity modulus          &  $0$ Pa$\cdot$s     \\
    \bottomrule
  \end{tabular}
\end{table}
\begin{figure}[h]
\centering
\includegraphics[width=0.5\textwidth]{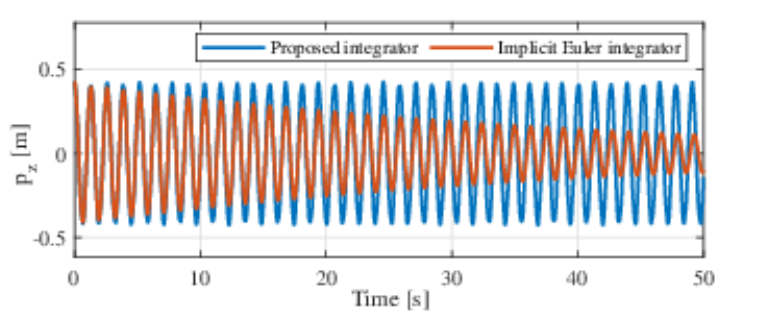}
\caption{\textbf{Time evolution of the rod tip position along the ${z}$-axis.} After the external force is removed, the rod undergoes sustained vertical oscillation. }
\label{fig:onerodpz}
\end{figure}
\begin{figure}[h]
\centering
\includegraphics[width=0.5\textwidth]{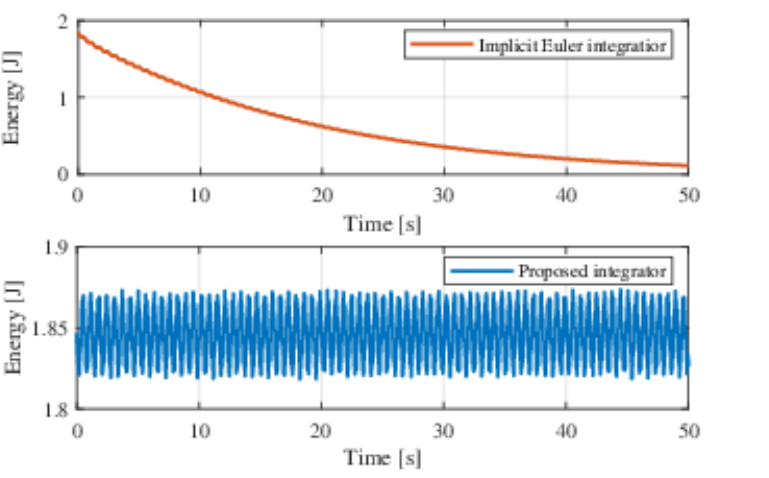}
	\caption{\textbf{Total energy evolution of the system over time.} The total energy is computed as the sum of kinetic and potential energy. The upper plot shows simulation results obtained using the implicit Euler integrator, while the lower plot presents results using the proposed symplectic integrator.}
	\label{fig:energy}
\end{figure}

\Cref{fig:onerodpz} shows the variation in the $z$ coordinate of the rod’s tip during the oscillation, over a time interval from 0~s to 50~s. Compared with the implicit Euler integrator (\cite{xun2024cosserat}, time step 0.01~s), the proposed symplectic integrator preserves the oscillation amplitude approximately constant throughout the simulation, exhibiting no significant decay.

\Cref{fig:energy} presents the time evolution of the total energy (kinetic plus potential) over the same time interval. Compared to the implicit Euler integrator with time step $0.01$ s, the proposed symplectic integrator maintains bounded energy with only slight fluctuations around the initial energy level. These results confirm the energy-preserving property of the symplectic integrator.
%
\section{Applications}\label{sec.applications}
%
In this section, we present several application cases that illustrate the generality and modularity of the proposed modeling framework. These examples demonstrate how our method can be systematically applied across a range of scenarios using reusable modeling components. Readers are referred to the supplementary video for the dynamic simulation results corresponding to each case.
%
\subsection{Cable-driven manipulator}
%
The simulation illustrates the deformation behavior of a cable-driven soft manipulator. As depicted in~\Cref{fig:cable12}, the manipulator comprises a blue soft rod whose top end is fixed at the origin of the coordinate system. Embedded along the central axis of the soft rod is a series of rigid triangular plates, sequentially arranged to form the internal structural framework. These plates are firmly attached to the soft rod to guide the cable paths.
\begin{figure}[h]
	\centering
\includegraphics[width=0.45\textwidth]{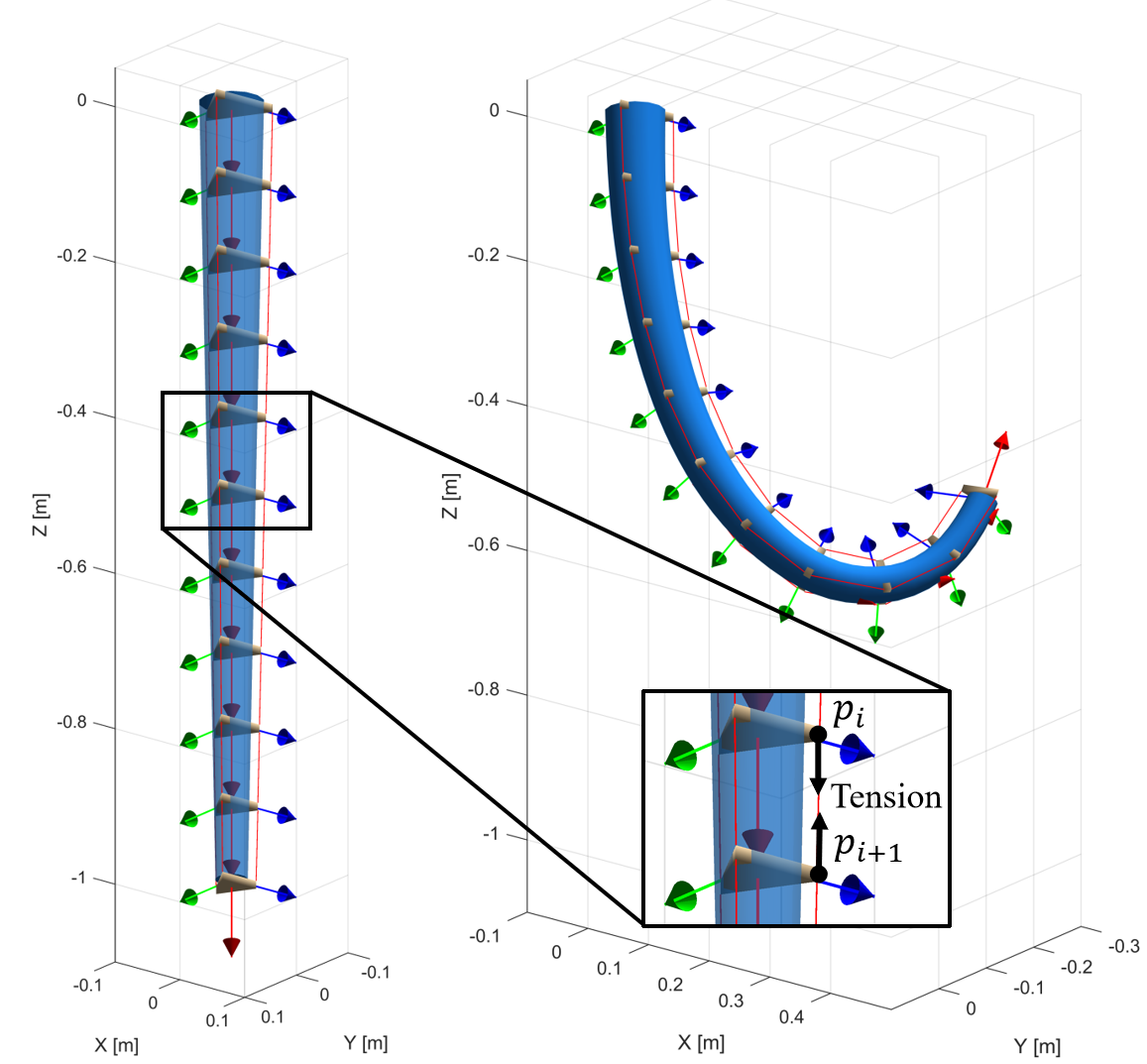}
	\caption{\textbf{Deformation of the cable-driven soft manipulator.} The left panel shows the initial configuration of the soft rod, with all cables untensioned and the structure in a straight, undeformed state. The right panel illustrates the deformed configuration resulting from the actuation of the cable. Three cables (red line) run through the vertices of the triangular plates, applying equal and opposite tension between the corresponding vertices of adjacent plates.}
	\label{fig:cable12}
\end{figure}

The modeling adopts a tree-structured robotic configuration, as illustrated in \Cref{fig:cable123}. The physical parameters of the rod are listed in~\Cref{tab:mani}.
\begin{table}[ht]
  \centering
  \small
  \caption{Physical Parameters}
  \label{tab:mani}
  \begin{tabular}{l@{\hskip 115pt}c}
    \toprule
    \textbf{Parameter}        & \textbf{Value / Unit}            \\
    \midrule
    Length                   &  1.0 m                      \\
    Front end radius              & 4.0 cm                     \\
    Rear end radius             &  2.0 cm                    \\
    Young’s modulus          &  $1.0\times10^{6}$ Pa     \\
    Possion's ratio        &  0.45\\
    Viscosity modulus          &  $10$ Pa$\cdot$s     \\
    \bottomrule
  \end{tabular}
\end{table}

\begin{figure}[h]
	\centering
\includegraphics[width=0.45\textwidth]{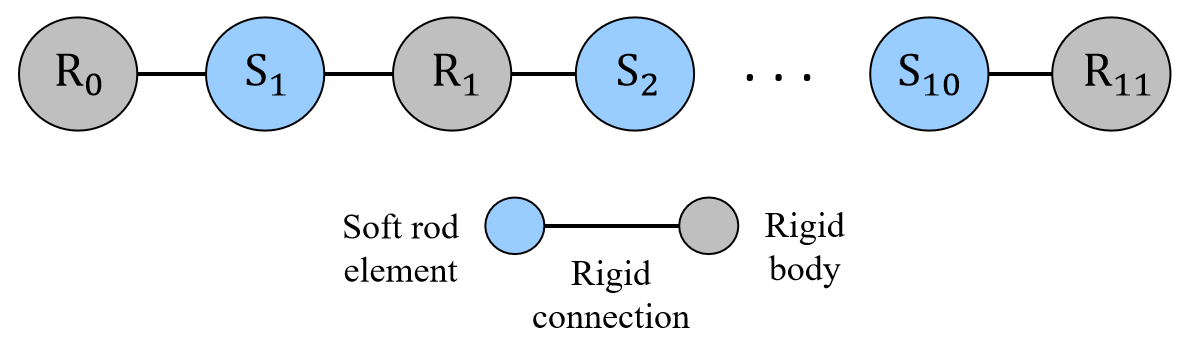}
	\caption{\textbf{The tree structure of the cable-driven soft manipulator.} The cable-driven robot consists of soft elements and rigid triangular plates connected in series. Each pair of adjacent rigid plates is linked by a soft segment, with the connections between them modeled as rigid joints. 
    }
	\label{fig:cable123}
\end{figure}

Three cables pass through the vertices of these triangular plates and extend along the length of the soft rod. The cable endpoints are anchored to the bottommost triangular plate, forming a closed-loop actuation system. By selectively actuating the cables, controlled deformation and curvature can be induced in the soft rod. To simulate the tendon-driven actuation of the manipulator, a pair of tensile forces is applied between each pair of adjacent rigid triangular plates along the line connecting their corresponding vertices, as illustrated in~\Cref{fig:cable12}. Let $p_i$ and $p_{i+1}$ denote the adjacent vertices through which the tendon passes. The cable tension applied at $p_i$ and $p_{i+1}$ are then given by
$T\frac{p_{i+1}-p_i}{\Vert p_{i+1}-p_i \Vert}$ and $-T\frac{p_{i+1}-p_i}{\Vert p_{i+1}-p_i \Vert}$
respectively, where $T$ represents the magnitude of the cable tension. 

In the left panel of the figure, the soft rod is in its initial, undeformed state, with all cables untensioned. In contrast, the right panel shows the rod under cable actuation, resulting in noticeable bending and curvature. The application of tension to specific cables alters the spatial configuration of the triangular plates, effectively manipulating the rod's overall posture.
%
\subsection{Concentric mechanisms}
%
Concentric tube robots, characterized by their nested, pre-curved elastic structures, exhibit complex deformation behaviors due to the coupling between tubes. Accurate simulation of these behaviors is essential for tasks such as trajectory planning, shape control, and mechanical design. In this subsection, we present representative simulation scenario to demonstrate the effectiveness of our cumulative parameterization approach in capturing the coupled kinematics and nonlinear deformation of concentric mechanisms. 
\subsubsection{Modeling framework}
As illustrated in~\Cref{fig:concentrictube}, the simulation model features a concentric tube robot composed of three nested tubes. The outermost tube is fixed to the world coordinate frame, functioning solely as structural support and remaining stationary during operation. The two inner tubes are pre-curved and actively actuated, allowing for both rotational and translational motion along the x-axis. The physical properties of each tube are detailed in~\Cref{tab:concen}. In the simulation, the first tube is parameterized with a fifth-order B-spline using 10 control points, the second (relative rotation) with a fifth-order B-spline using 8 control points, and the third (relative rotation) with a second-order B-spline using 3 control points.
\begin{figure}[h]
	\centering
\includegraphics[width=0.49\textwidth]{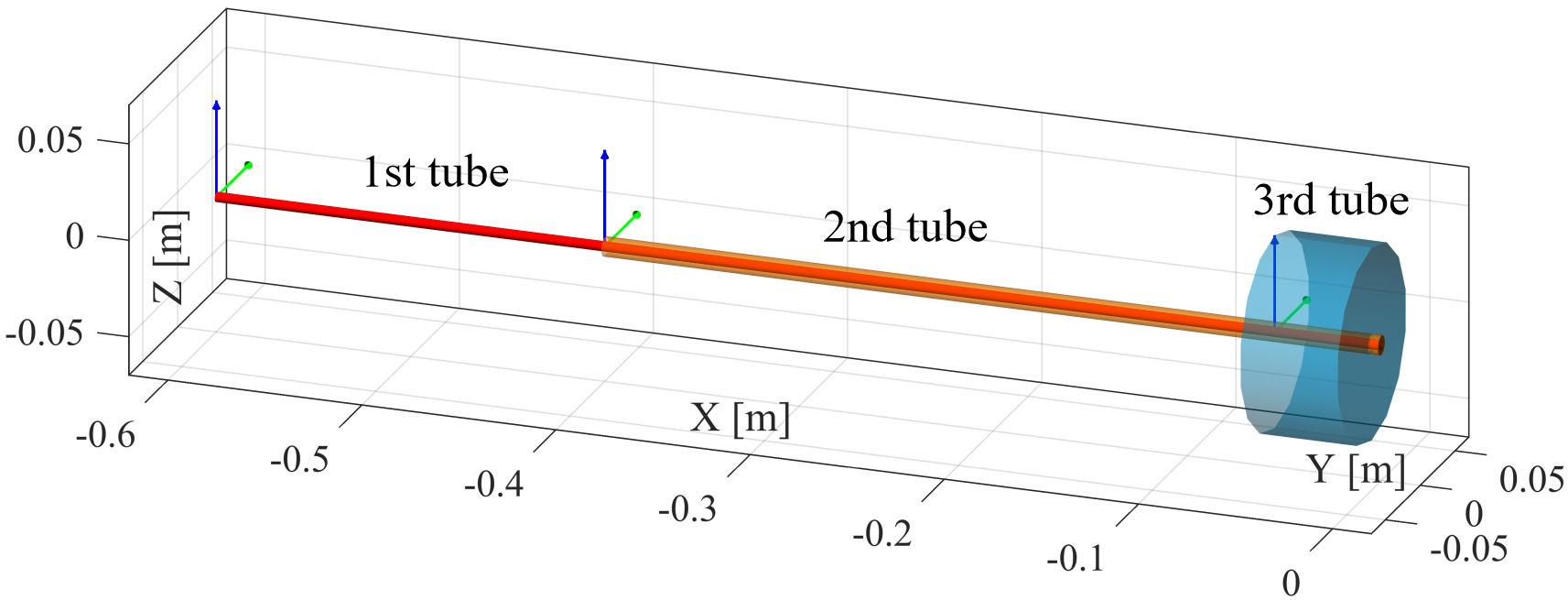}
	\caption{\textbf{Concentric tube robot}.The concentric tube robot consists of three tubes sharing a common central axis, with the third tube held stationary and fixed to the world coordinate frame.
}
\label{fig:concentrictube}
\end{figure}
\begin{table}[ht]
  \centering
  \small
  \caption{Physical Parameters}
  \label{tab:concen}
  \begin{tabular}{lccc}
    \toprule
    \textbf{Parameter} & \textbf{1st tube} & \textbf{2nd tube} & \textbf{3rd tube} \\
    \midrule
    Length (m)             &  0.6   &  0.4   &  0.05 \\
    Radius (mm)            &  2.5   &  5     &  50 \\  
    Young’s modulus (GPa)  &  2     &  2     &  10 \\
    Possion's ratio        &  0.45  & 0.45   & 0.45 \\
    Pre-curvature (1/m)    &  6     &  4     &  0 \\
    \bottomrule
  \end{tabular}
\end{table}

\begin{figure}[h]
	\centering
\includegraphics[width=0.4\textwidth]{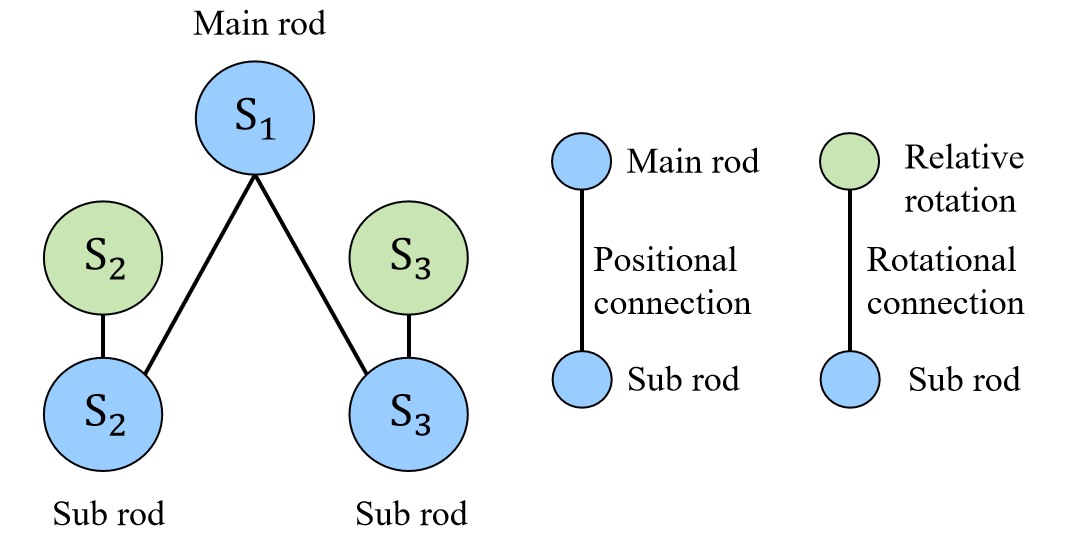}
	\caption{\textbf{The concentric structure of the three-tube concentric robot.} The first (innermost) tube, serving as the main tube, possesses independent degrees of freedom. The remaining tubes are subordinate, with their degrees of freedom determined by the configuration of the main tube and their relative rotations with respect to it.}
	\label{fig:concenstruc}
\end{figure}

Using the proposed concentric structural modeling framework, concentric tube robots composed of an arbitrary number of nested tubes can be easily and intuitively modeled. \Cref{fig:concenstruc} illustrates the general schematic of the model construction, with a specific example corresponding to the three-tube concentric robot simulated in this subsection.
\subsubsection{Simulation result and analysis}
For a pair of concentric tubes, when the relative rotation angle at the input exceeds a certain threshold, the rotation angle at the output exhibits a sudden transition — commonly referred to as the snapping phenomenon, a characteristic nonlinear behavior commonly observed in concentric tube robots. This simulation aims to investigate the snapping phenomenon by comparing the simulation results with the analytical solution. 
\begin{figure}[h]
\centering
\includegraphics[width=0.40\textwidth]{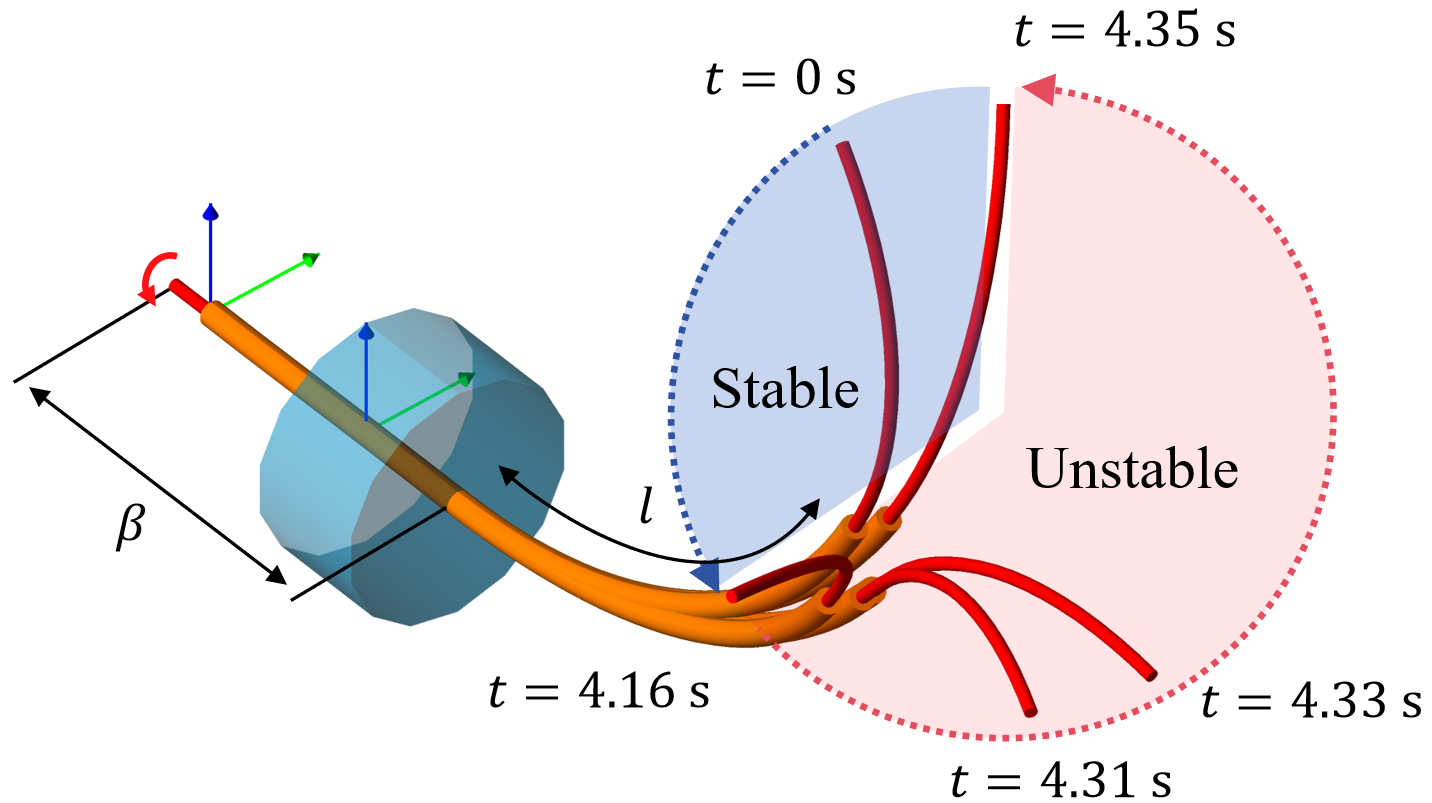}
\caption{\textbf{Representative deformation stages of the concentric tube robot.} 
$\beta$ denotes the distance from the first tube’s rotation tip to the bending entry point, and $l$ is the overlap length of the bent sections. 
Initially, the tubes remain stable. At \( t = 4.16 \) s, the snapping event induces instability, followed by a return to stability at \( t = 4.35 \) s.
}
\label{fig:concensnap}
\end{figure}

We captured this phenomenon in the simulation. \Cref{fig:concensnap} presents snapshots of the simulation at various time instances. The first tube and the second tube are advanced along the \( x \)-axis from a horizontal initial configuration shown in \Cref{fig:concentrictube}. At \( t = 0 \)~s, they reach the starting position for rotation (\( \beta = -0.21 \)~m, \( l = 0.23 \)~m), where the relative rotation angle \( \theta(\beta) \) is zero. From this point onward, the first tube begins rotating at a constant angular velocity of 1~rad/s. Prior to \( t = 4.16 \)~s, the rotation remains in a stable regime. However, once \( \theta(\beta) \) approaches approximately $4.1$~rad, the snapping event occurs, resulting in a sudden transition and system instability.

For a pair of concentric tubes, when the relative rotation angle at the input exceeds a certain threshold, the rotation angle at the output exhibits a sudden transition — commonly referred to as the snapping phenomenon. The relationship between the input and output relative rotation angles closely approximates an S-curve, which corresponds to the solution of the following nonlinear differential equation:
\begin{equation}\label{snapping}
    \begin{aligned}
        &\theta'' - \lambda \sin(\theta) = 0 \\
\theta(0) = &\theta(\beta) -\beta\theta'(0), \quad
\theta'(l) = 0
    \end{aligned}
\end{equation}

Here, the coefficient \( \lambda \) is determined by the physical parameters of the tubes (e.g., torsional stiffness, bending stiffness, cuvature). In our simulation, \( \lambda = 37.8 \), as computed from these parameters. For further details, readers are referred to~\cite{gilbert2015elastic}.

\begin{figure}[h]
\centering
\includegraphics[width=0.35\textwidth]{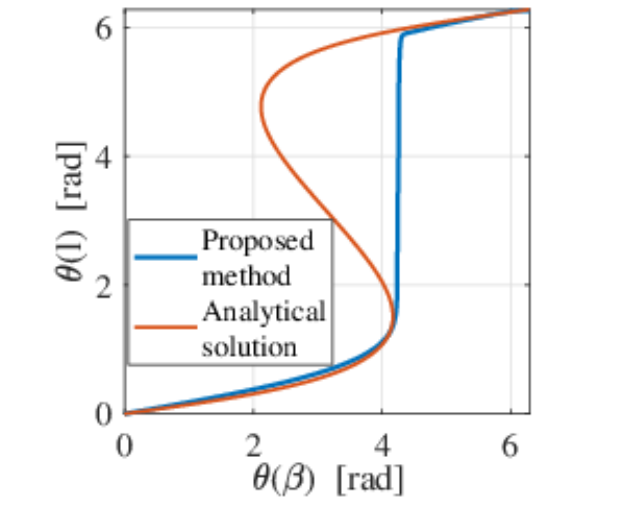}
\caption{\textbf{S-curve.} 
Comparison between the simulation results obtained using the proposed method and the analytical solution.}
\label{fig:S-curve}
\end{figure}
\Cref{fig:S-curve} compares the analytical solution with the simulation results over the full rotation range from $0$ to \( 2\pi \), illustrating the characteristic S-shaped snapping curve. As shown, the simulation results closely match the analytical solution and accurately capture the snapping point, which occurs near \( \theta(\beta) = 4.1 \)~rad. For a complete visualization of the simulation process, please refer to the supplemental video. 

This simulation demonstrates that the proposed modeling framework can accurately capture the nonlinear behavior of concentric tube robots, including the onset of snapping.
%
\subsection{Parallel mechanisms}
%
To evaluate the efficiency, generality, and robustness of the proposed modeling and simulation framework, we simulate a representative and mechanically complex parallel soft-rigid structure. This example demonstrates two key strengths of our approach: the ease of constructing complex soft robotic systems via the proposed parameterization method and the computational efficiency achieved through structured sparsity in the formulation.
\begin{figure}[h]
	\centering
\includegraphics[width=0.49\textwidth]{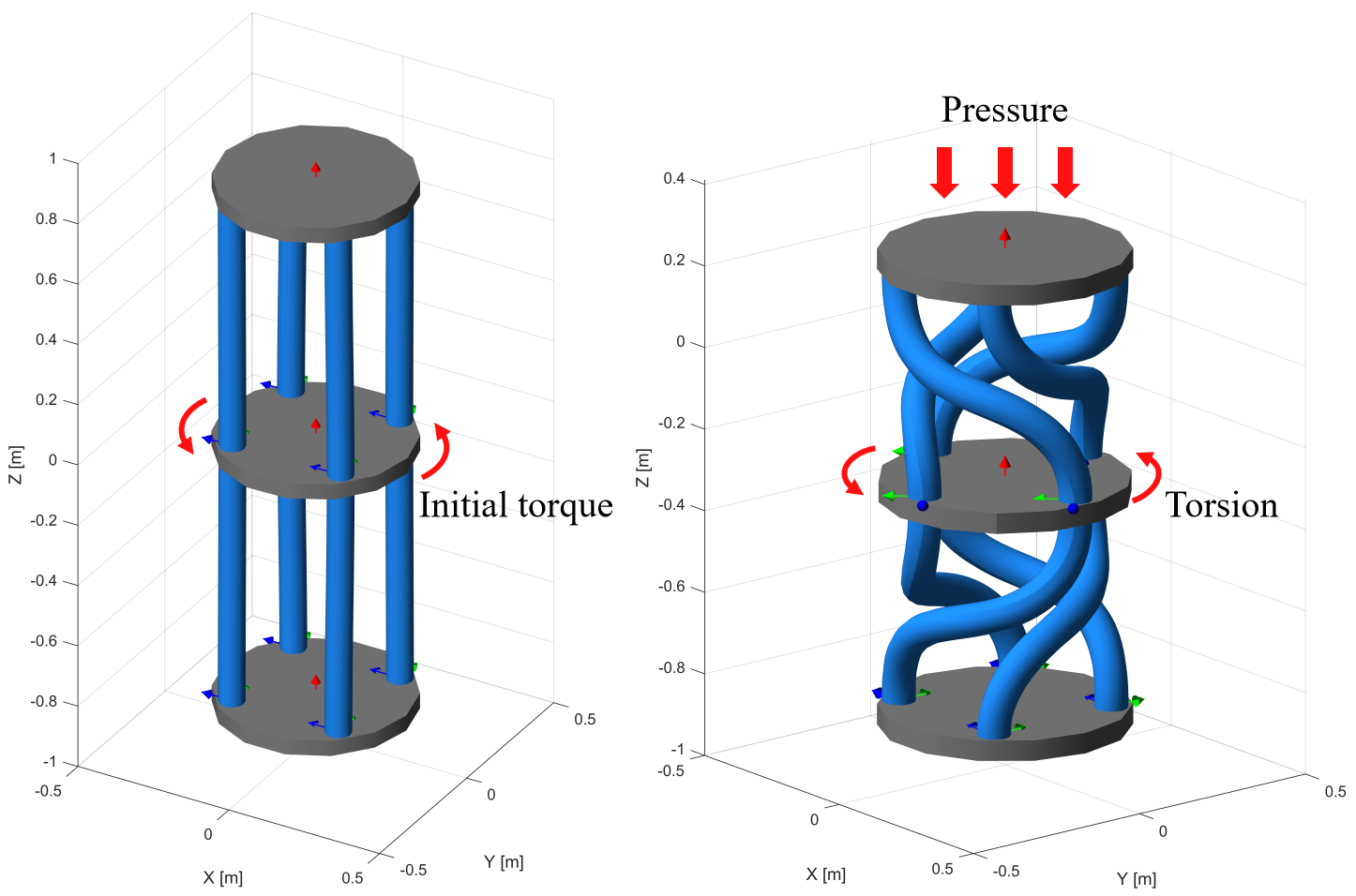}
	\caption{\textbf{Chiral structure under compressive loading.} The left image shows the initial unloaded state, in which straight vertical rods are symmetrically arranged between three horizontal circular plates. These rods are stress-free and maintain parallel alignment without deformation. The right image illustrates the deformed state under vertical compressive loading. The rods exhibit significant twist-buckling behavior, undergoing compression, bending, and torsion simultaneously.}
	\label{fig:twisting}
\end{figure}

We consider a chiral parallel mechanism inspired by~\cite{fang2025large}, which features coupled deformation modes, including bending, torsion, and axial compression. As shown in~\Cref{fig:twisting}, the structure consists of four soft rods connecting two rigid circular plates. Initially stress-free and aligned vertically, the rods exhibit complex twist-buckling behavior when subjected to vertical compressive forces, serving as a stress-redistributing and energy-absorbing mechanism.  To explore the simulation behavior under varying initial configurations, we apply different levels of pre-twist torque to the central disk before compressive loading. The structure then evolves through three key states: the unloaded configuration, a slightly pre-twisted shape, and the final load-deformed configuration.
\begin{table}[ht]
  \centering
  \small
  \caption{Physical Parameters}
  \label{tab:cosserat_params2}
  \begin{tabular}{l@{\hskip 115pt}c}
    \toprule
    \textbf{Parameter}        & \textbf{Value / Unit}            \\
    \midrule
    Length                   &  0.8 m                      \\
    Radius              & 4.0 cm                     \\
    Young’s modulus          &  $1.0\times10^{5}$ Pa     \\
    Possion's ratio        &  0.45\\
    \bottomrule
  \end{tabular}
\end{table}

In the unloaded state (\Cref{fig:twisting}, left), the structure consists of four vertical rods symmetrically connecting two circular plates, remaining parallel and deformation-free. A small vertical torque is then applied to the central disk to introduce a controlled pre-twist, with varying torque levels used to study different pre-twist conditions.

Under vertical compression (\Cref{fig:twisting}, right), the initially straight, stress-free rods deform into helical shapes through twist-buckling, combining compression, bending, and torsion. This multi-modal deformation redistributes stress and significantly enhances energy storage. Simulation results show that the chiral structure exhibits higher specific energy and energy density under load compared to conventional axial buckling.

The material parameters of the soft rod include the elastic modulus $E_s$, initial length $L_0$, and rod radius $r$. These physical parameters are summarized in \Cref{tab:cosserat_params2}. 
%
\subsubsection{Modeling framework}
\begin{figure}[t]
    \centering
    \includegraphics[width=0.49\textwidth]{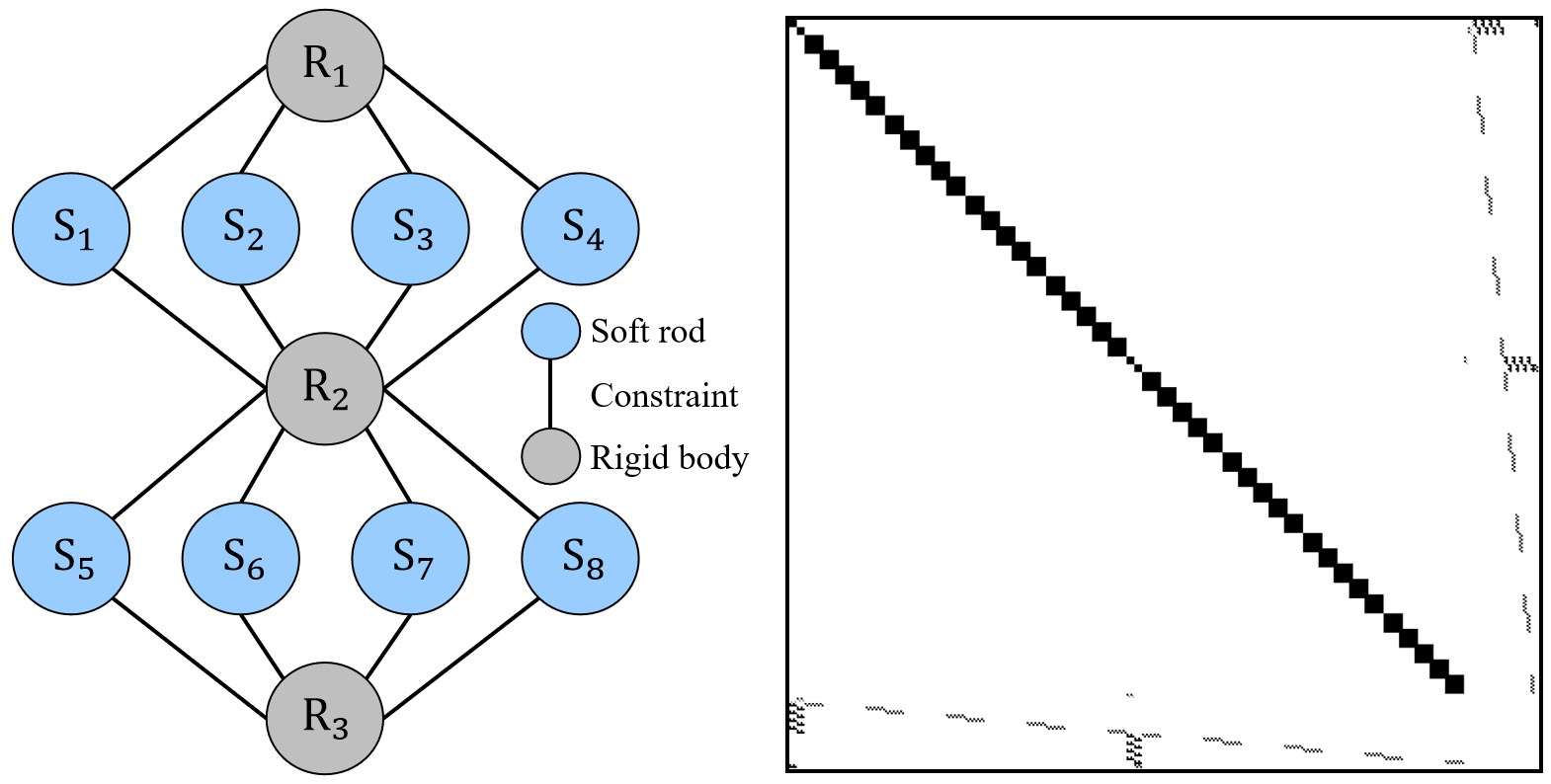}
    \caption{\textbf{Structure of the simulated parallel soft–rigid mechanism and its Jacobian sparsity pattern.} 
    The left panel illustrates the topological model of a parallel mechanism composed of rigid bodies (\( R_1 \)–\( R_3 \)) connected by soft rods (\( S_1 \)–\( S_8 \)). Nodes represent components, while edges represent geometric constraints between them. The right panel shows the sparsity pattern of the Jacobian matrix derived from the static KKT system, which includes both the Hessian and the Jacobians of all constraints. The structured sparsity reflects the modular and localized nature of the hybrid soft–rigid system.}
    \label{fig:Jacoparallel2}
\end{figure}

The system modeling architecture of the simulated parallel soft–rigid mechanism is shown in the left panel of Figure~\ref{fig:Jacoparallel2}. It is represented as a graph, where nodes correspond to soft (blue) and rigid (gray) components, and edges represent geometric constraints. Each soft rod is divided into five elements, with each component parameterized by a cubic B-spline using five control points.

Following this modeling architecture, the right panel of Figure~\ref{fig:Jacoparallel2} illustrates the sparsity pattern of the Jacobian matrix (of dimension \( 1172 \times 1172 \)) derived from the KKT system used for static analysis. This sparse structure arises from the localized coupling between components, facilitating efficient numerical computation. 

Owing to the symmetric and sparse nature of the matrix \( A \), we employ a sparse \( \text{LDL}^\top \) solver to solve the KKT linear system efficiently. In our implementation, a single solve for a KKT linear system of size \(1172 \times 1172\) takes approximately 2 ms on Matlab, demonstrating the practical efficiency enabled by the structured sparsity.
%
%
\subsubsection{Simulation results and analysis}
The simulation results obtained using our method are consistent with those reported in~\cite{fang2025large}, which used the finite element method, validating that our framework accurately captures the characteristic deformation behavior of chiral structures under compressive loading.
\begin{figure}[t]
	\centering
\includegraphics[width=0.49\textwidth]{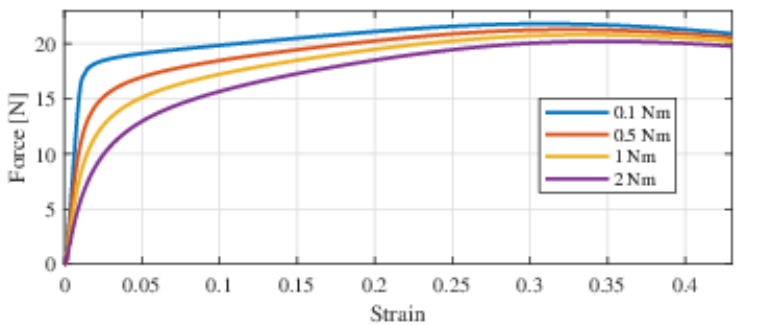}
	\caption{\textbf{Strain–force response under different initial torques.} The plot illustrates the vertical compressive force versus strain relationship for chiral structures with varying initial torques (0.1 Nm, 0.5 Nm, 1 Nm, and 2 Nm). Lower initial torque results in a stiffer response with rapid force increase, while higher torques induce more pronounced twist-buckling and bending, allowing for greater strain accommodation and enhanced energy absorption performance.}
	\label{fig:strain stress}
\end{figure}

\Cref{fig:strain stress} presents the simulated strain–force relationships. Each curve corresponds to a different pre-twist configuration, characterized by distinct initial torque magnitudes.
It illustrates how increasing the initial pre-twist torque leads to more gradual force–strain responses. At low torque (e.g., 0.1 Nm), the structure exhibits a steep rise in force as it resists deformation. At higher torques (e.g., 2 Nm), pronounced twist-buckling allows greater strain accommodation with lower force increase, reflecting enhanced flexibility and energy absorption. 

These results demonstrate that the proposed method reliably reproduces complex multi-modal deformation without case-specific tuning. It offers an effective tool for simulating soft–rigid mechanisms and provides a theoretical foundation for the design and validation of advanced robotic structures.

\subsection{Finger design}
%
In certain soft robotics applications, the natural shape of a soft rod is not necessarily a straight line; instead, a predefined, nontrivial shape may be required to achieve desired motion characteristics. This introduces the need for deliberate design of the rod’s natural geometry. Leveraging the isogeometric nature of the proposed method — where the degrees of freedom used during geometric design are identical to those used in simulation — we enable a seamless and efficient design-to-simulation workflow for soft robotic structures.

The simulations in this subsection focus on a cable-driven robotic finger, as illustrated in \Cref{fig:finger}. The finger consists of a rigid body, with its joints replaced by soft rods. Different natural shapes of these soft rod joints result in distinct motion behaviors. Our framework facilitates the intuitive design of these natural shapes by allowing direct manipulation of B-spline control points.
\begin{figure}[h]
	\centering
\includegraphics[width=0.49\textwidth]{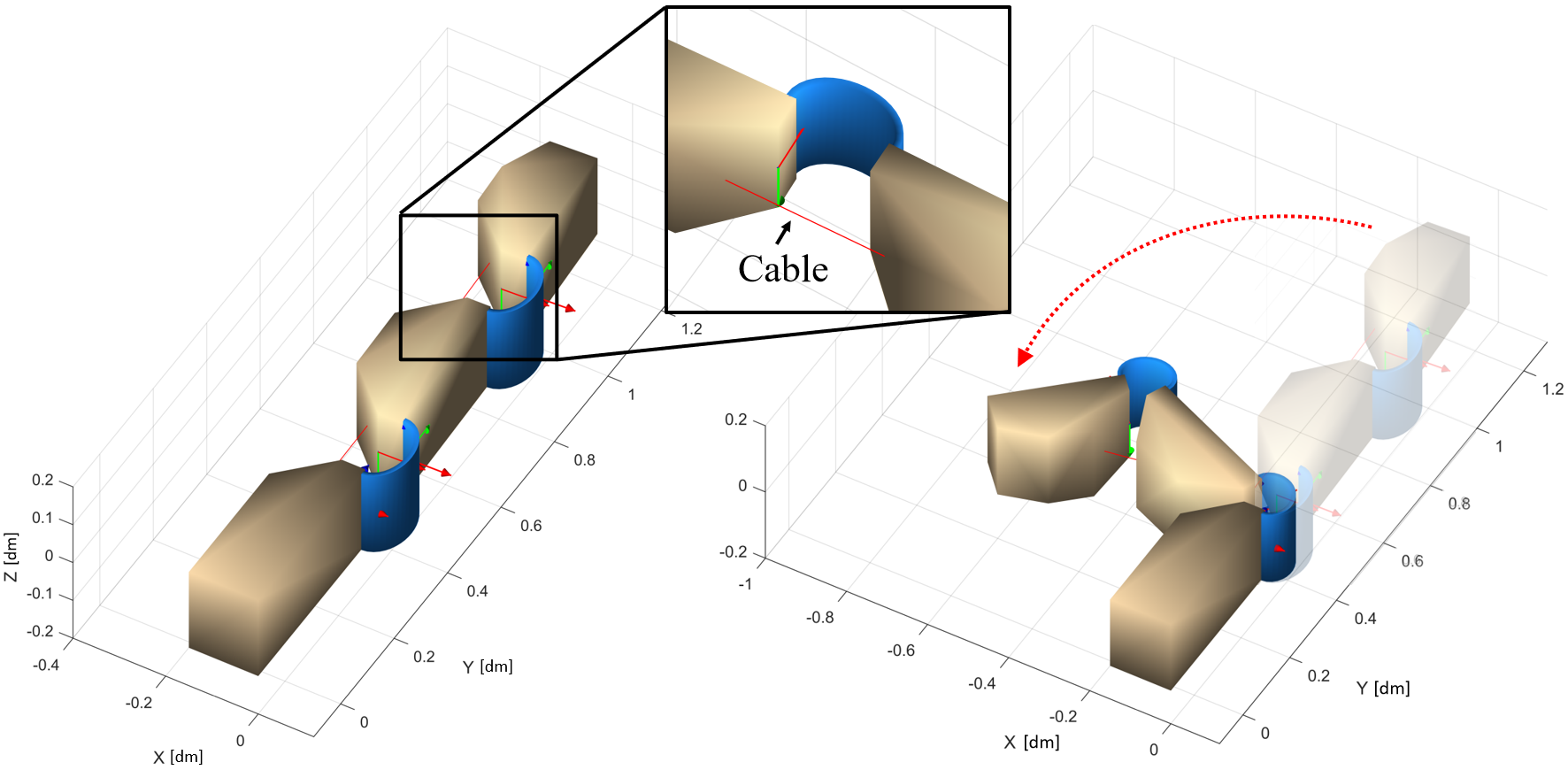}
	\caption{\textbf{Deformation of a cable-driven soft finger.} The left panel shows the initial undeformed configuration of the soft finger, where the rigid bone segments (yellow) are aligned and the soft joint rods (blue) remain straight with untensioned cables (red line). In the right panel, tensile forces applied to the embedded cables induce bending at the joints, causing the finger to flex.}
	\label{fig:finger}
\end{figure}

As depicted in \Cref{fig:fingerspline3d2}, we designed two types of soft rod joints with distinct natural geometries. The first configuration uses a second-order B-spline with four control points, while the second employs a third-order B-spline with five control points. The physical parameters of their cross section are shown in \Cref{tab:finger}.
\begin{figure}[h]
	\centering
\includegraphics[width=0.49\textwidth]{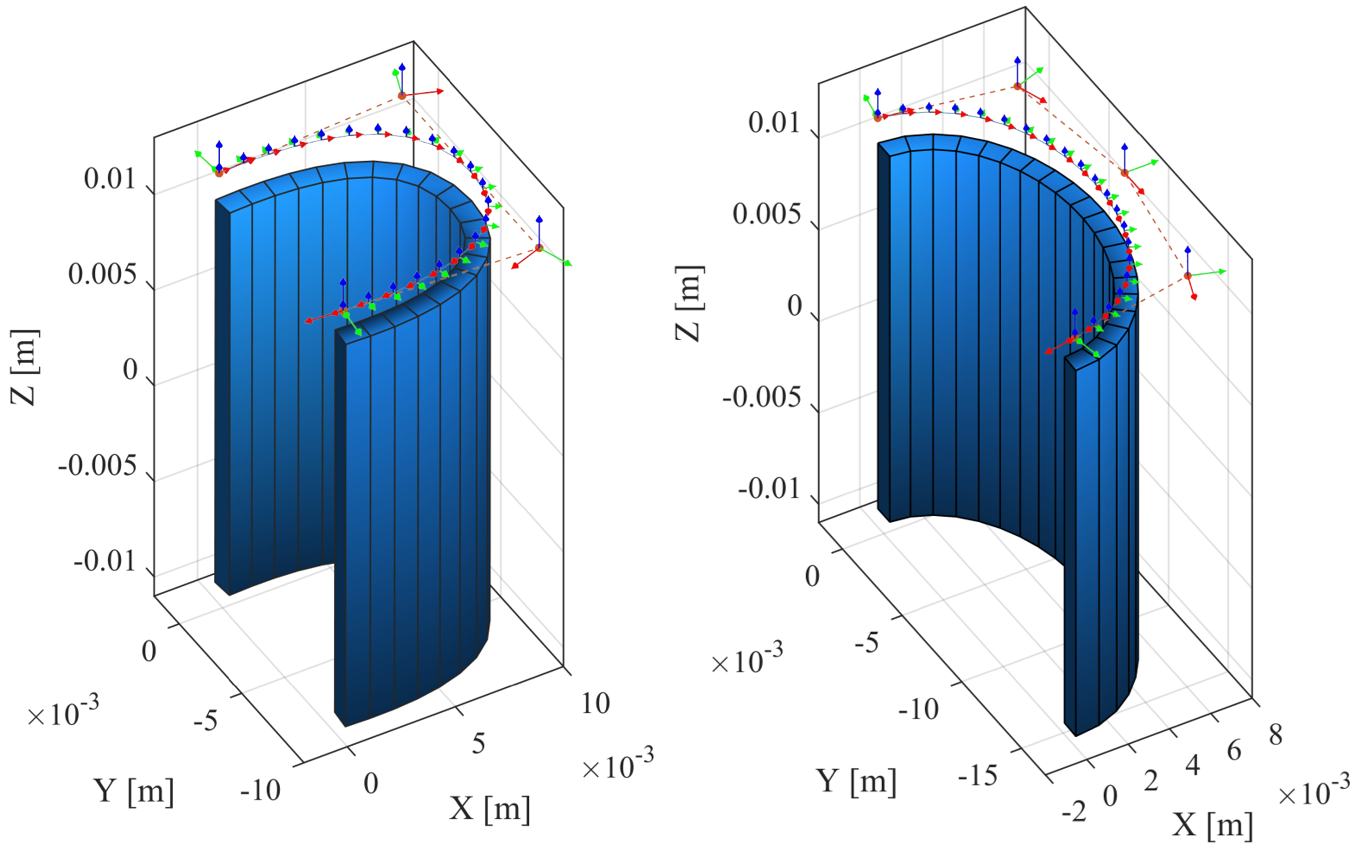}
	\caption{\textbf{Geometric design of the soft joint.} The shape of the soft joint is parameterized using a cubic B-spline curve defined by five control points (frame). The left panel shows the B-spline curve (blue) along with the control polygon (dashed orange) and control points (frame), while the right panel displays the corresponding 3D revolved geometry used to construct the soft joint.}
	\label{fig:fingerspline3d2}
\end{figure}
\begin{table}[ht]
  \centering
  \small
  \caption{Physical Parameters}
  \label{tab:finger}
  \begin{tabular}{l@{\hskip 115pt}c}
    \toprule
    \textbf{Parameter}        & \textbf{Value / Unit}            \\
    \midrule
    Thickness                   &  1 mm                      \\
    Width              & 20 mm                     \\
    Young’s modulus          &  $5$ GPa     \\
    \bottomrule
  \end{tabular}
\end{table}
\begin{figure}[h]
	\centering
\includegraphics[width=0.48\textwidth]{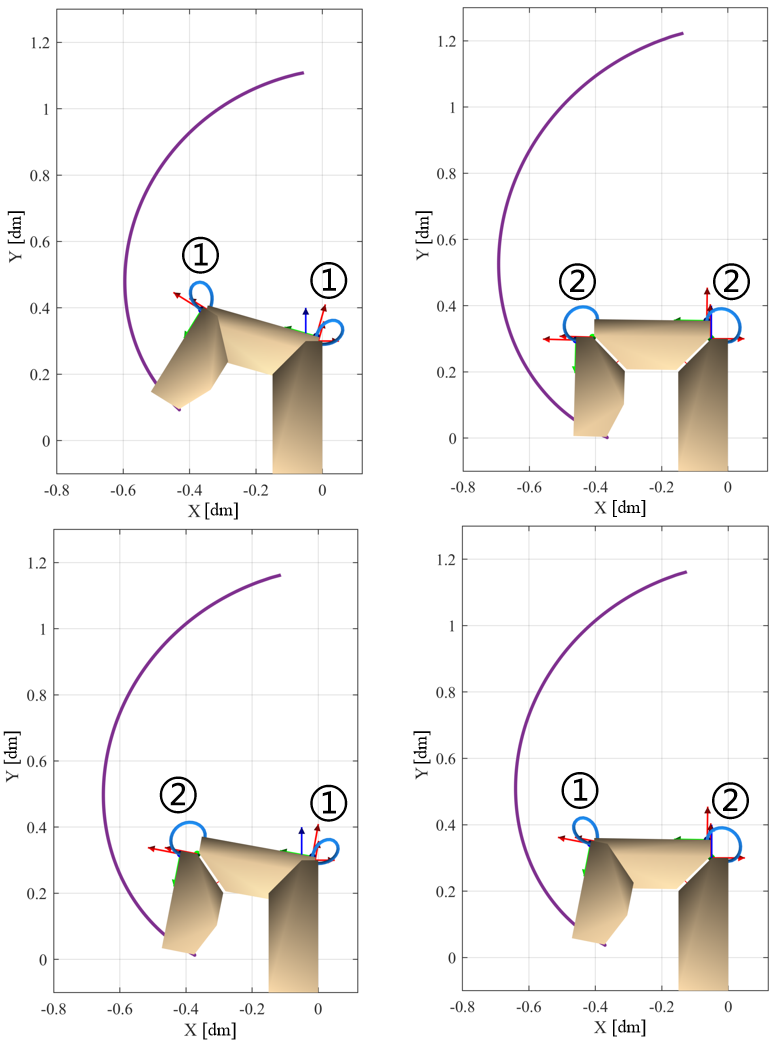}
	\caption{\textbf{Trajectory of the finger's tip with the different soft joint combination.} The labels \ding{172} and \ding{173}  indicate two distinct natural shapes of the soft joints. Each subfigure corresponds to a different joint combination in the robotic finger.  }
	\label{fig:fingerscomp}
\end{figure}

These two types of soft rod joints were then incorporated into the two joints of the robotic finger in various combinations. In all simulation scenarios, the same cable actuation was applied, with the tensile force gradually increasing from 0 to 3.3 N. \Cref{fig:fingerscomp} presents the simulation results, specifically the fingertip trajectories corresponding to each joint configuration. The results demonstrate that different combinations of soft rod joints lead to distinct fingertip motion paths. This confirms that by designing the natural shapes of soft joints through control point manipulation, the fingertip trajectory can be iteratively optimized to follow a desired path. These findings demonstrate the effectiveness of the proposed method in the geometric design of soft robotic structures. Future work will explore its potential for applications in both geometric design and topology optimization.

%
\section{Conclusion}\label{sec.con}
%
This paper introduces a general Lie Group framework for modeling soft–rigid coupling robotic systems, systematically incorporating cumulative parameterization on Lie groups into the modeling of Cosserat soft continuum rods for the first time. By representing the deformation field as increments in the Lie group $SE(3)$ and approximating these increments using general basis functions such as polynomials and B-splines, the proposed method ensures geometric consistency while maintaining a fixed-banded Jacobian structure. This design enables real-time simulation and control with efficiency and structural clarity comparable to rigid-body robotics.

Based on this formulation, we derive closed-form expressions for kinematics, statics, and dynamics and propose a symplectic integrator to preserve energy. The framework naturally supports segmental, branched, and soft–rigid hybrid structures without requiring explicit joint constraints, thereby enabling unified modeling and optimization. Numerical experiments show that it effectively captures multi-modal coupling and complex geometric constraints, making it a strong candidate for the integrated design and control of various soft robotic systems.

Future work will focus on introducing comprehensive modeling of contact and self-contact phenomena, as well as topology optimization of soft robots based on Lie group parameterization. Additionally, robot path planning and real-time control will be explored using the proposed approach. These enhancements are expected to further strengthen the practical applicability and efficiency of soft robotic systems in various applications, including medical surgery, space exploration, and animation.

\section*{Appendices}
\appendix

\section{Adjoint Operator}

\subsection{Adjoint Operator of the Lie Algebra}\label{notations}
The adjoint representation of the Lie algebra is given by
\[
\mathrm{ad}_{\xi} =
\begin{bmatrix}
\tilde{\kappa} & \mathbf{0}_{3\times3} \\
\tilde{\epsilon} & \tilde{\kappa}
\end{bmatrix}
\in \mathbb{R}^{6\times6}.
\]

\subsection{Adjoint Operator of the Lie Group}\label{notations2}
The adjoint representation of the Lie group is expressed as
\[
\mathrm{Ad}_{g} =
\begin{bmatrix}
R & \mathbf{0}_{3\times3} \\
\tilde{p}R & R
\end{bmatrix}
\in \mathbb{R}^{6\times6}.
\]

\section{Exponential Map on SE(3)}

Let \( \xi \in \mathfrak{se}(3) \) be an element of the Lie algebra, defined as
\[
\xi =
\begin{bmatrix}
\kappa \\
\epsilon
\end{bmatrix}
\in \mathbb{R}^{6}, \quad
\theta = \lVert \kappa \rVert.
\]
The exponential map is given by
\begin{equation}
\exp_{SE(3)}(\xi) =
\begin{bmatrix}
R(\kappa) & V^\top(\kappa)\, \epsilon \\[4pt]
\mathbf{0} & 1
\end{bmatrix},
\end{equation}
where \( R(\kappa) \) denotes the Rodrigues formula, and the left-trivialized velocity factor is
\begin{equation}
V(\kappa) =
I - \alpha\, \tilde{\kappa} + \beta\, \tilde{\kappa}^{2},
\end{equation}
with
\begin{equation}\label{eeeee}
\alpha = \frac{1 - \cos\theta}{\theta^{2}}, \quad
\beta = \frac{\theta - \sin\theta}{\theta^{3}}.
\end{equation}
The operator \( \tilde{(\cdot)} \) denotes the conversion of a 3D vector into its associated skew-symmetric matrix.

\section{Left-Trivialized Tangent}

\subsection{Series Form}
Let \( g = \exp(\xi) \) and \( g_s = \exp(s\xi) \). The tangent operator \( \mathrm{dexp}_{\xi} \) admits the following Lie algebraic expansion:
\begin{equation}\label{TJ}
\mathrm{dexp}_{\xi} = \mathrm{Ad}^{-1}_{g} \int_0^1 \mathrm{Ad}_{g_s} \, \mathrm{d}s = \sum_{n=0}^{\infty} \frac{(-1)^n}{(n+1)!} \, \mathrm{ad}_{\xi}^{n}.
\end{equation}

\subsection{Closed Form}
Let \( \xi = [\kappa^\top \ \epsilon^\top]^\top \). Using Rodrigues' formula, equation~\eqref{TJ} simplifies to the closed-form expression:
\begin{equation}
\mathrm{dexp}_{\xi} =
\begin{bmatrix}
V(\kappa) & {0} \\[4pt]
Q(\xi) & V(\kappa)
\end{bmatrix},
\end{equation}
where, for \( \kappa \neq 0 \),
\begin{equation}
\begin{aligned}
Q(\xi) =&
- \beta \, \tilde{\epsilon}
- \frac{(\alpha - 2\beta)}{\theta^2} (\kappa^\top \epsilon) \tilde{\kappa}
+ \frac{1 - \alpha}{\theta^2} \lceil \kappa, \epsilon \rceil \\
&+ \frac{\beta \theta^2 + 3(\alpha - 1)}{\theta^4} (\kappa^\top \epsilon) \tilde{\kappa}^2,
\end{aligned}
\end{equation}
and otherwise,
\[
Q(\xi) = \frac{1}{2} \tilde{\epsilon}.
\]
The double-bracket operator is defined as
\[
\lceil \kappa, \epsilon \rceil = \tilde{\kappa} \tilde{\epsilon} + \tilde{\epsilon} \tilde{\kappa},
\]
with \( \alpha \) and \( \beta \) as defined in equation~\eqref{eeeee}.
%
\section*{Funding}
%
This work was supported by a French government grant, managed by the Agence Nationale de la Recherche within the France 2030 program, under the reference ANR-21-RHUS-0006.
%
\section*{Declaration of conflicting interests}
%
The authors declared no potential conflicts of interest concerning the research, authorship, and/or publication of this article.
%
\section*{ORCID iDs}
%
L. Xun: \orcidlink{0000-0002-3924-6417} \href{https://orcid.org/0000-0002-3924-6417}{https://orcid.org/0000-0002-3924-6417} \\
B. Rosa: \orcidlink{0000-0002-7605-2056} \href{https://orcid.org/0000-0002-7605-2056}{https://orcid.org/0000-0002-7605-2056}\\
J. Szewczyk: \orcidlink{0000-0003-1925-3086} \href{https://orcid.org/0000-0003-1925-3086}{https://orcid.org/0000-0003-1925-3086}\\
B.Tamadazte: \orcidlink{0000-0002-4668-3092} \href{https://orcid.org/0000-0002-4668-3092}{https://orcid.org/0000-0002-4668-3092}
%
\section*{Supplemental Material}

Supplemental Material for this article is available online.
%
%
\bibliographystyle{plainnat}
\bibliography{biblio.bib}

\end{document}